\documentclass{article}

\usepackage[preprint]{neurips_2022}

\usepackage[utf8]{inputenc} 
\usepackage[T1]{fontenc}    
\usepackage{hyperref}       
\hypersetup{colorlinks,citecolor=blue}
\usepackage{url}            
\usepackage{booktabs}       
\usepackage{amsfonts}       
\usepackage{nicefrac}       
\usepackage{microtype}      
\usepackage{xcolor}         
\usepackage{tikz}

\usetikzlibrary{arrows.meta}
\usetikzlibrary{calc}
 \usepackage[utf8]{inputenc}   
\usepackage[T1]{fontenc}      

\usepackage{comment}

\usepackage{natbib}

\usepackage[tbtags]{amsmath}
\usepackage{amsthm}
\allowdisplaybreaks
\usepackage{amssymb,mathrsfs}
\usepackage{amsfonts}
\usepackage{upgreek}
\usepackage{xspace}
 \usepackage{comment}

\usepackage{graphicx}
\usepackage{subfig}
\usepackage{color}
\usepackage{algorithm, algpseudocode}

\usepackage[inline]{enumitem}
\usepackage{ifthen}
\usepackage{xargs}

\usepackage{aliascnt}
\usepackage{cleveref}
\usepackage{autonum}
\makeatletter
\newtheorem{theorem}{Theorem}
\crefname{theorem}{theorem}{Theorems}
\Crefname{Theorem}{Theorem}{Theorems}

\newtheorem*{lemma_nonumber*}{Lemma}

\newaliascnt{lemma}{theorem}
\newtheorem{lemma}[lemma]{Lemma}
\aliascntresetthe{lemma}
\crefname{lemma}{lemma}{lemmas}
\Crefname{Lemma}{Lemma}{Lemmas}

\newaliascnt{corollary}{theorem}
\newtheorem{corollary}[corollary]{Corollary}
\aliascntresetthe{corollary}
\crefname{corollary}{corollary}{corollaries}
\Crefname{Corollary}{Corollary}{Corollaries}

\newaliascnt{proposition}{theorem}
\newtheorem{proposition}[proposition]{Proposition}
\aliascntresetthe{proposition}
\crefname{proposition}{proposition}{propositions}
\Crefname{Proposition}{Proposition}{Propositions}

\newaliascnt{definition}{theorem}

\aliascntresetthe{definition}
\crefname{definition}{definition}{definitions}
\Crefname{Definition}{Definition}{Definitions}

\newaliascnt{remark}{theorem}

\aliascntresetthe{remark}
\crefname{remark}{remark}{remarks}
\Crefname{Remark}{Remark}{Remarks}

\crefname{example}{example}{examples}
\Crefname{Example}{Example}{Examples}

\crefname{technique}{technique}{techniques}
\Crefname{Technique}{Technique}{Techniques}

\crefname{figure}{figure}{figures}
\Crefname{Figure}{Figure}{Figures}

\crefformat{assumption}{{\textbf{A}}#2#1#3}

\newtheorem{assumptionF}{\textbf{F}\hspace{-3pt}}
\crefformat{assumptionF}{{\textbf{F}}#2#1#3}

\Crefname{assumptionB}{\textbf{B}\hspace{-3pt}}{\textbf{B}\hspace{-3pt}}
\crefname{assumptionB}{\textbf{B}}{\textbf{B}}

\Crefname{assumptionC}{\textbf{C}\hspace{-3pt}}{\textbf{C}\hspace{-3pt}}
\crefname{assumptionC}{\textbf{C}}{\textbf{C}}

\Crefname{assumptionH}{\textbf{H}\hspace{-3pt}}{\textbf{H}\hspace{-3pt}}
\crefname{assumptionH}{\textbf{H}}{\textbf{H}}

\Crefname{assumptionT}{\textbf{T}\hspace{-3pt}}{\textbf{T}\hspace{-3pt}}
\crefname{assumptionT}{\textbf{T}}{\textbf{T}}

\Crefname{assumptionT}{\textbf{T}\hspace{-3pt}}{\textbf{T}\hspace{-3pt}}
\crefname{assumptionT}{\textbf{T}}{\textbf{T}}

\Crefname{assumptionL}{\textbf{L}\hspace{-3pt}}{\textbf{L}\hspace{-3pt}}
\crefname{assumptionL}{\textbf{L}}{\textbf{L}}

\Crefname{assumptionQ}{\textbf{Q}\hspace{-3pt}}{\textbf{Q}\hspace{-3pt}}
\crefname{assumptionQ}{\textbf{Q}}{\textbf{Q}}

\Crefname{assumptionAR}{\textbf{AR}\hspace{-3pt}}{\textbf{AR}\hspace{-3pt}}
\crefname{assumptionAR}{\textbf{AR}}{\textbf{AR}}

\usepackage{bm}
\usepackage{wrapfig}

 \newcommand{\x} {{\bf x}}
\def\wG{\bar{\rmG}}

\newcommand \ox {{\bar x}}
\newcommand \om {{\omega}}
\newcommand \bs {{s}}  
\newcommand \twx {\bar{x}}
\newcommand \tox {{\bar x}} 
\newcommand\lb{{\langle}}
\newcommand\rb{{\rangle}}
\newcommand \oG {{\bar \rmG}}
\newcommand\Ld{{\bf L^2}}
\newcommand\error{{E}}  
\newcommand\txwav{x}

\definecolor{commentcolor}{rgb}{0., 0.5, 0.}

\def\pref{p_{\infty}}
\def\rmD{\mathrm{D}}

\def\contspace{\mathcal{C}}

\def\Pens{\mathscr{P}}

\newcommand{\tta}{\mathtt{A}}

\newcommand{\Capprox}{\tta}

\newcommandx\ctun[1][1=T]{\Capprox_{#1,1}}

\newcommandx{\expec}[2]{{\mathbb E}\left[#1 \middle \vert #2  \right]} 

\newcommand{\rme}{\mathrm{e}}

\newcommand{\Mtt}{M}

\newcommand{\Ktt}{K}

\def\x{{ \boldsymbol x}}

\newcommandx{\norm}[2][1=]{\ifthenelse{\equal{#1}{}}{\left\Vert #2 \right\Vert}{\left\Vert #2 \right\Vert^{#1}}}
\newcommandx{\normLigne}[2][1=]{\ifthenelse{\equal{#1}{}}{\Vert #2 \Vert}{\Vert #2\Vert^{#1}}}

\def\bfc{\mathbf{c}}

\def\bfZ{\mathbf{Z}}

\def\bfZ{\mathbf{Z}}

\def\bfM{\mathbf{M}}

\def\msa{\mathsf{A}}

\newcommand{\mcb}[1]{\mathcal{B}(#1)}

\def\Qbb{\mathbb{Q}}
\def\Rbb{\mathbb{R}}

\def\Pbb{\mathbb{P}}

\def\rset{\mathbb{R}}

\def\nset{\mathbb{N}}

\def\rmP{\mathrm{P}}

\def\rmd{\mathrm{d}}

\def\rme{\mathrm{e}}

\def\rmc{\mathrm{C}}

\def\rmH{\mathrm{H}}
\def\rmV{\mathrm{V}}

\newcommand{\R}{\mathbb R}
\newcommand{\Z}{\mathbb Z}

\def\trace{\operatorname{Tr}}
\newcommandx{\functionspace}[2][1=+]{\mathbb{F}_{#1}(#2)}

\newcommand{\argmin}{\operatorname*{arg\,min}}

\newcommandx{\VarDeux}[3][3=]{\operatorname{Var}^{#3}_{#1}\left\{#2 \right\}}

\newcommand{\LeftEqNo}{\let\veqno\@@leqno}

\newcommand{\N}{\ensuremath{\mathbb{N}}}

\newcommand{\PE}{\mathbb{E}}

\newcommand{\abs}[1]{\left\vert #1 \right\vert}
\newcommand{\absLigne}[1]{\vert #1 \vert}

\newcommand{\tvnormLigne}[1]{\| #1 \|_{\mathrm{TV}}}

\newcommandx{\Vnorm}[2][1=V]{\| #2 \|_{#1}}
\newcommandx{\VnormEq}[2][1=V]{\left\| #2 \right\|_{#1}}

\newcommand{\defEnsLigne}[1]{\lbrace #1 \rbrace }

\newcommand{\probaLigne}[1]{\mathbb{P}( #1 )}
\newcommandx\probaMarkovTilde[2][2=]
{\ifthenelse{\equal{#2}{}}{{\widetilde{\mathbb{P}}_{#1}}}{\widetilde{\mathbb{P}}_{#1}\left[ #2\right]}}

\newcommand{\expeLigne}[1]{\PE [ #1 ]}
\newcommand{\expeLignesqrt}[1]{\PE^{1/2} [ #1 ]}
\newcommand{\expeLignesqrtt}[1]{\PE^{1/4} [ #1 ]}

\def\eqsp{\;}
\newcommand{\coint}[1]{\left[#1\right)}
\newcommand{\ocint}[1]{\left(#1\right]}
\newcommand{\ooint}[1]{\left(#1\right)}
\newcommand{\ccint}[1]{\left[#1\right]}

\newcommandx{\weight}[2][2=n]{\omega_{#1,#2}^N}

\newcommand{\cball}[2]{\bar{\operatorname{B}}(#1,#2)}

\newcommandx\sequence[3][2=,3=]
{\ifthenelse{\equal{#3}{}}{\ensuremath{\{ #1_{#2}\}}}{\ensuremath{\{ #1_{#2}, \eqsp #2 \in #3 \}}}}

\newcommandx\sequenceD[3][2=,3=]
{\ifthenelse{\equal{#3}{}}{\ensuremath{\{ #1_{#2}\}}}{\ensuremath{( #1)_{ #2 \in #3} }}}

\newcommandx{\sequencen}[2][2=n\in\N]{\ensuremath{\{ #1_n, \eqsp #2 \}}}
\newcommandx\sequenceDouble[4][3=,4=]
{\ifthenelse{\equal{#3}{}}{\ensuremath{\{ (#1_{#3},#2_{#3}) \}}}{\ensuremath{\{  (#1_{#3},#2_{#3}), \eqsp #3 \in #4 \}}}}
\newcommandx{\sequencenDouble}[3][3=n\in\N]{\ensuremath{\{ (#1_{n},#2_{n}), \eqsp #3 \}}}

\newcommand{\opnorm}[1]{{\left\vert\kern-0.25ex\left\vert\kern-0.25ex\left\vert #1
    \right\vert\kern-0.25ex\right\vert\kern-0.25ex\right\vert}}

\def\Id{\operatorname{Id}}

\newcommandx{\CPE}[3][1=]{{\mathbb E}_{#1}\left[#2 \middle \vert #3  \right]} 
\newcommandx{\CPELigne}[3][1=]{{\mathbb E}_{#1}[#2  \vert #3  ]} 
\newcommandx{\CPEsq}[3][1=]{{\mathbb{E}^{1/2}}_{#1}\left[#2 \middle \vert #3  \right]} 
\newcommandx{\CPVar}[3][1=]{\mathrm{Var}^{#3}_{#1}\left\{ #2 \right\}}
\newcommand{\CPP}[3][]
{\ifthenelse{\equal{#1}{}}{{\mathbb P}\left(\left. #2 \, \right| #3 \right)}{{\mathbb P}_{#1}\left(\left. #2 \, \right | #3 \right)}}

\newcommand\Ent[2]{\mathrm{Ent}_{#1}\left(#2\right)}
\newcommandx{\osc}[2][1=]{\mathrm{osc}_{#1}(#2)}

\def\Id{\operatorname{Id}}

\def\V{V}

\def\bx{\bar{\boldsymbol x}}

\def\tx{\tilde{x}}

\def\rmD{\mathrm{D}}

\newcommand\coupling[2]{\Gamma(\mu,\nu)}

\def\vareps{\varepsilon}

\newcommandx{\KL}[2]{\operatorname{KL}\left( #1 \middle\Vert #2 \right)}
\newcommandx{\KLsqrt}[2]{\operatorname{KL}^{1/2}\left( #1 \middle\Vert #2 \right)}
\newcommandx{\KLLignesqrt}[2]{\operatorname{KL}( #1 \Vert #2 )^{1/2}}
\newcommandx{\Jef}[2]{\operatorname{J}\left( #1 , #2 \right)}
\newcommandx{\JefLigne}[2]{\operatorname{J}( #1 , #2 )}
\newcommandx{\KLLigne}[2]{\operatorname{KL}( #1 \Vert #2 )}

\def\gaStep
\def\QKer{Q}

\def\Tnplusun{\mathcal{T}_{k+1}}

\def\distance{\mathbf{d}}
\newcommandx{\wasserstein}[3][1=\distance,3=]{\mathbf{W}_{#1}^{#3}\left(#2\right)}
\newcommandx{\wassersteinLigne}[3][1=\distance,3=]{\mathbf{W}_{#1}^{#3}(#2)}
\newcommandx{\wassersteinD}[1][1=\distance]{\mathbf{W}_{#1}}
\newcommandx{\wassersteinDLigne}[1][1=\distance]{\mathbf{W}_{#1}}

\def\Rcoupling{\mathrm{R}}

\def\sigmaD{\sigma^2}

\newcommandx{\phibfs}[1][1=]{\pmb{\varphi}_{\sigmaD_{#1}}}

\newcommandx\sequenceg[3][2=,3=]
{\ifthenelse{\equal{#3}{}}{\ensuremath{( #1_{#2})}}{\ensuremath{( #1_{#2})_{ #2 \geq #3}}}}

\def\Rker{\Rcoupling}

\def\Qker{\mathrm{Q}}

\def\rmL{\mathrm{L}}
\def\rmG{G}

\newcommandx{\distV}[1][1=\bfc]{\mathbf{W}_{#1}}
\newcommandx{\distVdeux}[1][1=W_2]{\mathbf{d}_{#1}}

\def\mtt{\mathtt{m}}

 \usepackage{comment}

\title{Wavelet Score-Based Generative Modeling}

\author{%
  Florentin Guth \\
    Computer Science Department, \\
  ENS, CNRS, PSL University
  \And
  Simon Coste
  \\
    Computer Science Department, \\
  ENS, CNRS, PSL University
  \AND
  Valentin De Bortoli \\
    Computer Science Department, \\
  ENS, CNRS, PSL University  
  \And
  St\'ephane Mallat \\
 Collège de France, Paris, France \\
Flatiron Institute, New York, USA }

\begin{document}

\maketitle
\begin{abstract}
Score-based generative models (SGMs) synthesize new data samples from Gaussian white noise by running a time-reversed Stochastic Differential Equation (SDE) whose drift coefficient depends on some probabilistic score. The discretization of such SDEs typically requires a large number of time steps and hence a high computational cost. This is because of ill-conditioning properties of the score that we analyze mathematically. We show that SGMs can be considerably accelerated, by factorizing the data distribution into a product of conditional probabilities of wavelet coefficients across scales. The resulting Wavelet Score-based Generative Model (WSGM) synthesizes wavelet coefficients with the same number of time steps at all scales, and its time complexity therefore grows linearly with the image size. This is proved mathematically over Gaussian distributions, and shown numerically over physical processes at phase transition and natural image datasets.
\end{abstract}

\section{Introduction}
\label{sec:intro}

Score-based Generative Models (SGMs) have obtained remarkable results to learn
and sample probability distributions of image and audio signals
 \citep{song2019generative,nichol2021improved,nichol2021beatgans}.
They proceed as follows: the data distribution is mapped to a Gaussian white
distribution by evolving along a Stochastic Differential Equation (SDE), which
progressively adds noise to the data.  The generation is implemented using the
time-reversed SDE, which transforms a Gaussian white noise into a data sample.
At each time step, it pushes samples along the gradient of the log probability,
also called \emph{score function}.  This score is estimated by leveraging tools
from score-matching and deep neural networks
\citep{hyvarinen2005estimation,vincent2011connection}.  At sampling time, the
computational complexity is therefore proportional to the number of time steps,
i.e., the number of forward network evaluations.  Early SGMs in
\cite{song2019generative,song2020score,ho2020denoising} used thousands of time
steps, and hence had a limited applicability.

Diffusion models map a Gaussian white distribution into a highly
complex data distribution. We thus expect that this process will require a large number of time
steps.  It then comes as a surprise that recent approaches have
drastically reduced the the number of time steps. This is achieved by optimizing the
discretization schedule or by modifying the original SGM formulation
\citep{kadkhodaie2020solving,jolicoeur2021gotta,liu2022pseudo,zhang2022exponential,san2021noise,nachmani2021non,song2020denoising,kong2021fast,ho2020denoising,luhman2021knowledge,salimans2022progressive,xiao2021tackling}.
High-quality score-based generative models have also been improved by cascading
multiscale image generations \citep{saharia2021image, ho2022cascaded,
  nichol2021beatgans} or with subspace decompositions
\citep{jing2022subspace}. We show that with an important modification, these approaches provably improve the performance of SGMs. 

A key idea is that typical high-dimensional probability distributions coming
from physics or natural images have complex multiscale properties. They can be
simplified by factorizing them as a product of conditional probabilities of
normalized wavelet coefficients across scales. These conditional probabilities
are more similar to Gaussian white noise than the original image distribution,
and can thus be sampled more efficiently.  On the physics side, this observation
is rooted in the renormalization group decomposition in statistical physics
\citep{wilson1971renormalization}, and has been used to estimate physical
energies from data \cite{WCRG}. In image processing, it relies on statistical
observations of wavelet coefficient properties \citep{wainwright1999scale}.  A Wavelet Score-based Generative
Model (WSGM) generates normalized wavelet coefficients from coarse to fine
scales, as illustrated in \Cref{fig1}.  The conditional probability of
each set of wavelet coefficients, given coarse scale coefficients, is sampled
with its own (conditional) SGM. The main result is that a normalization of
wavelet coefficients allows fixing the same  discretization schedule at all
scales. Remarkably, and as opposed to existing algorithms, it implies that the
total number of sampling iterations per image pixel does not depend on the image
size.

After reviewing score-based generation models,
\Cref{sec:time-sampling-score} studies the mathematical properties of its time
discretization, with a focus on Gaussian models and multiscale processes.
Images and many physical processes are typically non-Gaussian, but do have a
singular covariance with long- and short-range correlations. In
\Cref{sec:wavel-diff-score}, we explain how to factorize these processes
into probability distributions which capture interactions across scales by
introducing orthogonal wavelet transforms.  We shall prove that it allows considering SGMs
with the same time schedule at all scales, independently of the image
size. In \Cref{sec:time-reduct-wavel}, we present numerical results on Gaussian
distributions, the $\varphi^4$ physical model at phase transition, and the
CelebA-HQ image dataset \citep{celeba-hq}.  The main contributions of the paper are as follows:
\begin{itemize}
\item A Wavelet Score-based Generative Model (WSGM) which generates conditional
  probabilities of normalized wavelet coefficients, with the same discretization
  schedule at all scales. The number of time steps per image pixel does not need
  to depend upon the image size to reach a fixed error level.
\item Theorems controlling errors of time discretizations of SGMs, proving
  accelerations obtained by scale separation with wavelets. These results are
  empirically verified by showing that WSGM provides an acceleration for the
  synthesis of physical processes at phase transition and natural image datasets.
\end{itemize}

\begin{figure}[h]
  \centering
  \includegraphics[width=\linewidth]{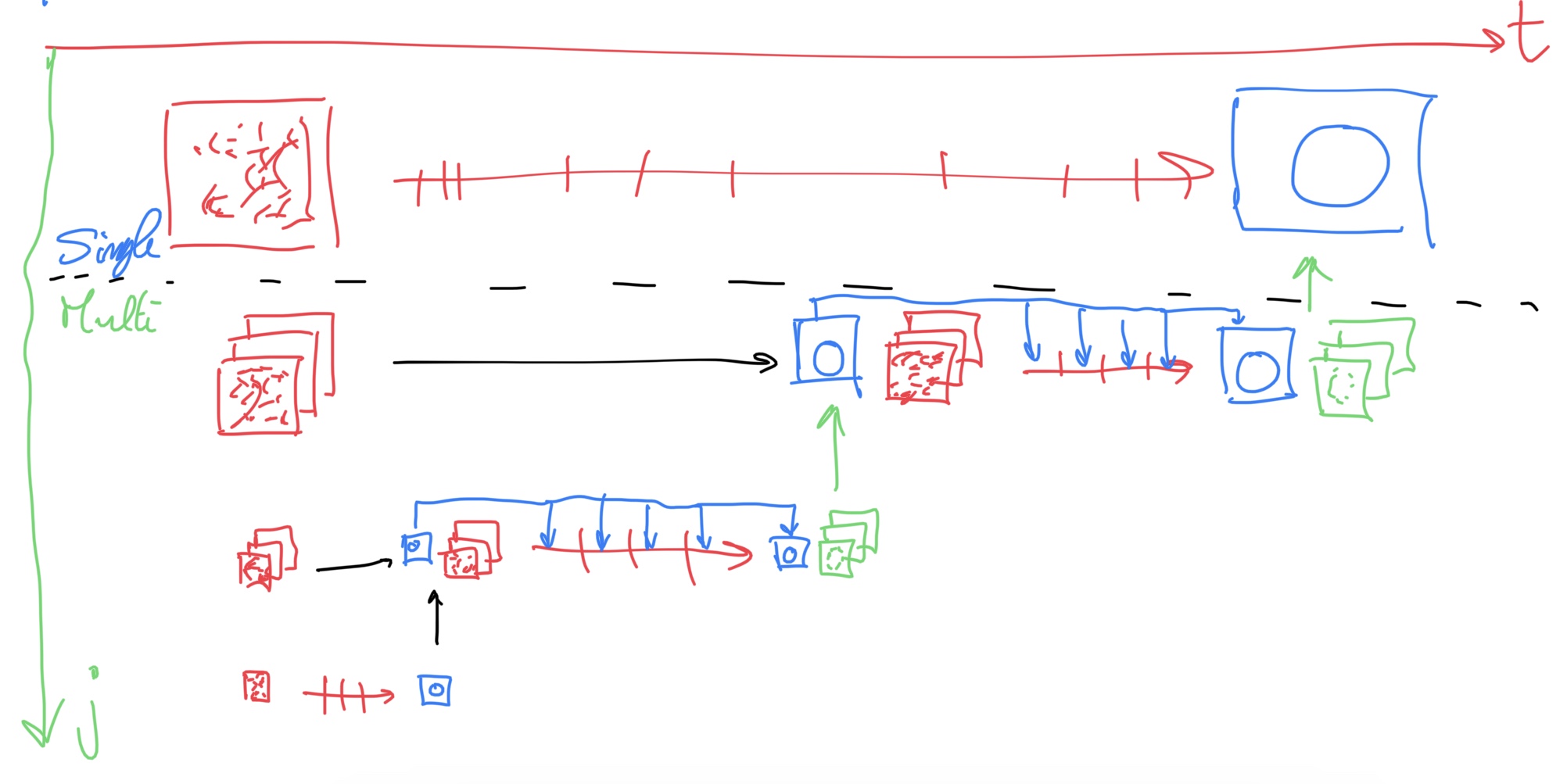}
    \caption{
    An SGM generates images by discretizing a reverse diffusion, which progressively transforms white Gaussian noise into a natural image. A WSGM generates increasingly higher-resolution images by discretizing reverse diffusions on wavelet coefficients at each scale. It begins by generating a first low-resolution image. Renormalized wavelet coefficients are then generated conditionally to this low-resolution image. A fast inverse wavelet transform reconstructs a higher-resolution image from these wavelet coefficients. This process is repeated at each scale. The number of steps is the same at each scale, and can be orders of magnitude smaller than for SGM.
  }
        \label{fig1}
\end{figure}

\section{Sampling and Discretization of Score-Based Generative Models}
\label{sec:time-sampling-score}
\subsection{Score-Based Generative Models}

\paragraph{Diffusions and time reversal} 
A Score-based Generative Model (SGM)
\citep{song2019generative,song2020score,ho2020denoising} progressively maps the
distribution of data $x$ into the normal distribution, with a forward SDE which
iteratively adds Gaussian white noise.
It is associated with a {\it noising
process} $(x_t)_t$, with $x_0$ distributed according to the data
distribution $p$, and satisfying:

\begin{equation}
  \rmd x_t = -x_t \rmd t + \sqrt{2} \rmd w_t, \label{eq:noising}
\end{equation}
where $(w_t)_t$ is a Brownian motion.
The solution is an Ornstein-Uhlenbeck
process which admits the following representation for any $t \geq 0$:
\begin{equation}
  \label{eq:OU_sol}
  x_t = \rme^{-t}\, x_0 + (1- \rme^{-2t})^{1/2} z , \qquad z \sim {\cal N}(0,\Id) . 
\end{equation}
The process $(x_t)_t$ is therefore an interpolation between a data sample $x_0$ and
Gaussian white noise. The \emph{generative process} inverts
\eqref{eq:noising}. Under mild assumptions on $p$
\citep{cattiaux2021time,haussmann1986time}, for any $T \geq 0$, the reverse time
process $x_{T-t}$ satisfies:
\begin{equation} \label{eq:generative_modeling} 
\rmd x_{T-t} = \{ x_{T-t} + 2
  \nabla \log p_{T-t}(x_{T-t})\}\, \rmd t + \sqrt{2} \,\rmd w_t,
\end{equation}
where $p_t$ is the probability density of $x_t$, and $\nabla \log p_t$ is called
the {\it Stein score}.  Since $x_T$ is close to a white Gaussian
variable, one can approximately sample from $x_T$ by sampling from the normal distribution.
We can generate $x_0$ from $x_T$ by solving this
time-reversed SDE, if we can estimate an accurate approximation of the score
$\nabla \log p_t$ at each time $t$, and if we can discretize the SDE without
introducing large errors.

Efficient approximations of the Stein scores
are the workhorse of SGM. \cite{hyvarinen2005estimation} shows that
the score $\nabla \log p_t$ can be approximated with parametric functions
${\bs}_\theta$ which minimize the following score matching loss:
    \begin{equation}
      \label{eq:ism_loss}
       \textstyle{
      { s}_t =  \argmin_{\theta} {\mathbb E}_{p_t} [
      \frac 1 2 \normLigne{ {\bs}_\theta(x_t)}^2 +
        \mathrm{div}({\bs}_\theta)(x_t)}]  .
    \end{equation}
For image generation, $\bs_{\theta}$ is calculated by a neural
network parameterized by $\theta$. In statistical physics problems where the
energy can be linearly expanded with coupling parameters, we obtain linear
models $\bs_\theta(x) = \theta^\top \nabla U(x)$.
This is the case for Gaussian processes where $U(x) = x x^\top$; it also
applies to non-Gaussian processes, using non-quadratic terms in $U(x)$. 
    
\paragraph{Time discretization of generation} 
An approximation of the generative process \eqref{eq:generative_modeling} is
computed by approximating $\nabla \log p_t$ by $\bs_t$ and
discretizing time. It amounts to approximating the time-reversed SDE by a Markov
chain which is initialised by $\tx_T \sim {\cal N}(0, \Id)$, and computed over
times $t_k$ which decrease from $t_N = T$ to $t_0 = 0$, at intervals
$\delta_k = t_{k} - t_{k-1}$:
\begin{equation}
  \label{eq:backward_disc}
  \tx_{t_{k-1}} = \tx_{t_{k}} + \delta_{k} \{ \tx_{t_{k}} + 2  {\bs}_{ t_{k}}(\tx_{t_{k}})\} + \sqrt{2 \delta_{k}} z_{k} , \qquad z_k  \overset{\textup{i.i.d.}}{\sim} {\cal N}(0,\Id) . 
\end{equation}
Ignoring the error due to the score model, the minimum number of time steps is
limited by the Lipschitz regularity of the score $\nabla \log p_t$, see
\cite[Theorem 1]{debortoli2021diffusion}. The overall complexity of this
generation is $O(N)$ evaluations of the score ${\bs}_{ t}(x)$.

\subsection{Discretization of SGM and Score Regularity}
\label{sec:sampl-accel-accur}
We study how the regularity of the score $\nabla \log p$ affects the
discretization of \eqref{eq:backward_disc}. Assuming that the score is known,
i.e., that $\bs_t = \nabla \log p_t$, we prove that for Gaussian
processes, the number of time steps to reach a fixed error $\vareps$ depends on
the condition number of its covariance.  This result is generalized to
non-Gaussian processes by relating this error to by the regularity
of $\nabla \log p_t$.

\paragraph{Gaussian distributions}
Suppose that the data distribution is a Gaussian
$p =\mathcal{N}(0, \Sigma)$ with covariance matrix $\Sigma$, in dimension $d$. 
Let $p_t$ be the distribution of $x_t$.
Using \eqref{eq:OU_sol}, we have:
\begin{equation}
  \nabla \log p_t(x) = - (\Id + (\Sigma - \Id)\rme^{-2t})^{-1}x  .
\end{equation}
Let $\tilde p_t$ be the distribution of $\tx_t$ obtained by
the time discretization \eqref{eq:backward_disc}. 
The approximation error between the distribution $\tilde p_0$
obtained with the time-reversed SDE and the data distribution $p$ stems from 
\begin{enumerate*}[label=(\roman*)]
\item the mismatch between the distributions of $x_T$ and  $\tx_T$, and
\item the time discretization.
\end{enumerate*}
The following theorem relates these two errors to the covariance $\Sigma$ of $x$
in the particular case of a uniform time sampling at intervals
$\delta_k = \delta$.  We normalize the signal energy by imposing that ${\rm Tr}(\Sigma) = d$,
and we write
$\kappa$ the condition number of $\Sigma$, which is the ratio between its
largest and smallest eigenvalues.

\begin{theorem}
  \label{prop:approx_gaussian_kl}
  We have $\KL{p}{\tilde{p}_0} \leq \error_{T} + \error_{\delta} + \error_{T, \delta}$, with :
     \begin{align}
  \label{firsteq0}
    &&\error_T =  f(\rme^{-4T}\,\absLigne{\mathrm{Tr}((\Sigma - \Id)\Sigma))}) , \\ 
     \label{firsteq0}
    &&\error_{\delta} = f(\delta \absLigne{\mathrm{Tr}(\Sigma^{-1} - \Sigma (\Sigma - \Id)^{-1} \log(\Sigma)/2 + (\Id - \Sigma^{-1})/3))})  ,
  \end{align}
  where $f(t) = t-\log(1+t)$ and $\error_{T, \delta}$ is a higher-order term with $E_{T, \delta} = o(\delta + \rme^{-4T})$ when $\delta \to 0$ and $T \to +\infty$. Furthermore, for any $\vareps > 0$,
  there exists $T, \delta \geq 0$ such that:
  \begin{equation}
    \label{eq:N_conditioning}
    (1/d) (\error_T + \error_\delta) \leq \vareps~~\mbox{and}~~
    T/\delta \leq C \vareps^{-2} \kappa^3  .
  \end{equation}
  with $C \geq 0$ a universal constant.
  \end{theorem}

This theorem specifies the dependence of the Kullback-Leibler error on the
covariance matrix. It computes an upper bound on the number of time steps
$N = T / \delta$ as a function of the condition number $\kappa$ of $\Sigma$.
As expected, it indicates that the number of time steps increases with the
condition number of the covariance. This theorem is proved in a more general case in \Cref{sec:proof-section21}, which includes the case where $p$ has a non-zero mean. An exact expansion of the Kullback-Leibler divergence is also given.

For stationary processes of images, the covariance eigenvalues
are given by the power spectrum, which typically decays like $|\om|^{-1}$ at a
frequency $\om$. It results that $\kappa$ is proportional to a power of the
image size. Many physical phenomena produce such stationary images with a
power spectrum having a power law decay. In these typical cases, the number of time steps
must increase with the image size. This is indeed what is observed in
numerical SGM experiments, as seen in \Cref{sec:wavel-diff-score}.

\paragraph{General processes} \Cref{prop:approx_gaussian_kl} can be
extended to non-Gaussian processes.  The
number of time steps then depends on the regularity of the score $\nabla \log p_t$.

\begin{theorem}
  \label{thm:control_diffusion}
  Assume that $\nabla \log p_t (x)$ is $\mathscr{C}^2$ in both $t$ and $x$, and that:
  \begin{equation}\label{params:thm2}
    \textstyle{
  \sup_{x,t} \normLigne{\nabla^2 \log p_t(x)} \leq
  \Ktt}~~\mbox{and}~~
  \textstyle{\normLigne{\partial_t \nabla \log p_t(x)} \leq \Mtt\, \rme^{-\alpha t}\,
  \normLigne{x}}.
  \end{equation}
  for some $\Ktt, \Mtt , \alpha >0$. Then $\tvnormLigne{p - \tilde p_0} \leq \error_T + \error_\delta + \error_{T, \delta}$, where:
  \begin{align}
  \label{firsteq}
    \error_T &= \sqrt{2} \rme^{-T} \KLLignesqrt{p}{{\mathcal{N}(0, \Id)}}  ,\\
     \label{secondeq}
    \error_\delta &= 6\,\sqrt{\delta}\, [1+{\mathbb E}_p (\normLigne{x}^4 )^{1/4}]\, [1 +  \Ktt + \Mtt (1 +1/(2 \alpha)^{1/2}) ]  ,
  \end{align}
  and $\error_{\delta, T}$ is a higher order term with $E_{T, \delta} = o(\sqrt{\delta} + \rme^{-T})$ when $\delta \to 0$ and $T \to +\infty$.
\end{theorem}

The proof of \Cref{thm:control_diffusion} is postponed to
\Cref{sec:proof-section21} and we show that the result can be strengthened by
providing a quantitative upper bound on $\tvnormLigne{p - \tilde p_0}$.
\Cref{thm:control_diffusion} improves on \cite[Theorem
1]{debortoli2021diffusion} by proving explicit bounds exhibiting the
dependencies on the regularity constants $\Ktt$ and $\Mtt$ of the score and by
eliminating an exponential growth term in $T$ in the upper bound.
\Cref{thm:control_diffusion} is much more general but not as tight as
\Cref{prop:approx_gaussian_kl}.  

The first error term \eqref{firsteq} is due to
the fact that $T$ is chosen to be finite. 
The second error term \eqref{secondeq}
controls the error depending upon the discretization time step $\delta$.  Since
$p_t$ is obtained from $p$ through a high-dimensional convolution with a
Gaussian convolution of variance proportional to $t$, the regularity of
$\nabla \log p_t (x)$ typically increases with $t$ so
$\normLigne{\nabla^2 \log p_t(x)}$ and
$\normLigne{\partial_t \nabla \log p_t(x)}$ rather decrease when $t$
increase. This qualitatively explains why a \emph{quadratic} discretization
schedule with non-uniform time steps $\delta_k \propto k$ are usually chosen in numerical
implementations of SGMs \citep{nichol2021improved,song2020improved}.  For
simplicity, we focus on the uniform discretization schedule, but our result could
be adapted to non-uniform time steps with no major difficulties.

If $p$ is Gaussian, then the Hessian $\nabla^2 \log p$ is the negative inverse of the covariance matrix. We verify in \Cref{sec:proof-section21} that in this case, the assumptions of  \Cref{thm:control_diffusion} are satisfied. Furthermore, the constants $\Ktt$ and $\Mtt$, and hence the number of discretization steps, are controlled using the condition number of $\Sigma$. We thus conjecture that  non-Gaussian processes with an ill-conditioned covariance matrix will require many discretization steps to have a small error. This will be verified numerically. As we now explain, such processes are ubiquitous in physics and natural image datasets.

\paragraph{Multiscale processes}
 Most images have variations on a wide range of scales.  They require to use
 many time steps to sample using an SGM, because their score is not well-conditioned.
 This is also true for a wide range of phenomena encountered in physics,
 biology, or economics \citep{kolmogorov_1962,mandelbrot1983fractal}.  We define a \emph{multiscale
 process} as a stationary process whose power spectrum has a power law
 decay. The stationarity implies that its covariance is diagonalized in a
 Fourier basis.  Its eigenvalues, also called power spectrum, have a power law
 decay defined by:
\begin{equation}\label{eq:power_law_spectrum}
P(\om) \sim (\xi^\eta + |\om|^\eta)^{-1} ,
\end{equation}
where $\eta > 0$ and $2 \pi/\xi$ is the maximum correlation length. 
Physical processes near phase transitions have such a power-law decay, but it is also the case of many disordered systems such as fluid and gas turbulence. 
Natural images also typically define stationary processes. Their power spectrum satisfy this property with $\eta = 1$
and $2\pi / \xi \approx L$ which is the image width.
To efficiently synthesize images and more general multiscale signals, we must eliminate
the ill-conditioning properties of the score. This is done by applying a wavelet transform.

\section{Wavelet Score-Based Generative Model}
\label{sec:wavel-diff-score}

The numerical complexity of the SGM algorithm depends on the number of time steps,
which itself depends upon the regularity of the score.  
We show that an important acceleration is obtained by
factorizing the data distribution into normalized wavelet conditional
probability distributions, which are closer to normal Gaussian
distributions, and so whose score is better-conditioned.

\subsection{Wavelet Whitening and Cascaded SGMs}
\label{sec:whit-with-wavel}

\paragraph{Normalized orthogonal wavelet coefficients}
Let $x$ be the input signal of width $L$ and dimension $d = L^n$, with $n=2$ for
images.  We write $x_j$ its low-frequency
approximation subsampled at intervals $2^j$, of size $(2^{-j} L)^n$, with
$x_0 = x$.  At each scale $2^{j-1} \geq 1$, a fast wavelet orthogonal transform
decomposes $x_{j-1}$ into $(\ox_j , x_j)$ where $\ox_j$ are the wavelet
coefficient which carries the higher frequency information over $2^{n}-1$
signals of size $(2^{-j} L)^n$ \cite{mallat1999wavelet}.  They are calculated
with convolutional and subsampling operators $\rmG$ and $\oG$ specified in
\Cref{sec:wavelet-transforms}:
\begin{equation}
  \label{eq:low_res}
  \textstyle{
    x_j = \gamma_j^{-1}\, \rmG\, x_{j-1} ~~\mbox{and}~~
    \ox_j = \gamma_j^{-1}\, \oG\, x_{j-1} ~.
    }
\end{equation}
The normalization factor $\gamma_j$ guarantees that 
${\mathbb E}[\|\ox_j \|^2] = (2^n-1) (2^{-j} L)^n$.  We consider wavelet
orthonormal filters where $(\rmG,\wG)$ is a unitary operator, i.e.:
\[
\oG\rmG^\top = \rmG\oG^\top = 0~~\mbox{and}~~\rmG^\top\rmG + \oG^\top\oG = \Id .
\]
It results that $x_{j-1}$ is recovered from $(\ox_j,x_j)$ with:
\begin{equation}
  \label{eq:reconstruction_UNO}
x_{j-1} = \gamma_j \, \rmG^\top x_j + \gamma_j \, \oG^\top \ox_j .
\end{equation}
The wavelet transform is computed over $J \approx \log_2 L$ scales by iterating
$J$ times on \eqref{eq:low_res}. The last $x_J$ has a size
$(2^{-J} L)^n \sim 1$.  The choice of wavelet
filters $\rmG$ and $\oG$ specifies the properties of the wavelet transform and
the number of vanishing moments of the wavelet, as explained in 
\Cref{sec:wavelet-transforms}.

\paragraph{Renormalized probability distribution}
A conditional wavelet renormalization 
factorizes the distribution $p(x)$ of signals $x$ into conditional probabilities over wavelet coefficients:
\begin{equation}
  \label{product}
  \textstyle{
    p(x) =  \alpha \prod_{j=1}^{J} \bar p_j (\ox_j | x_j)\, p_{J} (x_{J})~.
    }
\end{equation}
where $\alpha$ (the Jacobian) depends upon all $\gamma_j$.

Although $p(x)$ is typically highly non-Gaussian, the factorization
\eqref{product} involves distributions that are closer to Gaussians.  The
largest scale distribution $p_{J}$ is usually close to a Gaussian when the image
has independent structures, because $x_J$ is an averaging of $x$ over large
domains of size $2^J$.  Remarkably, for large classes of signals and images,
each conditional probabilities
$\bar p_j (\ox_j | x_j)$ also
exhibit Gaussian properties.  In images, the wavelet coefficients $\ox_j$ are
usually sparse and thus have highly non-Gaussian distribution; however, it has
been observed \citep{wainwright1999scale} that their conditional distributions
$\bar p_j (\ox_j | x_j)$ become much more Gaussian, due to dependencies of
wavelet coefficients across scales.  Similarly, in statistical physics,
factorization of probabilities of high frequencies 
conditioned by lower frequencies have been introduced in
\cite{wilson1983renormalization}. More recently, normalized wavelet
factorizations \eqref{product} have been introduced in physics to implement
renormalization group calculations, and model probability distributions with
maximum likelihood estimators near phase transitions \citep{WCRG}.

\paragraph{Wavelet Score-based Generative Model} 
Instead of computing a Score-based Generative Model (SGM) of the distribution
$p(x)$, a Wavelet Score-based Generative Model (WSGM) applies an SGM at the
coarsest scale $p_J (x_J)$ and then on each conditional distribution
$\bar p_j (\ox_j | x_j)$ for $j \leq J$.  It is thus a cascaded SGM, similarly
to \cite{ho2022cascaded,saharia2021image}, but calculated on
$\bar p_j (\ox_j | x_j)$ instead of $p_j (x_{j-1}|x_j)$.  The normalization of
wavelet coefficients $\ox_j$ effectively produces a whitening which can
considerably accelerate the algorithm by reducing the number of time steps. This
is not possible on $x_{j-1}$ because its covariance is ill-conditioned. It will
be proved for Gaussian processes.

A forward noising process is computed on each $\ox_j$ for $j \leq J$ and $x_J$:
\begin{equation}\label{eq:noisingscales}
  \rmd \ox_{j,t} = -\ox_{j,t} \,\rmd t + \sqrt{2} \rmd \bar{w}_{j,t}~~\mbox{and}~~
  \rmd x_{J,t} = -x_{J,t} \,\rmd t + \sqrt{2} \rmd w_{J,t},
\end{equation}
where the $\bar{w}_{j,t}, w_{J,t}$ are Brownian motions.  Since $\bar{x}_j$ is
nearly white and has Gaussian properties, this diffusion converges much more
quickly than if applied directly on $x$.  Using \eqref{eq:ism_loss}, we compute a
score function ${\bs}_{J,t}(x_{J,t})$ which approximates the score
$\nabla \log p_{J,t}(x_{J,t})$. For each $j \leq J$ we also compute the
conditional score $\bar {\bs}_{j,t} (\ox_{j,t} | x_{j})$ which
approximates $\nabla \log \bar p_{j,t}(\ox_{j,t} | x_{j})$.

The inverse generative process is computed from coarse to fine scales as follows. 
At the largest scale $2^J$, we sample
the low-dimensional $x_J$ by time discretizing the inverse SDE. Similarly to 
\eqref{eq:backward_disc}, the generative process is given by:
\begin{equation}
  \label{eq:backward_disc2}
  \txwav_{J,t_{k+1}} = \txwav_{J,t_{k}} + \delta_{k} \{ \txwav_{J,t_{k}} + 2  {\bs}_{ J,t_{k}}(\txwav_{J,t_{k}})\} + \sqrt{2 \delta_{k}} z_{J,k} , \qquad z_{J,k}  \overset{\textup{i.i.d.}}{\sim} {\cal N}(0,\Id) . 
\end{equation}
For $j$ going from $J$ to $1$, we then generate the wavelet
coefficients $\ox_j$ conditionally to the previously calculated $\txwav_{j}$, by
keeping the same time discretization schedule at all scales:
  \begin{equation}
  \label{eq:backward_disc2}
  \tox_{j,t_{k+1}} = \tox_{j,t_{k}} + \delta_{k} \{ \tox_{j,t_{k}} + 2  \bar {\bs}_{ j,t_{k}}(\tox_{j,t_{k}}|\txwav_{j})\} + \sqrt{2 \delta_{k}}\, z_{j,k} , \qquad z_{j,k}  \overset{\textup{i.i.d.}}{\sim} {\cal N}(0,\Id) . 
\end{equation}
The inverse wavelet transform then approximately computes
a sample of $x_{j-1}$ from $(\tox_{j,0}, \txwav_{j})$:
\begin{equation}
  \label{eq:reconstruction}
\tx_{j-1} = \gamma_j \, \rmG^\top \txwav_{j} + \gamma_j \, \oG^\top \tox_{j,0} .
\end{equation}
The generative process is illustrated in \Cref{fig1} and its pseudocode is given in \Cref{sec:wsgm-pseudocode}, see
\Cref{alg:cascaded_wavelet}.  The appendix also verifies that if $x$ is of size $d$ then its numerical complexity is $O(N d)$ where $N$ is the number of time steps, which is the same at each scale. For multiscale processes, we shall see that the number of time steps $N$ does not depend upon $d$ to reach a fixed error measured with a KL divergence.

\paragraph{Related work}
The use of wavelets in multiscale architectures is often motivated by the spectral bias phenomenon \cite{rahaman2019spectral}: \cite{huang2017wavelet} uses a wavelet CNN for super-resolution, \cite{gal2021swagan} incorporates spectral information by training GANs in the wavelet domain. Closer in spirit to our work, \cite{yu2020wavelet} introduces Wavelet Flow, a normalizing flow with a cascade of layers generating wavelet coefficients conditionally on lower-scales, then aggregating them with an inverse wavelet transform. This method yields training time acceleration and high-resolution (1024x1024) generation. 

WSGM is closely related to other cascading diffusion algorithms, such as the ones
introduced in \cite{ho2022cascaded,saharia2021image,nichol2021beatgans}. The
main difference lies in that earlier works on cascaded SGMs do not model the
\emph{wavelet coefficients} $\{\bar{x}_j\}_{j=1}^J$ but the \emph{low-frequency}
coefficients $\{x_j\}_{j=1}^J$. As a result, cascaded models do not exploit the
whitening properties of the wavelet transform.   We
also point out the recent work of \cite{jing2022subspace} which, while not using
the cascading framework, drop subspaces from the noising process at different
times. This allows using only one SDE to sample approximately from the data
distribution. However, the reconstruction is still computed with respect to
$\{x_j\}_{j=1}^J$ instead of the wavelet coefficients.

Finally, we highlight that our work could be combined with other acceleration
techniques such as the ones of
\cite{jolicoeur2021gotta,liu2022pseudo,zhang2022exponential,
  san2021noise,nachmani2021non,song2020denoising,ho2020denoising,kong2021fast,luhman2021knowledge,salimans2022progressive,xiao2021tackling}
in order to improve the empirical results of WSGM.

\subsection{Discretization and Accuracy for Gaussian Processes}
\label{sec:time-sampl-accur}
We now illustrate \Cref{prop:approx_gaussian_kl} and the effectiveness of WSGM on Gaussian multiscale processes.  We use the
whitening properties of the wavelet transform to show that the time
complexity required in order to reach a given error is linear in the image dimension.

The following result proves that the normalization of wavelet coefficients
performs a preconditioning of the covariance, whose eigenvalues then remain of the
order of $1$. This is a consequence of a theorem proved by \cite{Meyer:92c} on
the representation of classes of singular operators in wavelet bases, see
\Cref{sec:wavelet-transforms}. As a result, the number of iterations $T/\delta$ required to reach an error $\varepsilon$ is independent of the dimension.

\begin{theorem}
  \label{lemma:bound_spectrum}
  Let $x$ be a Gaussian stationary process of power spectrum
  $P(\om) = c\,(\xi^\eta + |\om|^\eta)^{-1}$ with $\eta > 0$ and $\xi > 0$.
  If the wavelet has a compact support, $q  \geq \eta$ vanishing moments and is ${\mathscr{C}}^q$, then 
  $\KL{p}{\tilde{p}_0} = \error_T + \error_{\delta} + \error_{\delta, T}$ with
  $\error_{\delta, T}$ a higher order term such that $E_{T, \delta} = o(\delta + \rme^{-4T})$ when $\delta \to 0$ and $T \to +\infty$. 
  Furthermore, for any $\vareps >0$, there exists $C > 0$ such that for any $\delta$, $T$:
  \begin{equation}
    \label{eq:N_conditioning_WSGM}
    (1/d) (\error_T + \error_\delta) \leq  \vareps~~\mbox{and}~~
    T/\delta \leq C \vareps^{-2}   .
  \end{equation}
\end{theorem}
To prove this result, we show that the conditioning number of the covariance matrix of the renormalized wavelet coefficients does not depend on the dimension using \citep{Meyer:92c}. We conclude upon combining this result, the cascading property of the Kullback-Leibler divergence and an extension of Theorem 1 to the setting with non-zero mean. The detailed proof is postponed to \Cref{sec:proof33}.

\paragraph{Numerical results}
We illustrate \Cref{lemma:bound_spectrum} on a Gaussian
field $x$, with power spectrum $P$ as \eqref{eq:power_law_spectrum}. In \Cref{fig:gaussian}, we display the sup-norm between $P$ and
the power spectrum $\hat{P}$ of the samples obtained using either vanilla SGM or WSGM
with uniform stepsize $\delta_k = \delta$. In the case of vanilla SGM, the number $N(\varepsilon)$ of time steps needed to reach a small error $\Vert P - \hat{P} \Vert=\varepsilon$ increases with the size of the image $L$ \Cref{fig:gaussian}, right). Equation \eqref{eq:N_conditioning} suggests that $N(\varepsilon)$ scales like a power of the conditioning number $\kappa$ of $\Sigma$, which, for multiscale gaussian processes, is $\kappa \sim L^\eta$. In the WSGM case, we sample conditional distributions $\bar{p}_j$ of wavelets $\bar{x}_j$  given low-frequencies $x_j$. At scale $j$, the conditioning numbers $\bar{\kappa}_j$ of the conditional covariance become dimension-independent (\Cref{sec:wavelet-transforms}), removing the dependency of $N(\varepsilon)$ on the image size $L$ as suggested by \eqref{eq:N_conditioning_WSGM}.

\begin{figure}[h!]
    \centering
    \begin{tabular}{cc}
    \includegraphics[height=0.17\textheight]{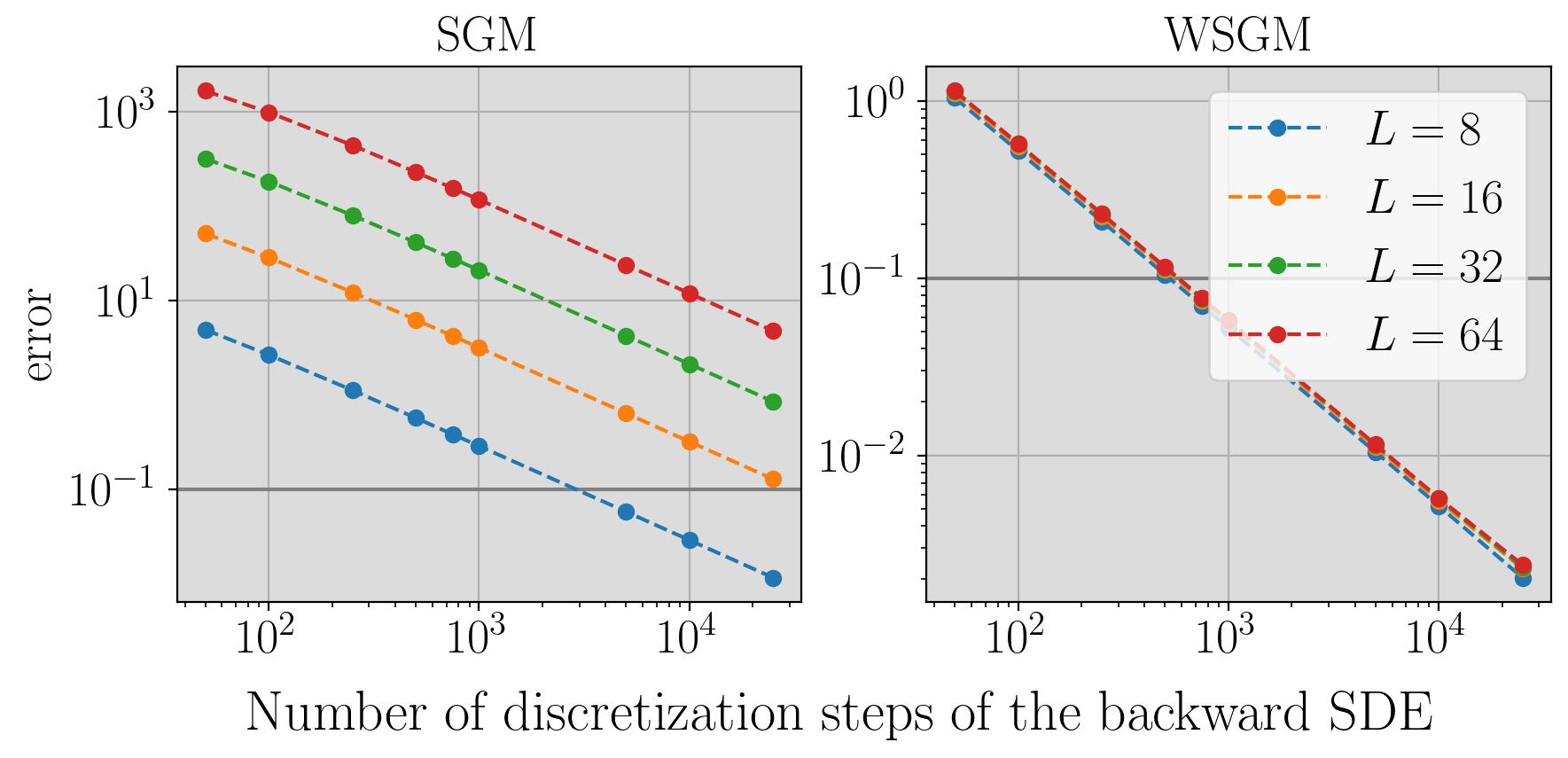}&\includegraphics[height=0.17\textheight]{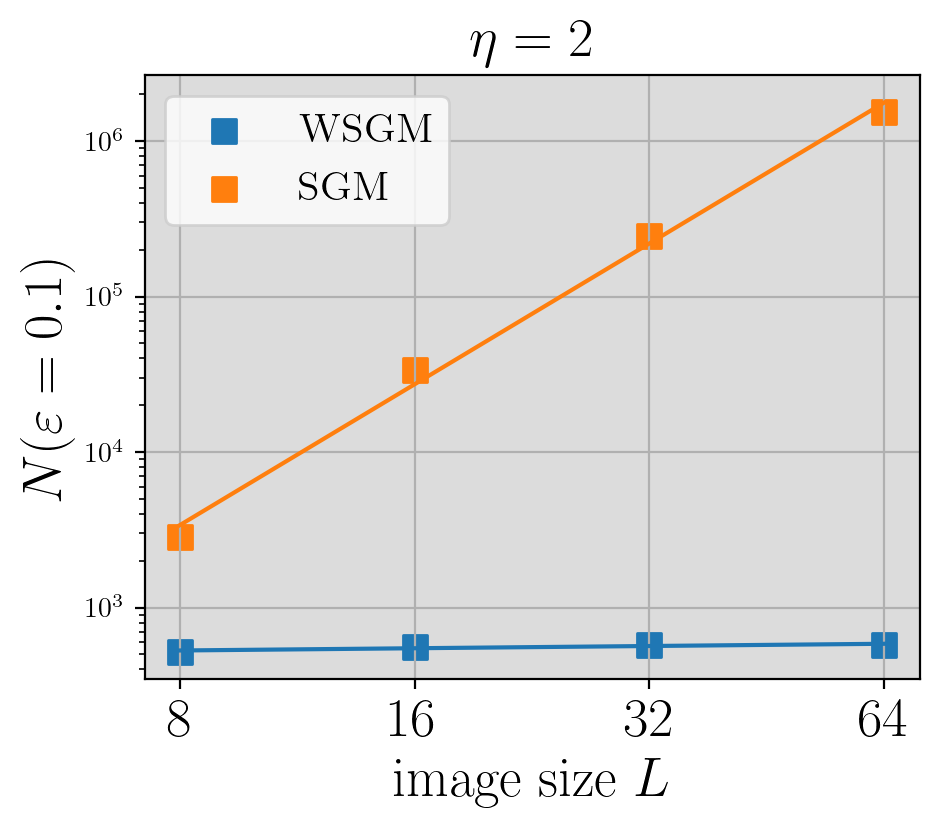}
    \end{tabular}
    \caption{\textbf{Left and middle:} evolution of the error on the estimated
      covariance matrix using either SGM or WSGM w.r.t.\ the number of stepsizes
      used in the model ($T=10$ is fixed). \textbf{Right:} number $N(\varepsilon)$ of discretization steps
      required to reach a given error $\varepsilon = 0.1$ using either SGM or WSGM.}
    \label{fig:gaussian}
\end{figure}

\section{Acceleration with WSGM: Numerical Results}
\label{sec:time-reduct-wavel}



For multiscale Gaussian processes, we proved that with WSGMs, the number of time steps $N(\varepsilon)$ to reach a fixed error $\varepsilon$ does not depend on the signal size, as opposed to SGMs. This section
shows that this result applies to non-Gaussian multiscale processes. We consider
a physical process near a phase transition and images from the CelebA-HQ database \citep{celeba-hq}.

\subsection{Physical Processes with Scalar Potentials}\label{sec:results_phi4}

Gaussian stationary processes are maximum entropy processes conditioned by second order moments defined by a circulant matrix. More complex physical processes are modeled by imposing a constraint on their marginal distribution, with a so-called scalar potential. 
The marginal distribution of $x$ is the probability distribution of $x(u)$, which does not depend upon $u$ if $x$ is stationary. Maximum entropy processes conditioned by second order moments and marginal distributions have a probability density which is a Gibbs distribution
$p (x) = Z^{-1}\, \rme^{-E (x)}$ with:
\begin{equation} \label{eq:gibbs}
  \textstyle{
    E(x) = \frac{1}{2}x^\top C x + \sum_u V(x(u)) ~~,
    }
\end{equation}
where $C$ is a circulant matrix and $V\colon\R\to\R$ is a scalar potential.
\Cref{sec:deta-varph-model} explains how to parametrize $V$ as a linear combination of a family of fixed elementary functions. The $\varphi^4$ model is a particular example where $C = \Delta$ is a Laplacian and $V$ is a fourth order polynomial, adjusted in order to imposes that $x(u) = \pm 1$ with high probability. For so-called critical values of these parameters, the resulting process becomes multiscale with long range interactions and a power law spectrum, see
\Cref{fig:phi4synthesis}-(c).

We train SGMs and WSGMs on critical $\varphi^4$ processes of different sizes; for the score model ${\bs}_\theta$, we use a simple linear parametrization detailed in \Cref{app:score_phi4}. To evaluate the quality of the generated samples, it is sufficient to verify that
these samples have the same second order moment and marginals as $\varphi^4$. We define the error metric as the sum of the $\mathrm{L}^2$ error on the power spectrum and the total-variation distance between marginal distributions. \Cref{fig:phi4synthesis}-(a) shows the decay of this error as a function of the number of time steps used in an SGM and
WSGM with a uniform discretization. With vanilla SGM, the loss has a strong dependency in $L$, but becomes almost independent of $L$ for WSGM. 
This empirically verifies the claim that an ill-conditioned covariance matrix leads to slow sampling of SGM, and that WSGM is unaffected by this issue by working with the conditional distributions of normalized wavelet coefficients.

\begin{figure}\centering
    \includegraphics[height=0.14\textheight]{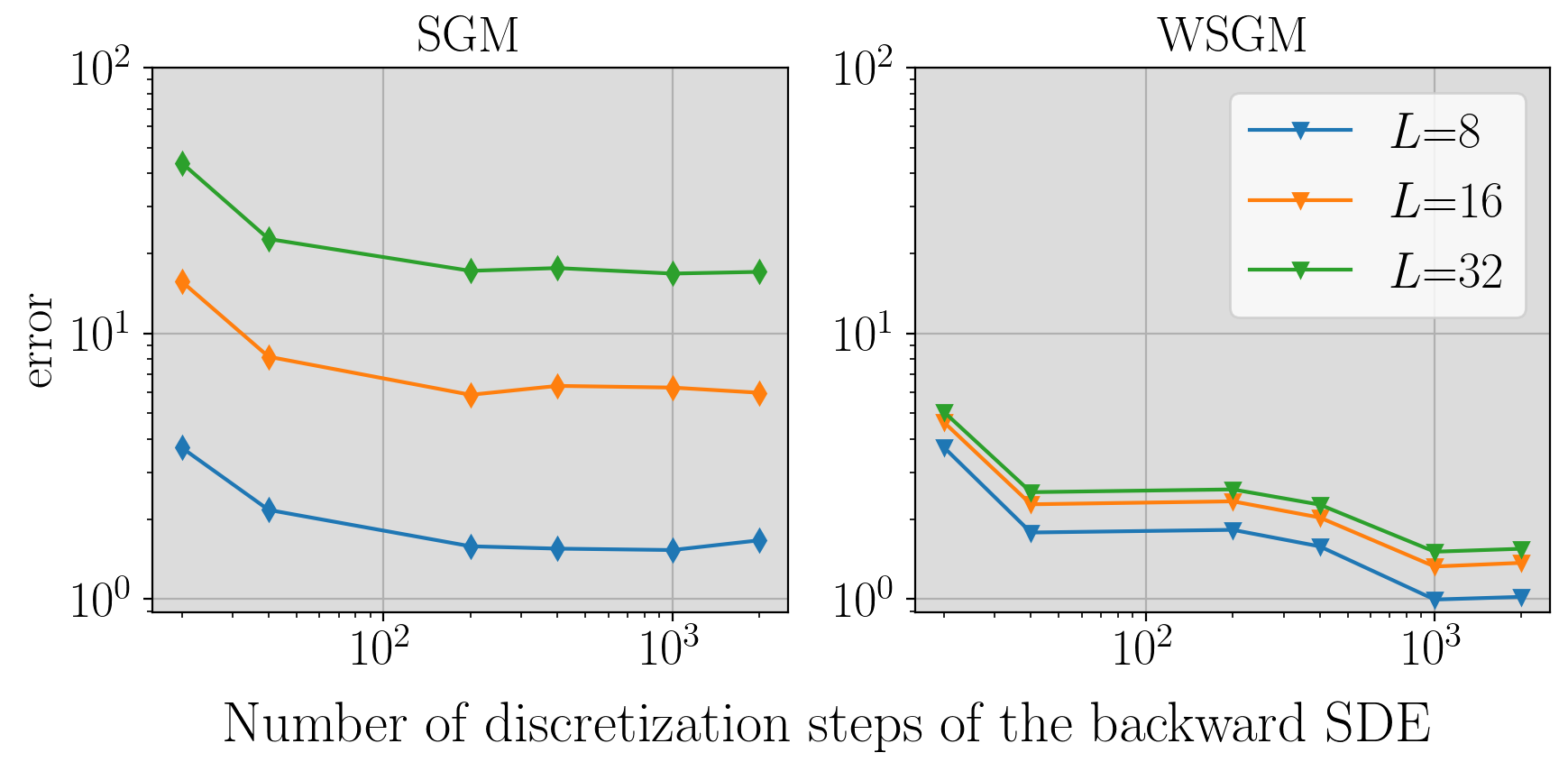} \hfill \includegraphics[height=0.14\textheight]{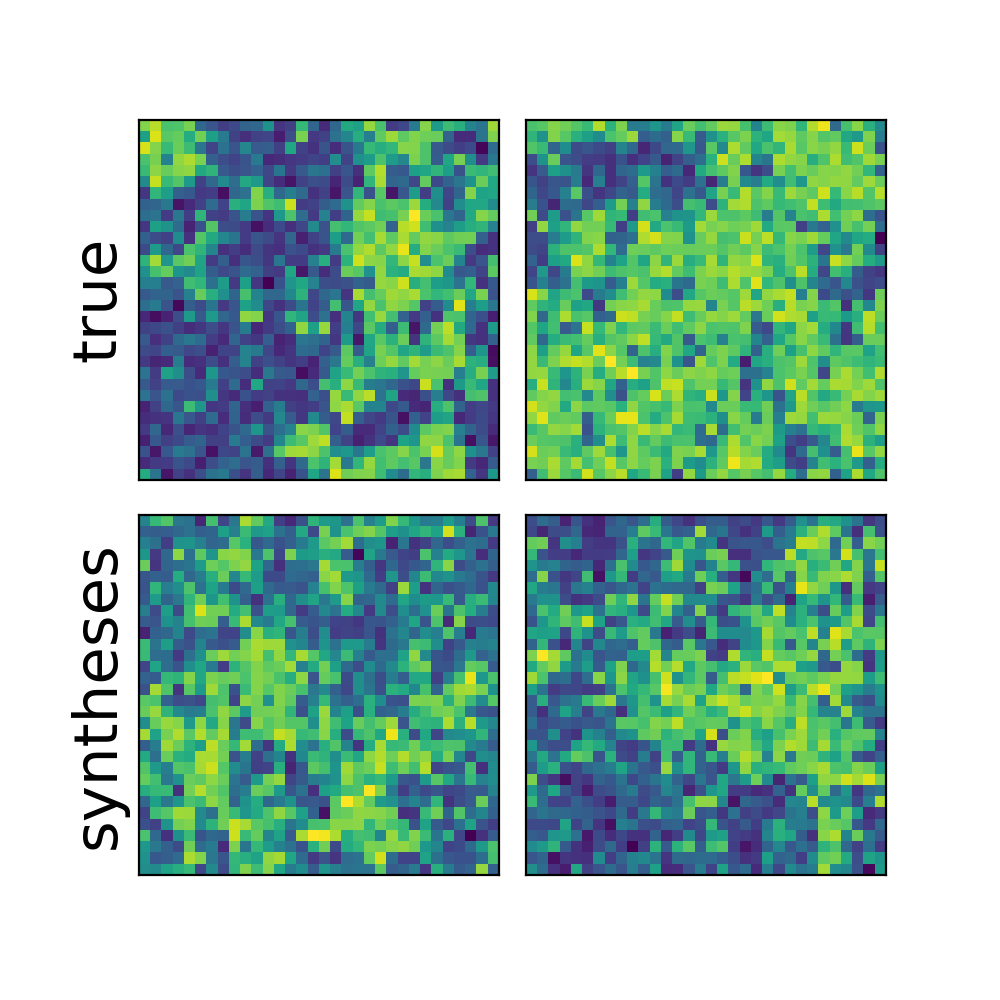} \hfill \includegraphics[height=0.14\textheight]{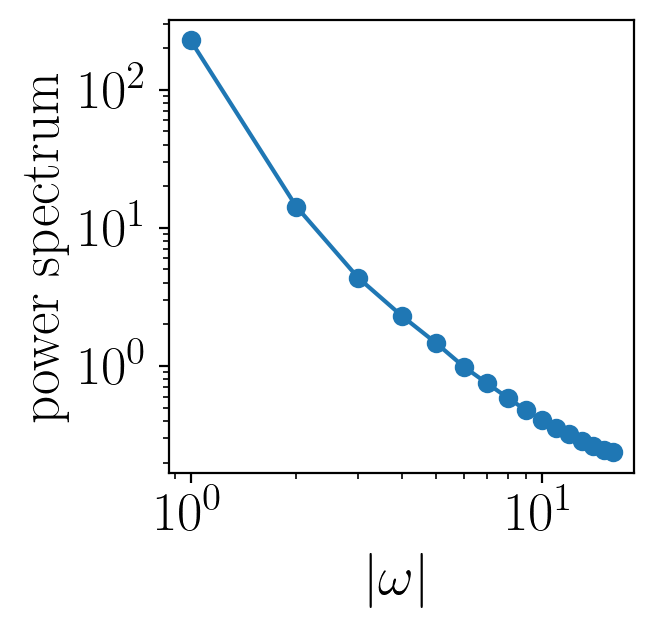}
    \caption{\textbf{Left:} error between ground-truth $\varphi^4$ datasets in various dimensions $L$, and the synthetized datasets with SGM and WSGM, for various number of discretization steps. \textbf{Middle:} realizations of $\varphi^4$ (top) and WSGM samples (bottom). \textbf{Right:} power spectrum of $\varphi^4$ for $L=32$. }
    \label{fig:phi4synthesis}
\end{figure}

\subsection{Scale-Wise Time Reduction in Natural Images }

Images are highly non-Gaussian multiscale processes whose power spectrum has a power law decay. We now show that WSGM also provides an acceleration over SGM in this case, by being independent of the image size. 

We focus on the CelebA-HQ dataset \citep{liu2015faceattributes} at a $128\times 128$ resolution. Though non-stationary, its power spectrum still has a power law decay, as shown in \Cref{fig:celeba}. We compare SGM \cite{ho2020denoising} samples at the $128 \times 128$ resolution with WSGM samples which start from the $32 \times 32$ resolution. Though smaller, the $32 \times 32$ resolution still suffers from a power law decay of its spectrum over several orders of magnitude. The reason why we limit this coarsest resolution is because border effects become dominant at lower image sizes. To simplify the handling of border conditions, we use Haar wavelets. 

Following \cite{nichol2021improved}, the global scores $\bs_\theta(x)$ are parametrized by a neural network with a UNet architecture. It has 3 residual blocks at each scale, and includes multi-head attention layers at lower scales. The conditional scores $\bs_\theta(\bar x_j | x_j)$ are parametrized in the same way, and the conditioning on the low frequencies $x_j$ is done with a simple input concatenation along channels  \citep{nichol2021improved,saharia2021image}. The details of the architecture are in \Cref{sec:exper-deta-addit}. We use a uniform discretization of the backward SDE to stay in the setting of \Cref{thm:control_diffusion}. 

The generation results are given in \Cref{fig:celeba}. With the same computational budget of $16$ discretizations steps at the largest scale (iterations at smaller scales having a negligible cost due to the exponential decrease in image size), WSGM achieves a much better perceptual generation quality. Notably, SGM generates noisy images due to discretization errors. This is confirmed quantitatively with the Fréchet Inception Distance (FID) \citep{fid}. The FID of the WSGM generations decreases with the number of steps, until it plateaus. This plateau is reached with at least 2 orders of magnitude less steps for WSGM than SGM. This number of steps is also independent of the image size for WSGM, thus confirming the intuition given in the Gaussian case by \Cref{prop:approx_gaussian_kl}-\ref{lemma:bound_spectrum}. Our results confirm that vanilla SGM on a wide range of multiscale processes, including natural images, suffers from ill-conditioning, in the sense that the number of discretization steps grows with the image size. WSGM, on the contrary, leads to uniform discretization schemes whose number of steps at each scale does not depend on the image size.

We also stress that there exists many techniques \citep{kadkhodaie2020solving,jolicoeur2021gotta,liu2022pseudo,zhang2022exponential,san2021noise,nachmani2021non,song2020denoising,kong2021fast,ho2020denoising,luhman2021knowledge,salimans2022progressive,xiao2021tackling} to accelerate the sampling of vanilla SGMs, with sometimes better FID-time complexity tradeoff curves. Notably, the FID plateaus at a relatively high value of 20 because the coarsest resolution $32 \times 32$ is still ill-conditioned and requires thousands of steps with a non-uniform discretization schedule to achieve FIDs less than 10 with vanilla SGM \cite{nichol2021improved}. Such improvements (including proper handling of border conditions) are beyond of the scope of this paper. The contribution of WSGM is rather to show the reason behind this sampling inefficiency and mathematically prove in the Gaussian setting that wavelet decompositions of the probability distribution allow solving this problem. We believe that extending this theoretical result to a wider class of non-Gaussian multiscale processes and combining WSGM with other sampling accelerations are interesting avenues for future work.

\begin{figure}
    \centering
    \includegraphics[height=0.15\textheight]{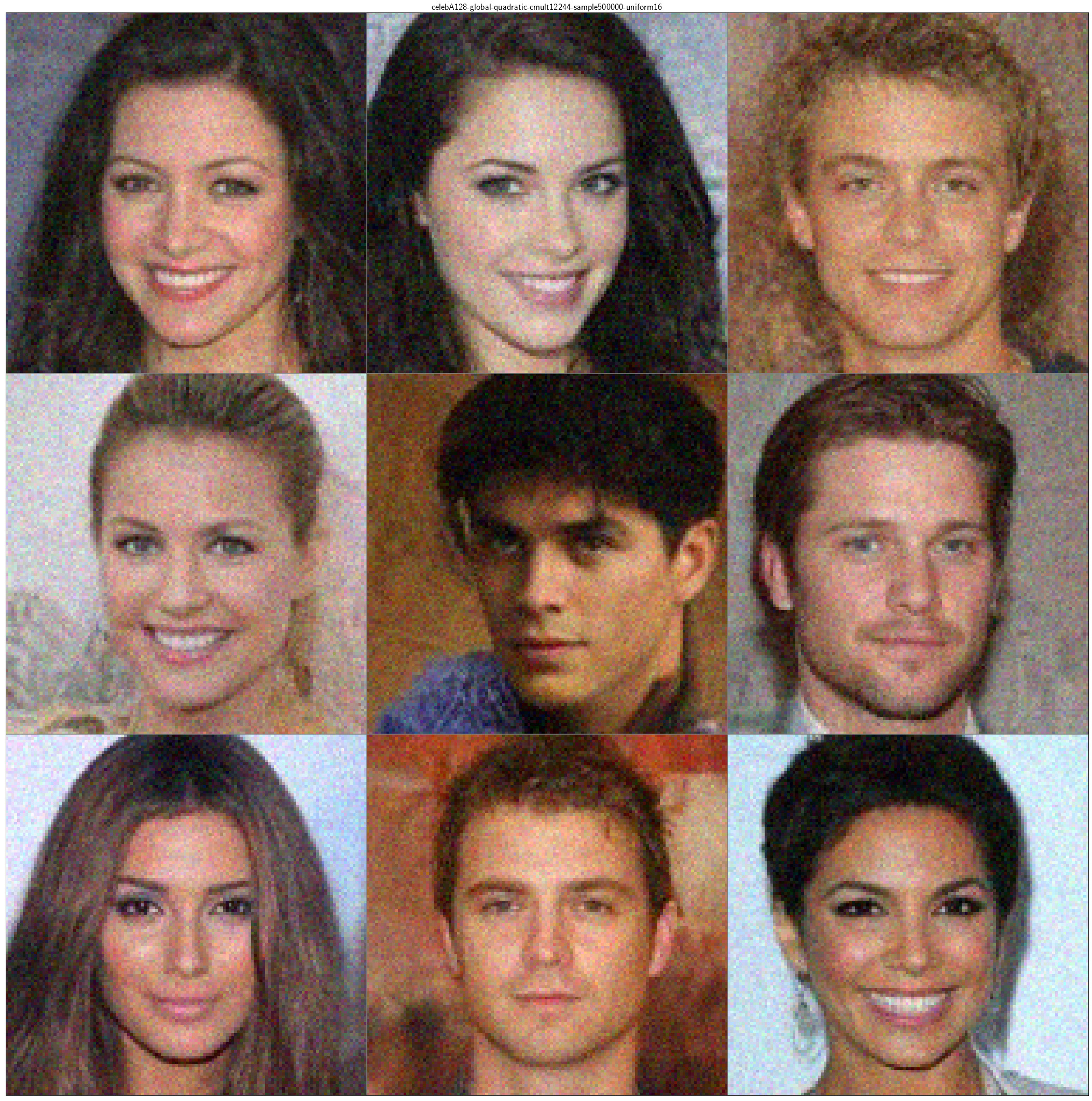} 
    \hfill
    \includegraphics[height=0.15\textheight]{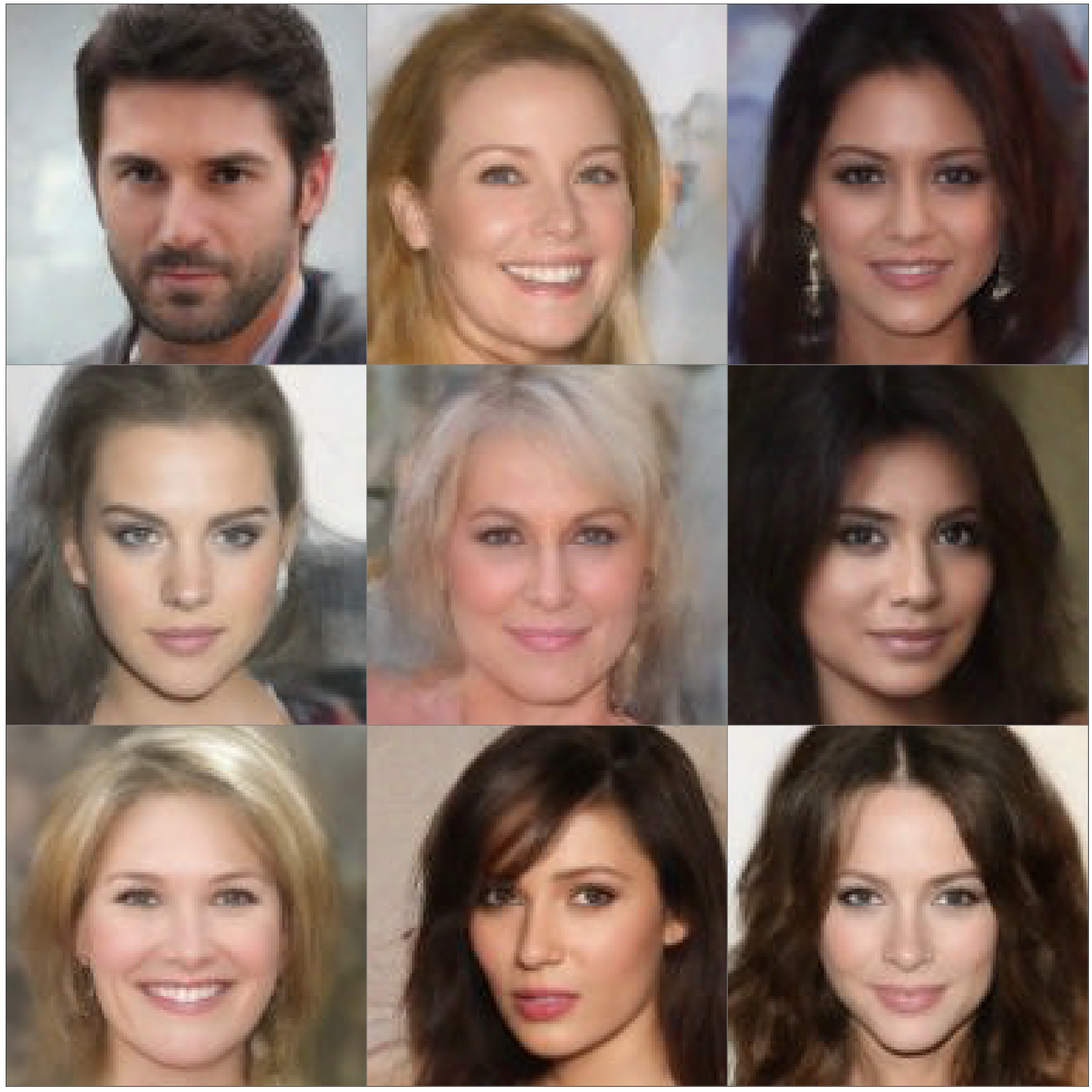}
    \hfill
    \includegraphics[height=0.12\textheight]{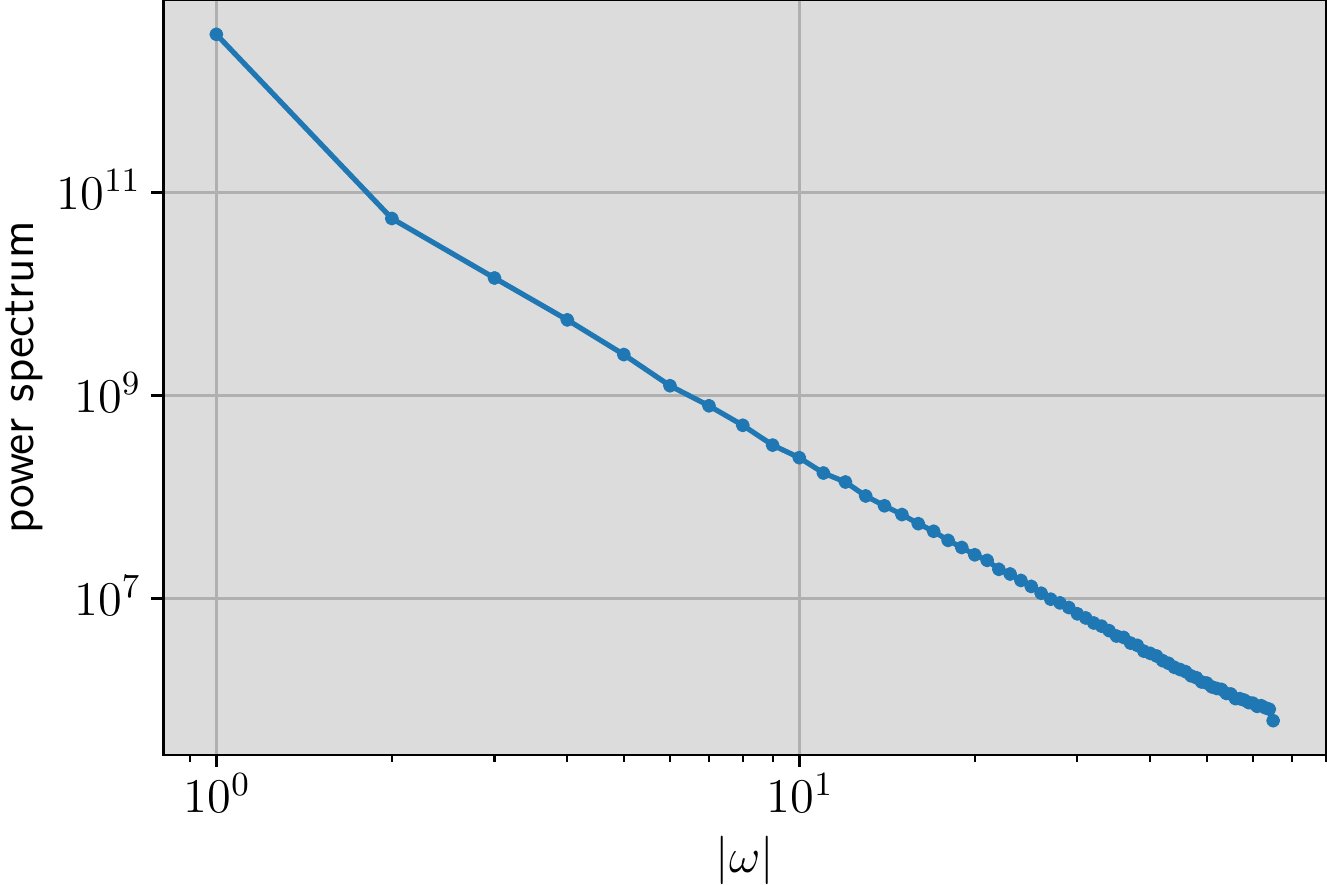}
    \\[0.5em]
    \includegraphics[height=0.12\textheight]{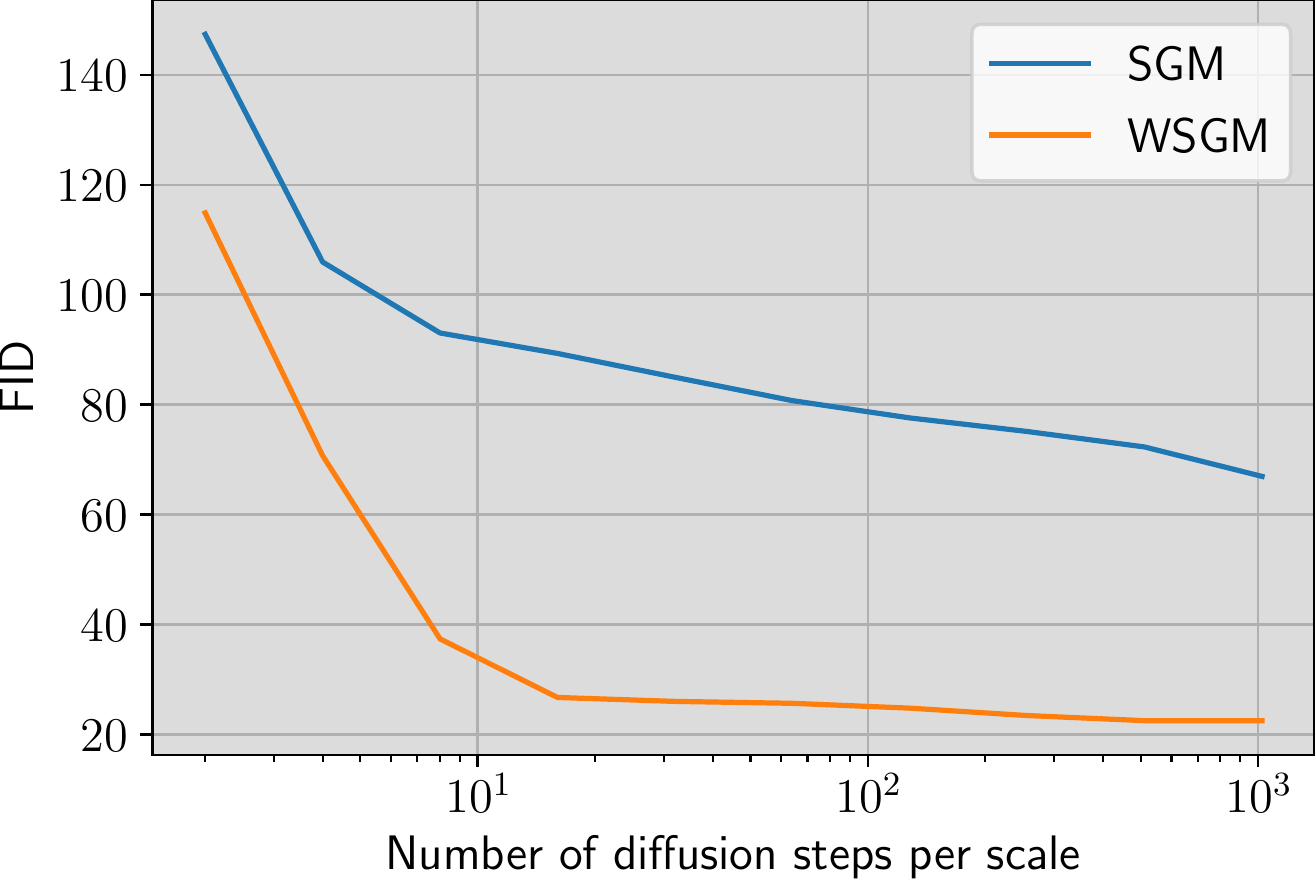}
    \hfill
    \includegraphics[height=0.12\textheight]{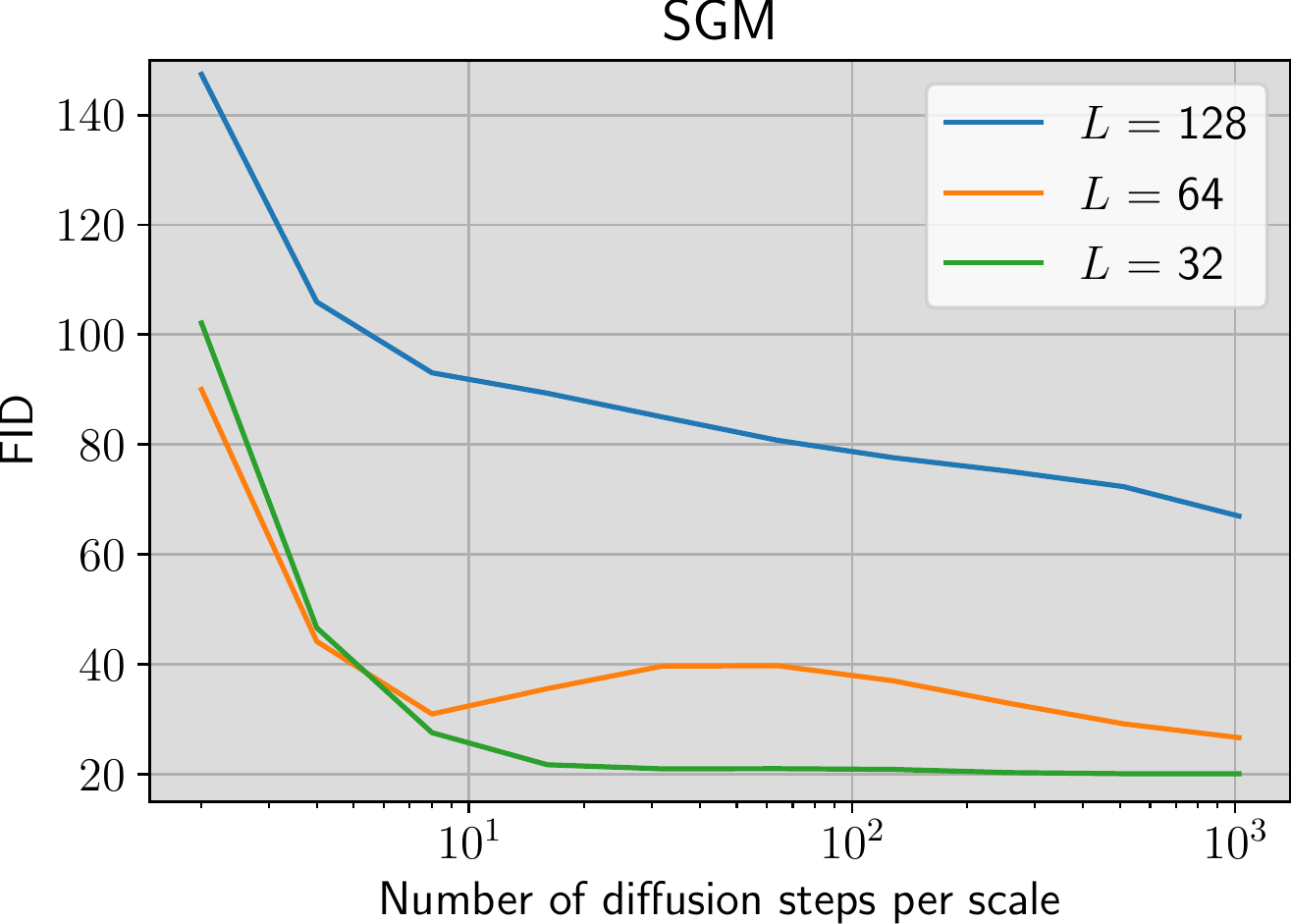}
    \hfill
    \includegraphics[height=0.12\textheight]{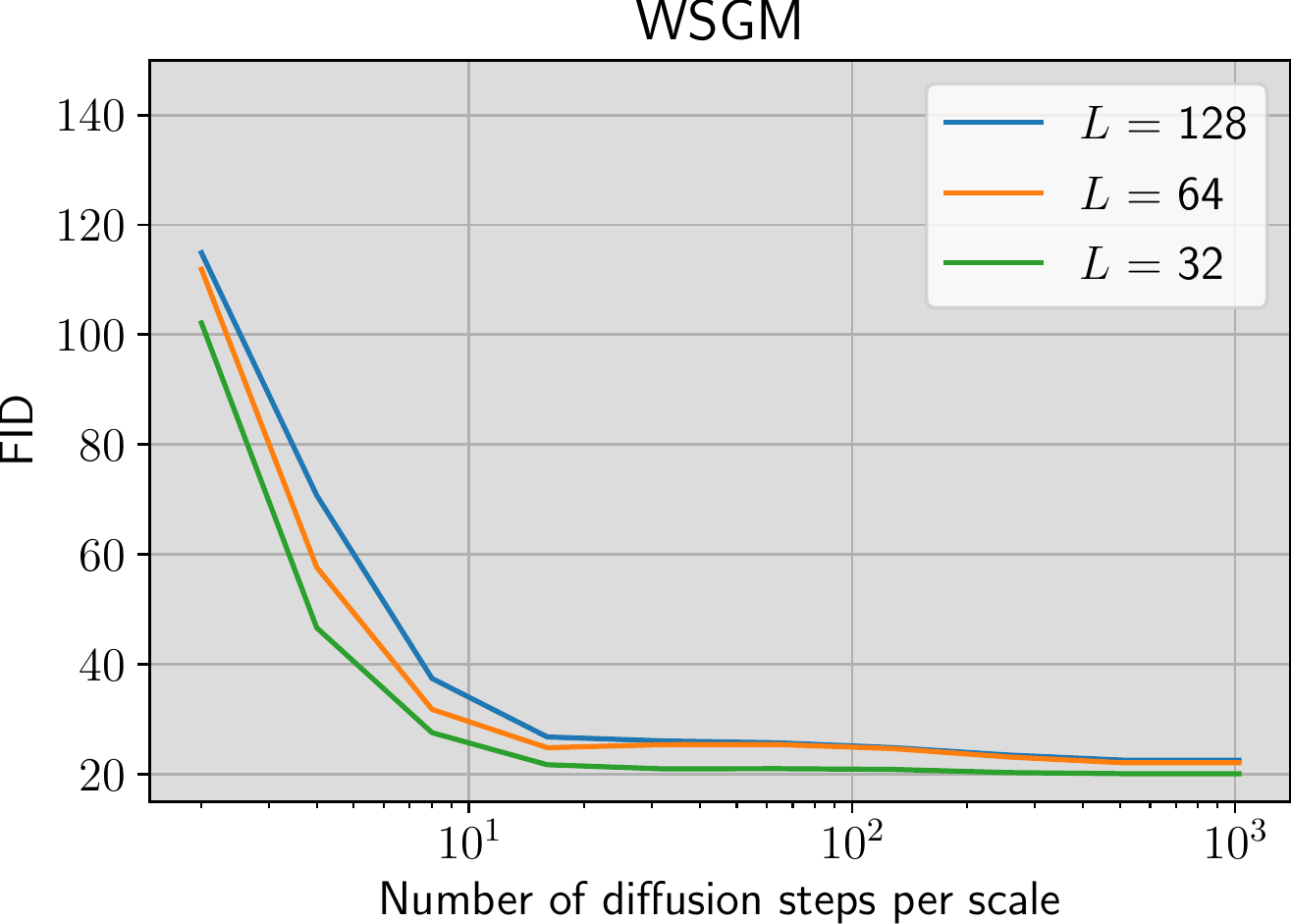}
    \caption{\textbf{Top.} (a): Generations from SGM with 16 discretization steps. (b): Generations from WSGM with 16 discretization steps at each scale. (c): Power spectrum of CelebA-HQ. \textbf{Bottom.} (a): Evolution of the FID w.r.t.\ the number of diffusion steps for SGM and WSGM with $L = 128$. (b): Evolution of the FID w.r.t.\ the number of diffusion steps for SGM at several image sizes $L$. (c) Evolution of the FID w.r.t.\ the number of diffusion steps for WSGM at several image sizes $L$.}
    \label{fig:celeba}
\end{figure}

\section{Discussion}

This paper introduces a Wavelet Score-based Generative Model (WSGM)
which applies an SGM to normalized wavelet coefficients conditioned by lower frequencies.
We prove that the number of
stepsizes in SGMs is controlled by the regularity of the score of the target
distribution. For multiscale processes such as images, it requires a considerable number
of time steps to achieve a good accuracy, which increases quickly with the image size.
We show that a WSGM eliminates ill-conditioning issues by normalizing 
wavelet coefficients. As a result, the number of stepsizes in WSGM does not increase
with the image dimension.
We illustrated our results on Gaussian
distributions, physical processes and image datasets.

One of the main limitations of WSGM is that it is limited to multiscale processes for which the conditional wavelet probabilities are close to white Gaussian distributions. A promising direction for future work is to combine WSGM with other acceleration techniques such as adaptive time discretizations to handle such cases.
In another direction, one could strengthen the theoretical
study of SGM and extend our results beyond the Gaussian setting, in
order to fully describe SGM on physical processes which can be seen as
perturbations of Gaussian distributions.

\subsubsection*{Acknowledgments}
This work was supported by a grant from the PRAIRIE 3IA Institute of the French ANR-19-P3IA-0001 program. We would like to thank the Scientific Computing Core at the Flatiron Institute for the use of their computing resources.

\bibliographystyle{apalike}
\bibliography{bibliography.bbl}

\section{Organization of the Supplemental Material}

We provide the pseudocode for the WSGM algorithm in
\Cref{sec:wsgm-pseudocode}. Details about wavelet transforms and their whitening
properties are presented in \Cref{sec:wavelets-introduction,sec:wavelet-transforms}.  The proofs of
\Cref{sec:time-sampling-score} and \Cref{sec:wavel-diff-score} are gathered
\Cref{sec:proof-section21} and \Cref{sec:proof33} respectively. Details about
the Gaussian model and the $\varphi^4$ model are given in
\Cref{app:gaussian_exps} and \Cref{sec:deta-varph-model} respectively. Finally,
experimental details and additional experiments are described in
\Cref{sec:exper-deta-addit}.

\section{WSGM Algorithm}
\label{sec:wsgm-pseudocode}

In \Cref{alg:cascaded_wavelet}, we provide the pseudo code for WSGM. Notice that the training of score models at each scale can be done in parallel, while the sampling is done sequentially one scale after the next. The sequential nature of sampling is the price to pay for the conditional decomposition, as we model only the conditional scores $\nabla_{\bar x_j} \log \bar p_j(\bar x_j | x_j)$. Generating all scales simultaneously would require learning the so-called free energy gradient $\nabla_{x_j} \log \bar p_j(\bar x_j | x_j)$. It is also required for likelihood computation. However, the conditional decomposition of WSGM allows not learning this free energy gradient and possibly reduces score estimation errors.


\begin{algorithm}[!h]
\caption{\small Wavelet Score-based Generative Model}
\label{alg:cascaded_wavelet}
\begin{algorithmic}[1]
  \small
  \Require $J$, $N_{\textrm{iter}}$, $N$, $T$,$\{\bar\theta_{j,0}, \theta_{J, 0}\}_{j=0}^{J}$, $\{x^m_0\}_{m=1}^M$
  \State{\color{commentcolor}{/// WAVELET TRANSFORM ///}}
  \For{$j \in \{1, \dots, J\}$} 
  \For{$m \in \{1, \dots, M\}$}
  \State $x_{j}^m = \gamma_j^{-1} \rmG x_{j-1}^m$,  $\bar{x}_{j}^m =  \gamma_j^{-1} \wG x_{j-1}^m$ \Comment Wavelet transform of the dataset
  \EndFor
  \EndFor
  \State{\color{commentcolor}{/// TRAINING ///}}
  \State Train score network $\bs_{\theta^\star_{J}}$ at scale $J$ with dataset $\{ x_{J}^m\}_{m=0}^M$ \Comment Unconditional SGM training
\For{$j \in \{J, \dots, 1\}$} \Comment Can be run in parallel
\For{$n \in \{0, \dots, N_{\textrm{iter}}-1\}$}
\State Sample $(\bar{x}_{j,0}, x_{j})$ from $\{\bar{x}_j^m, x_{j}^m\}_{m=1}^M$ 
\State Sample $t$ in $\ccint{0, T}$ and $\bar{Z} \sim \mathrm{N}(0, \Id)$
\State $\bar{x}_{j,t} = \rme^{-t} \bar{x}_{j,0} + (1 - \rme^{-2t})^{1/2} \bar{Z}$ 
\State $\ell(\bar\theta_{j,n}) = \normLigne{(\rme^{-t}\bar{x}_{j,0} - \bar{x}_{j,t}) - (1 - \rme^{-2t})^{1/2} \bar\bs_{\bar\theta_{j,n}}(t, \bar{x}_{j,t} | x_{j})}^2$ 
\State $\bar\theta_{j,n+1} = \verb|optimizer_update|(\bar\theta_{j,n}, \ell(\bar\theta_{j,n}))$ \Comment ADAM optimizer step 
\EndFor
\State $\bar\theta^\star_j = \bar\theta_{j, N_{\textrm{iter}}}$
\EndFor
\State{\color{commentcolor}{/// SAMPLING ///}}
\State ${\txwav}_{J} = \verb|unconditionalSGM|(T, N, \bs_{\theta_J^\star})$ \Comment Unconditional SGM sampling
\For{$j \in \{ J, \dots, 1\}$}
\State $\tox_{j} = \verb|EulerMaruyama|(T, N, \bar\bs_{\bar\theta_j^\star}(\cdot, \cdot | \txwav_{j}))$ \Comment
Euler-Maruyama recursion following \eqref{eq:backward_disc}
\State $\txwav_{j-1} = \gamma_j \rmG^\top \txwav_{j} + \gamma_j \wG^\top \tox_{j}$ \Comment Wavelet reconstruction
\EndFor
\State {\bfseries return} $\{\bar\theta^\star_{j}, \theta^\star_J\}_{j=1}^{J}$, $\txwav_{0}$
\end{algorithmic}
\end{algorithm}

\section{Introduction to Orthogonal Wavelet Bases}
\label{sec:wavelets-introduction}

This section introduces properties of orthogonal wavelet bases as specified by the operators $G$ and $\bar G$. In this section, we drop the normalizing factors $\gamma_j$ for clarity.

Let $x_0$ be a signal. The index $u$ in $x_0(u)$ belongs to an $n$-dimensional grid of linear size $L$ and hence with $L^n$ sites, with $n=2$ for images.

Let us denote $x_j$ the coarse-grained version of $x_0$ at a scale $2^j$ defined over a coarser grid with intervals $2^j$ and hence $(2^{-j} L)^n$ sites. The coarser signal $x_j$ is iteratively computed from $x_{j-1}$ by applying a coarse-graining operator, which acts as a scaling filter $G$ which eliminates high frequencies and subsamples
the grid \cite{Mallat:89}:
\begin{equation}
  \label{low-pass200}
(G x_{j-1}) (u) =  \sum_{u'} x_{j-1} (u')\, G (2u-u')\,\, .
\end{equation}
The index $u$ on the left-hand side runs on the coarser grid, whereas $u'$ runs on the finer one. 

\newcommand\aphi{x}
\newcommand\dphi{\bar x}
The degrees of freedom of $\aphi_{j-1}$ that are not in  $\aphi_{j}$ are encoded in orthogonal wavelet coefficients $\dphi_j$.
The representation $(\aphi_j , \dphi_j)$ 
is an orthogonal change of basis calculated from $\aphi_{j-1}$. The coarse signal
$\aphi_j$ is calculated in \eqref{low-pass200} with a low-pass scaling filter $G$ and a subsampling. In dimension $n$,
the wavelet coefficients $\dphi_j$ have $2^n-1$ channels computed with a convolution and subsampling operator $\bar G$. We thus have:
\begin{equation}
  \label{fastdec30}
\aphi_j = G\, \aphi_{j-1} ~~\mbox{and}~~
\dphi_j = \bar G\, \aphi_{j-1} .
\end{equation}
The wavelet filter
$\bar G$ computes $2^n-1$ wavelet coefficients $\dphi_j (u,k)$ indexed by $1 \leq k \leq 2^n-1$, 
with separable high-pass filters $\bar G_k(u)$:
\[
\dphi_{j} (u,k) =  \sum_{u'} \aphi_{j-1} (u')\, \bar G_k (2u-u').
\]

As an example, the Haar wavelet leads to a block averaging filter $G$:
\[
\aphi_j(u)= \frac{\aphi_{j-1}(2u)+\aphi_{j-1}(2u+1)}{\sqrt{2}}.
\]
In $n=1$ dimension, there is a single wavelet channel in $\dphi_j$. The corresponding wavelet filter $\bar G$ computes the wavelet coefficients with normalized increments:
\begin{equation}
    \label{Harwnsdf}
\dphi_j(u)= \frac{\aphi_{j-1}(2u)-\aphi_{j-1}(2u+1)}{\sqrt{2}}.
\end{equation}
If $n = 2$, then there are $3$ channels as illustrated in \Cref{fig1}.

The fast wavelet transform cascades \eqref{fastdec30} for $1 \leq j \leq J$ 
to compute the decomposition of the high-resolution signal $\aphi_0$ into its orthogonal wavelet representation
over $J$ scales:
\begin{equation}
\label{orthowave}
\left\{ \aphi_J\, ,\, \dphi_j \right\}_{1 \leq j \leq J} .
\end{equation}

The wavelet orthonormal filters $G$ and $\bar G$ define
a unitary transformation, which satisfies:
\begin{equation}
  \label{consdsf}
\bar G G^\top = G \bar G^\top = 0~~\mbox{and}~~G^\top G + \bar G^T \bar G = \mathrm{Id}~,
\end{equation}
where $\mathrm{Id}$ is the identity. \cite{Mallat:89} gives a condition on the Fourier transforms of $G$ and $\bar G$ to build such filters.
The filtering equations \eqref{fastdec30}
can then be inverted with the adjoint operators:
\begin{equation}
  \label{fastrec3}
  \aphi_{j-1} = G^\top \aphi_j + \bar G^\top \dphi_j~.
\end{equation}
The adjoint $G^\top$ enlarge the grid size of $\aphi_j$ by inserting a zero between each coefficients, and then filters the output:
\[
(G^\top \aphi_j)(u) = \sum_{u'} \aphi_j (u' )\, G (2u'-u) .
\]
The adjoint of $\bar G$ performs the same operations over the $2^n-1$ channels and adds them:
\[
(\bar G^\top \dphi_j )(u) = \sum_{k=1}^{2^n-1} \sum_{u'} \dphi_j (u',k )\, \bar G_k (2u'-u) .
\]
The fast inverse wavelet transform \cite{Mallat:89} recovers 
$\aphi_0$ from its wavelet representation \eqref{orthowave} 
by progressively recovering $\aphi_{j-1}$ from $\aphi_j$ and $\dphi_j$
with \eqref{fastrec3}, for $j$ going from $J$ to $1$.


\section{Orthogonal Wavelet Transform}
\label{sec:wavelet-transforms}

This appendix relates the fast discrete wavelet transform to decomposition of
finite energy functions in orthonormal bases of $\Ld([0,1]^n)$.  We shall then
prove that the covariance of normalized wavelet coefficients is a well-conditioned matrix, independently from scales. This is a central result to prove \Cref{lemma:bound_spectrum}.  The results of this appendix are based on
the multiresolution theory \cite{mallat1999wavelet} and the representation of
singular operators in wavelet orthonormal bases \cite{Meyer:92c}.

\paragraph{Orthonormal wavelet transform}
From an input discrete signal $x_0 (u) = x(u)$ defined over an $n$-dimensional
grid of width $L$, we introduced in \eqref{eq:low_res}
a normalized wavelet transform which
computes wavelet coefficients $\ox_j(u,k)$ having
$2^n -1$ channels $1 \leq k < 2^n$.
A non-normalized orthonormal wavelet transform is obtained 
by setting $\gamma_j = 1$:
\begin{equation}
  \label{fastdec3}
x^w_j = G\, x^w_{j-1} ~~\mbox{and}~~
\ox^w_j = \oG\, x^w_{j-1} ,
\end{equation}
with $x^w_0 = x_0$ and where $G$ is a low-pass filtering and subsampling by $2$:
\begin{equation}
  \label{low-pass20}
  \textstyle{
    G x^w_{j-1} (u) =  \sum_{u'} x_{j-1} (u')\, G (2u-u')\,.
    }
\end{equation}
The wavelet convolutional and subsampling operator
$\oG$ computes $2^n-1$ wavelet signals $\ox^w_j (u,k)$, 
with  $2^n-1$ high-pass filters $\oG_k(u)$ indexed by $1 \leq k \leq 2^n-1$: 
\[
  \textstyle{
    \ox^w_{j-1} (u,k) =  \sum_{u'} x^w_{j-1} (u')\, \oG_k (2u-u').
    }
\]
We write $\ox^w = (\ox^w_j , x^w_J)_{j \leq J}$ the vector of non-normalized wavelet coefficients.

The multiresolution wavelet theory \cite{mallat1999wavelet,Meyer:92c} proves
that the coefficients of $\ox^w$ can also be written as the decomposition
coefficients of a finite energy function, in a wavelet orthonormal basis of the
space $\Ld(\R^n)$ of finite energy functions.  These wavelets arise from the
cascade of the convolutional filters $G$ and $\oG$ in \eqref{fastdec3} when we
iterate on $j$ \cite{mallat1999wavelet}.  This wavelet orthonormal basis is thus
entirely specified by the choice of the filters $G$ and $\oG$.  A wavelet
orthonormal basis is defined by a \emph{scaling function} $\psi^0 (v)$ for
$v \in \R^n$ which has a unit integral $\int {\psi^0 (v)} \,\rmd v = 1$, and $2^n-1$
\emph{wavelets} which have a zero integral $\int \psi^k(v)\, \rmd v = 0$ for
$1 \leq k < 2^n$.  Each of these functions are dilated and translated by
$u \in \Z^n$, for $1 \leq k <2^n$ and $j \in \Z$:
\begin{equation}
  \label{wavensdof}
\psi^k_{j,u} (v) = 2^{- nj/2}\, \psi^k (2^{-j} v - u).    
\end{equation}
The main result  proved in \cite{mallat1999wavelet,Meyer:92c},
is that for appropriate filters $G$ and $\bar G$ such that
$(G,\oG)$ is unitary, 
the family of translated and dilated wavelets
up to the scale $2^J$:
\begin{equation}
  \label{waveletbasis}
  \textstyle{
    \{ \psi^0_{J,u}\,,\,\psi^k_{j,u}  \}_{1 \leq k < 2^n\,,\,j \leq J \,,\,u \in \Z^n}
    }
\end{equation}
is an orthonormal basis of ${\bf L^2}(\R^n)$.
A periodic wavelet basis of ${\bf L^2}([0,1]^n)$ is defined
by replacing each  wavelet $\psi^k_{j,u}$ by the periodic function
$\sum_{r \in \Z^n} \psi^k_{j,u}(v- r)$ which we shall still
write $\psi^k_{j,u}$. 

The properties of the wavelets $\psi^k_{j,u}$ depend upon the choice of the
filters $G$ and $\oG$. If these filters have a compact support then one can
verify \cite{mallat1999wavelet} 
that all wavelets $\psi^k_{j,u}$ have a compact support of size proportional
to $2^j$.
With an appropriate choice of filters, one can also define wavelets having
$q$ vanishing moments, which means that they are orthogonal to any polynomial $Q(v)$ of degree strictly smaller than $q$:
\[
  \textstyle{
    \int_{[0,1]^n} Q(v)\,\psi_{j,u}^k (v)\, \rmd v = 0 .
    }
\]
One can also ensure that wavelet are $q$ times continuously differentiable.
Daubechies wavelets \cite{mallat1999wavelet} are examples of orthonormal wavelets which can have $q$ vanishing moments and be ${\bf C}^q$ for any $q$.

The relation between the fast wavelet transform and these wavelet orthonormal
bases proves \cite{mallat1999wavelet} that any discrete signal $x_0 (u)$ of
width $L$ can be written as a discrete approximation at a scale $2^{\ell}=L^{-1}$ ($\ell < 0$) of a 
(non-unique) function $f \in \Ld([0,1]^n)$. The support of $f$ is normalized whereas the
approximation scale $2^\ell$ decreases as the number of samples $L$ increases. The coefficients $x_0(u)$ are inner
products of $f$ with the orthogonal family of scaling functions at the scale
$2^{\ell}$ for all $u \in \Z^n$ and $2^{\ell} u \in [0,1]^n$:
\[
  \textstyle{
    x_0 (u) = \int_{[0,1]^n} f(v)\, \psi^0_{\ell,u} (v) \, \rmd v = \lb f , \psi^0_{\ell,u}\rb .
    }
\]
Let $V_{\ell}$ be the space generated by the orthonormal family of scaling functions
$\{ \psi^0_{\ell,u} \}_{2^{\ell} u \in [0,1]^n}$, and  $P_{V_{\ell}} f$ be
the orthogonal projection  of $f$ in $V_{\ell}$.
The signal $x_0$ gives the orthogonal 
decomposition coefficients of  $P_{V_{\ell}} f$ in this
family of scaling functions. One can prove \cite{mallat1999wavelet} that
the non-normalized wavelet coefficients $\ox^w_j$ of $x_0$ computed with a fast
wavelet transform are equal to the orthogonal wavelet coefficients of $f$
at the scale $2^{j+\ell}$, for all $u \in \Z^n$ and $2^{j+\ell} u \in [0,1]^n$:
\[
  \textstyle{
    \ox_j^w (u,k) = \int_{[0,1]^n} f(v)\, \psi^k_{j+\ell,u} (v) \, \rmd v = \lb f , \psi^k_{j+\ell,u}\rb .
    }
\]
and at the largest scale $2^J$
\[
  \textstyle{
    x_J^w (u,k) = \int_{[0,1]^n} f(v)\, \psi^0_{J+\ell,u} (v) \, \rmd v = \lb f , \psi^k_{j+J,u}\rb .
    }
\]

\paragraph{Normalized covariances}
We now consider a periodic stationary multiscale random process $x(u)$ of width $L$.
It covariance is diagonalised in a Fourier basis
and its power spectrum (eigenvalues) has a power-law decay
$P(\om) = c (\xi^\eta + |\om|^\eta)^{-1}$, for frequencies $\om = 2 \pi m / L$
with $m \in \{0, \dots, L-1\}^n$.  The following lemma proves that the covariance matrix
$\bar \Sigma$ of the normalized wavelet coefficients $\ox$ of $x$ is well
conditioned, with a condition number which does not depend upon $L$. It relies on an equivalence between Sobolev norms and weighted norms in a wavelet orthonormal basis.

\begin{lemma}
  \label{sec:norm-covar}
  For a wavelet transform corresponding to wavelets having $q > \eta$ vanishing
  moments, which have a compact support and are $q$ times continuously differentiable, there exists
  $C_2 \geq C_1 > 0$ such that for any $L$ the covariance $\bar \Sigma$
  of $\bar x = (\ox_j , x_J)_{j \leq J}$ satisfies:
\begin{equation}
    \label{stability}
  C_1\, \Id \leq \bar \Sigma \leq C_2 \,\Id .
\end{equation}
\end{lemma}

The remaining of the appendix is a proof of this lemma.  Without loss of
generality, we shall suppose that ${\mathbb E}[x] = 0$.  Let $\sigma_{j,k}^2$ be
the variance of $\ox^w_{j} (u,k)$, and $D$ be the diagonal matrix whose
diagonal values are $\sigma_{j,k}^{-1}$ .  The vector of normalized wavelet
coefficients $\ox = (\ox_j , x_J)_{j \leq J}$ are related to the non-normalized
wavelet coefficients $\ox^w$ by a multiplication by $D$:
\[
\ox = D\, \ox^w .
\]
 Let $\bar \Sigma_w$ be the covariance of $\ox^w$.  It
results from this equation that the covariance $\bar \Sigma$ of $\ox$ and the
covariance $\bar \Sigma_w$ of $\ox^w$ satisfy:
\begin{equation}
    \label{xwx}
\bar \Sigma = D \bar \Sigma_w  D .
\end{equation}
The diagonal normalization $D$ is adjusted so that the variance of each coefficient of $\ox$ is equal to $1$, which implies that
the diagonal of $\bar \Sigma$ is the identity. We must now prove that $\bar \Sigma$
satisfies \eqref{stability}, which is equivalent to prove
that there exists $C_1$ and $C_2$ such that:
\begin{equation}
\label{intermdres}
C_1\, \Id \leq  D \bar \Sigma_w  D \leq C_2\, \Id .
\end{equation}
To prove \eqref{intermdres},
we relate it to Sobolev norm equivalences that have been
proved in harmonic analysis. We begin by stating the result on Sobolev inequalities and then prove that it implies \eqref{intermdres} for appropriate constants $C_1$ and $C_2$.

Let $\Sigma_\infty$ be the singular self-adjoint
convolutional operator over $\Ld(\R^n)$
defined in the Fourier domain for all $\om \in \R^n$:
\[
\widehat{\Sigma_\infty f}(\om) = \hat f(\om) \,(\xi^\eta + |\om|^\eta) .
\]
Observe that:
\[
  \textstyle{
\lb \Sigma_\infty f , f \rb = \frac 1 {(2 \pi)^n} \int_{\R^n}
|\hat f(\om)|^2 \,(\xi^\eta + |\om|^\eta)\, \rmd\om 
}
\]
is a Sobolev norm of exponent $\eta$.  Such Sobolev norms are equivalent to
weighted norms in wavelet bases, as proved in Theorem 4, Chapter 3 in
\cite{Meyer:92c}. To take into account the constant $\xi$, we introduce
a maximum scale $2^{J'} = \xi^{-1}$. 
For all $f \in \Ld(\R^n)$, there exists $B \geq A > 0$ such that:
\begin{align}
  \label{Sobolev0}
  &A\, \lb \Sigma_\infty f , f \rb \leq   \textstyle{
    \sum_{u \in \Z^n} 2^{-J \eta} |\lb f , \psi^0_{{J'},u} \rb|^2} \\
  &\qquad \qquad \qquad \textstyle{+
  \sum_{j =-\infty}^{J'} \sum_{u \in \Z^n} \sum_{k=1}^{2^n-1}
  2^{-j \eta}\, |\lb f , \psi^k_{j,u} \rb|^2 \leq 
  B\, \lb \Sigma_\infty f , f \rb .}
\end{align}
The remaining of the proof shows that these inequalities imply similar inequalities
over the covariance of discrete wavelet coefficients. This is done by first restricting it to a finite support and then using the correspondence between the orthonormal wavelet coefficients of
$f$ and the discrete wavelet coefficients $\ox^w$ of $x_0$. 

One can verify that the equivalence \eqref{Sobolev0} remains valid for 
functions $f \in \Ld([0,1]^n)$ decomposed over periodic wavelet bases,
because functions in $\Ld([0,1]^n)$ can be written $f(v) = \sum_{r \in \Z^n} \tilde f(v-r)$
with $\tilde f \in \Ld(\R^n)$:
\begin{align}
  \label{Sobolev2}
&\textstyle{A\, \lb \Sigma_\infty f , f \rb \leq 
  \sum_{2^j u \in [0,1]^n} 2^{-{J'} \eta} |\lb f , \psi^0_{{J'},u} \rb|^2} \\
  & \qquad \qquad \qquad \textstyle{+
 \sum_{j =-\infty}^{J'} \sum_{2^j u \in [0,1]^n} \sum_{k=1}^{2^n-1}
2^{-j \eta}\, |\lb f , \psi^k_{j,u} \rb|^2 \leq 
B\, \lb \Sigma_\infty f , f \rb .}
\end{align}
Applying this result to $f \in V_{\ell}$ is equivalent to restricting
$\Sigma_\infty$ to $V_{\ell}$, which proves that $\Sigma_\ell = P_{V_\ell} \Sigma_\infty P_{V_\ell}$ satisfies:
\begin{align}
  \label{Sobolev4}
&\textstyle{A\, \lb \Sigma_\ell f , f \rb \leq 
  \sum_{2^{J'} u\in [0,1]^n} 2^{-{J'} \eta} |\lb f , \psi^0_{{J'},u} \rb|^2 }
\\ & \qquad \qquad \qquad \textstyle{+
 \sum_{j=\ell+1}^{J'} \sum_{2^j u \in [0,1]^n} \sum_{k=1}^{2^n-1}
2^{-j \eta}\, |\lb f , \psi^k_{j,u} \rb|^2 \leq 
B\, \lb  \Sigma_\ell f , f \rb .}
\end{align}
The operator $\Sigma_\ell = P_{V_\ell} \,\Sigma_\infty\,P_{V_\ell}$
is covariant with respect to shifts by any $m 2^{\ell}$ for $m \in \Z^n$
because $P_{V_\ell}$ and $\Sigma_\infty$ are covariant to such shifts. 
Its representation in the basis of scaling functions $\{ \psi^0_{\ell,u} \}_{2^{\ell} u \in  [0,1]^n}$ is thus a Toeplitz matrix which is diagonalized by a Fourier transform.
There exists $0 < A_1 \leq B_1$ such that for all $\ell < 0$ and all 
$\om \in [-2^{-\ell} \pi , 2^{-\ell} \pi]^n$,
\begin{equation}
    \label{spec-rest}
A_1 \, (\xi^\eta + |\om|^\eta) \leq P_\ell (\om) \leq B_1\, (\xi^\eta + |\om|^\eta) .
\end{equation}
Indeed, the spectrum of $\Sigma_\infty$ is $c\, (\xi^\eta + |\om|^\eta)$ for $\om \in \R^n$
and $P_{\V_\ell}$ performs a filtering with the scaling function $\psi^0_\ell$ whose
support is essentially restricted to the frequency interval $[- \pi 2^{-\ell} , \pi 2^{-\ell}]$ so that the spectrum of $P_{V_\ell} \,\Sigma_\infty\,P_{V_\ell}$ is equivalent
to the spectrum of $\Sigma_\infty$ restricted to this interval. 

The lemma hypothesis supposes that the covariance $\widetilde \Sigma$ of $x_0$
has a spectrum equal to $c\, (\xi^\eta + |\om|^\eta)^{-1}$ and hence that the
spectrum of $\widetilde \Sigma^{-1}$ is $c^{-1}\, (\xi^\eta + |\om|^\eta)$.
Since $x_0$ are decomposition coefficients of $f \in \V_\ell$ in the basis of
scaling functions, \eqref{spec-rest} can be rewritten for any $f \in \V_\ell$:
\begin{equation}
    \label{spec-rest3}
A_1\,c\,\lb  \widetilde \Sigma^{-1} x_0 , x_0\rb \leq \lb  \Sigma_\ell f , f \rb
\leq B_1\, c \,\lb  \widetilde \Sigma^{-1} x_0 , x_0\rb .
\end{equation}
Since the orthogonal wavelet coefficients $\bar x^w$ defines an orthonormal representation
of $x_0$, the covariance $\bar \Sigma_w$ of $\bar x^w$ satisfies
$ \lb \bar \Sigma_w^{-1} \bar x^w , \bar x^w\rb = 
\lb  \widetilde \Sigma^{-1} x_0 , x_0\rb $. 
Moreover, we saw that that the wavelet coefficients $\bar x^w$ of $x_0$ satisfy
$\ox^w_j (u,k) = \lb f , \psi^k_{j+\ell,u}\rb$ and at the largest scale
$\ox^w_J (u,k) = \lb f , \psi^0_{J+\ell,u}\rb$. Hence for $J +\ell = J'$,
we derive from
\eqref{Sobolev4} and \eqref{spec-rest3} that:
\begin{align}
  \label{Sobolev}
&\textstyle{A\,A_1\,c\, \lb \bar \Sigma_w^{-1} \bar x^w , \bar x^w \rb \leq 
  \sum_{2^J u\in [0,1]^n} 2^{-(J+\ell) \eta} \,| x^w_J (u)|^2 }
\\ & \qquad \qquad \qquad \textstyle{+
 \sum_{j=1}^J \sum_{2^j u \in [0,1]^n} \sum_{k=1}^{2^n-1}
2^{-(j+\ell) \eta}\, |\ox^w_j (u,k)|^2 \leq 
B\,B_1\,c\, \lb \bar \Sigma_w^{-1} \bar x^w , \bar x^w \rb .}
\end{align}
It results that for $A_2 = A\,A_1\,c$ and $B_2 = B\,B_1\,c$ we have:
\begin{equation}
    \label{nonsdf0asd}
A_2\,\lb \bar \Sigma_w ^{-1}\bar x^w , \bar x^w \rb \leq
2^{-(J+\ell) \eta}\, \|x^w_J \|^2 + \sum_{j=1}^J
2^{-(j+\ell) \eta}\, \|\ox_j^w \|^2
\leq B_2\,\lb \bar \Sigma_w^{-1} \bar x^w , \bar x^w \rb .
\end{equation}
Let $\tilde D$ be the diagonal operator over the wavelet coefficients $\ox^w$,
whose diagonal values are $2^{-\eta (j+\ell)/2}$ at all scales $2^j$.
These inequalities can be rewritten as operator inequalities:
\begin{equation}
    \label{nonsdf0asd2}
A_2 \,\bar \Sigma_w^{-1} \leq \tilde D^{2} \leq B_2\, \bar \Sigma_w^{-1},
\end{equation}
and hence:
\begin{equation}
  \label{Dinequal0}
A_2 \Id \leq \tilde D \bar \Sigma_w \tilde D \leq B_2 \Id .
\end{equation}

Since $D^{-2}$ is the diagonal of $\bar \Sigma_w$, we derive from \eqref{Dinequal0} that:
\begin{equation}
  \label{Dinequal}
A_2\,  \tilde D^{-2} \leq   D^{-2} \leq B_2\,  \tilde D^{-2} .
\end{equation}
Inserting this equation in \eqref{Dinequal0}
proves that:
\[
 A_2\,B_2^{-1}\, \Id \leq D\,\bar \Sigma_w\,D \leq B_2\,A^{-1}_2 \,\Id ,
\]
and since $\bar \Sigma = D\,\bar \Sigma_w\,D$ it proves
the lemma result \eqref{intermdres}, with $C_1 = A_2 B^{-1}_2$ and $C_2 = B_2 A_2^{-1}$.

\section{Proof of \Cref{sec:time-sampling-score}}
\label{sec:proof-section21}

In this section, we first present the continuous-time framework in a Gaussian
setting in \Cref{sec:gaussian-setting-1}. The general outline of the proof of
\Cref{prop:approx_gaussian_kl} is presented in
\Cref{sec:conv-results-discr}. Technical lemmas are gathered in
\Cref{sec:technical-lemmas}. The proof of \Cref{thm:control_diffusion} is
presented in \Cref{sec:main-content}.

\subsection{Gaussian setting}
\label{sec:gaussian-setting-1}

In what follows we present the Gaussian setting used in the proof of
\Cref{prop:approx_gaussian_kl}. We assume that $p_0 = \mathrm{N}(0, \Sigma)$
with $\Sigma \in \mathbb{S}_d(\rset)_+$.  Let $\rmD \in \mathcal{M}_d(\rset)_+$
a diagonal positive matrix such that $\Sigma = \rmP^\top \rmD \rmP$ with $\rmP$
an orthonormal matrix. We consider the following forward dynamics
\begin{equation}
  \rmd x_t = - x_t \rmd t + \sqrt{2} \rmd w_t ,
\end{equation}
with $\mathcal{L}(x_0) = p_0$.
We also consider the backward dynamics given by
\begin{equation}
  \rmd y_t = \{ y_t + 2\nabla \log p_{T-t}(y_t)\} \rmd t + \sqrt{2} \rmd w_t ,
\end{equation}
with $\mathcal{L}(y_0) = p_\infty = \mathrm{N}(0, \Id)$. Note that since for any
$t \in \ccint{0,T}$ and $x \in \rset^d$, $\nabla \log p_t(x) = - \Sigma_t^{-1}x$
with $\Sigma_t = \exp[-2t]\Sigma + (1-\exp[-2t]) \Id$, we have that
$(y_t)_{t \in \ccint{0,T}}$ is a Gaussian process. In particular, we can compute
the mean and the covariance matrix of $y_t$ in a closed form for any
$t \in \ccint{0,T}$. The results of \Cref{sec:gaussian-setting-continuous} will
not be used to prove \Cref{prop:approx_gaussian_kl}. However, they provide some
insights regarding the evolution of the mean and covariance of the backward
process.

\begin{proposition}
  \label{sec:gaussian-setting-continuous}
  For any $t \in \ccint{0,T}$, we have that
  $\mathcal{L}(y_t) = \mathrm{N}(0, \bar{\Sigma}_t)$ with
  \begin{equation}
    \bar{\Sigma}_t = \rmP^\top ((1 - \exp[-2t])\bar{\rmD}_t + \exp[-2t] \bar{\rmD}_t^2) \rmP,
  \end{equation}
  and
  \begin{equation}
    \bar{\rmD}_t = (\Id + (\rmD - \Id)\exp[-2(T-t)]) \oslash (\Id + (\rmD - \Id)\exp[-2T]) \eqsp . 
  \end{equation}
\end{proposition}

Note that $\bar{\rmD}_0 = \Id$ and
$\bar{\rmD}_T = \rmD \oslash (\Id + (\rmD - \Id)\exp[-2T]) \approx \rmD$. Hence,
we have $\bar{\Sigma}_T \approx \Sigma$ and therefore
$\mathcal{L}(y_T) \approx p_0$.

\begin{proof}
  First, note that for any $t \in \ccint{0,T}$ we have that
  \begin{equation}
    \textstyle{
      y_t = y_0 + \int_0^t ( \Id - 2\Sigma_{T-t}^{-1}) y_t + \sqrt{2} w_t = y_0 + \int_0^t ( \Id - 2\rmP^\top D_{T-t}^{-1} \rmP ) y_t + \sqrt{2} w_t ,
      }
  \end{equation}
  with $\rmD_{T-t} = \exp[-2(T-t)] \rmD + (1 - \exp[-2(T-t)]) \Id$. Denote
  $\{y_t^\rmP\}_{t \in \ccint{0,T}} = \{\rmP y_t\}_{t \in
    \ccint{0,T}}$. Using that $\rmP^\top \rmP = \Id$, we have that for any
  $t \in \ccint{0,T}$
  \begin{equation}
    \textstyle{
      y_t^\rmP =y_0^\rmP + \int_0^t ( \Id - 2 D_{T-t}^{-1}) y_t^\rmP + \sqrt{2} w_t^\rmP ,
      }
    \end{equation}
    where $\{w_t^\rmP\}_{t \in \ccint{0,T}} = \{\rmP w_t\}_{t \in
      \ccint{0,T}}$. Note that since $\rmP$ is orthonormal,
    $\{w_t^\rmP\}_{t \in \ccint{0,T}}$ is also a $d$-dimensional Brownian
    motion. We also have that $\mathcal{L}(y_0^\rmP) = \mathrm{N}(0,
    \Id)$. Hence for any $\{\{y_t^{\rmP, i}\}_{t\in \ccint{0,T}}\}_{i=1}^d$ is a
    collection of $d$ independent Gaussian processes, where for any
    $i \in \{1, \dots, d\}$ and $t \in \ccint{0,T}$,
    $y_t^{\rmP, i} = \langle y_t^{\rmP}, e_i \rangle$ and $\{e_i\}_{i=1}^d$ is
    the canonical basis of $\rset^d$. Let $i \in \{1, \dots, d\}$ and for any
    $t \in \ccint{0,T}$ denote $u_t^i = \expeLigne{y_t^{\rmP, i}}$ and
    $v_t^i = \expeLigne{(y_t^{\rmP, i})^2}$. We have that for any
    $t \in \ccint{0,T}$, $\partial_t u_t^i = (1 - 1/\rmD_t^i) u_t^i$ with
    $u_0 =0$ and $\rmD_t^i = \exp[-2t] \rmD_i + 1 - \exp[-2t]$. Hence, we get
    that for any $t \in \ccint{0,T}$, $u_t^i = 0$. Using It\^o's lemma we have
    that
    \begin{equation}
      \label{eq:edo}
      \partial v_t^i = \{ 2 - 4 / \rmD_{T-t}^i \} v_t^i + 2 ,
    \end{equation}
    with $v_0^i = 1$. Denote $\alpha^i_T = (\rmD^i - 1)\exp[-2T]$, we have that
    for any $t \in \ccint{0,T}$, $\rmD_{T-t}^i = 1 + \alpha^i_T
    \exp[2t]$. Therefore, we get that for any $t \in \ccint{0,T}$
    \begin{equation}
      2 - 4 / \rmD_{T-t}^i = -2 + 2 \times (2\alpha_T^i)\exp[2t] /(1 + \alpha_T^i \exp[2t]) = - 2 + 2 \partial_t \log(1 + \alpha_T^i \exp[2t]) .
    \end{equation}
    Hence, we have that for any $t \in \ccint{0,T}$
    \begin{equation}
      \textstyle{
        \int_0^t 2 - 4 / \rmD_{T-s}^i \rmd s = -2t + \log((1 + \alpha_T^i \exp[2t])^2/(1 + \alpha_T^i )^2) .
        }
      \end{equation}
      Hence, there exists $C_t^i \in \rmc^1(\ccint{0,T}, \rset)$ such that for
      any $t \in \ccint{0,T}$, $v_t^i = C_t^i \exp[-2t](1 + \alpha_T^i \exp[2t])^2/(1 + \alpha_T^i )^2$.
      Using \eqref{eq:edo}, we have that for any $t \in \ccint{0,T}$
      \begin{equation}
        \partial_t C_t^i = 2 \exp[2t] ((1 + \alpha_T^i \exp[2t])/(1 + \alpha_T^i))^{-2} = -(1/\alpha_T^i)(1+\alpha_T^i)^2 \partial_t (1 + \alpha_T^i \exp[2t])^{-1}.
      \end{equation}
      Hence, we have that for any $t \in \ccint{0,T}$
      \begin{equation}
        C_t^i = (1/\alpha_T^i)(1+\alpha_T^i)^2 [ (1 + \alpha_T^i)^{-1} - (1 + \alpha_T^i \exp[2t])^{-1}  ] + A ,
      \end{equation}
      with $A \geq 0$. Hence, we get that for any $t \in \ccint{0,T}$
      \begin{align}
        v_t^i &= (1/\alpha_T^i) \exp[-2t] (1 + \alpha_T^i \exp[2t]) [ (1 + \alpha_T^i \exp[2t])/(1 + \alpha_T^i) - 1 ] \\
        & \qquad \qquad + A \exp[-2t](1 + \alpha_T^i \exp[2t])^2/(1 + \alpha_T^i )^2 .
      \end{align}
      In addition, we have that $v_0^i=1$ and therefore $A=1$. Therefore, for any $t \in \ccint{0,T}$ we have
      \begin{align}
        v_t^i &= (1/\alpha_T^i) \exp[-2t] (1 + \alpha_T^i \exp[2t]) [ (1 + \alpha_T^i \exp[2t])/(1 + \alpha_T^i) - 1 ] \\
              & \qquad \qquad + \exp[-2t](1 + \alpha_T^i \exp[2t])^2/(1 + \alpha_T^i )^2 \\
         &=  (1 - \exp[-2t]) (1 + \alpha_T^i \exp[2t]) /(1 + \alpha_T^i)  + \exp[-2t](1 + \alpha_T^i \exp[2t])^2/(1 + \alpha_T^i )^2 ,      
      \end{align}
      which concludes the proof.
    \end{proof}

    \subsection{Convergence results for the discretization}
\label{sec:conv-results-discr}

In what follows, we denote
$(Y_k)_{k \in \{0, \dots, N-1\}} = (\twx_{t_k})_{k \in \{0, \dots, N-1\}}$, the
sequence given by \eqref{eq:backward_disc}. The following result gives an
expansion of the covariance matrix and the mean of $Y_N$, i.e. the output of
SGM, in the case where $p = \mathrm{N}(\mu, \Sigma)$.

\begin{theorem}
  \label{prop:approx_gaussian}
  Let $N \in \nset$, $\delta > 0$ and $T = N \delta$. Then, we have that
  $\twx_{t_N} \sim \mathrm{N}(\hat{\mu}_N, \hat{\Sigma}_N)$ with
  \begin{equation}
    \hat{\Sigma}_N = \Sigma + \exp[-4T] \hat{\Sigma}_T + \delta \hat{\mathrm{E}}_T + \delta^2 \mathrm{R}_{T, \delta} \eqsp , \qquad \hat{\mu}_N = \mu + \exp[-2T] \hat{\mu}_T + \delta \hat{e}_T + \delta^2 r_{T, \delta} \eqsp ,
  \end{equation}
  where
  $\hat{\Sigma}_T, \hat{\mathrm{E}}_T, \mathrm{R}_{T, \delta} \in \rset^{d
    \times d}$, $\hat{\mu}_T, \hat{e}_T, r_{T, \delta} \in \rset^d$ and
  $\normLigne{\mathrm{R}_{T, \delta}} + \normLigne{r_{T, \delta}} \leq R$ not
  dependent on $T \geq 0$ and $\delta >0$. We have that
  \begin{align}
  &\hat{\Sigma}_T =  - (\Sigma - \Id)(\Sigma \Sigma_T^{-1})^{2}  \eqsp , \label{eq:T_approx}\\
    &\textstyle{ \hat{\mathrm{E}}_T = \Id - (1/2)\Sigma^2 (\Sigma - \Id)^{-1} \log(\Sigma)  + \exp[-2T] \tilde{\mathrm{E}}_T\eqsp .} \label{eq:gamma_approx}
  \end{align}
  In addition, we have 
  \begin{align}
    &\hat{\mu}_T = - \Sigma_T^{-1}\Sigma \mu \eqsp ,  \\
    &\textstyle{ \hat{e}_T = \{-2 \Sigma^{-1} -  (1/4) \Sigma (\Sigma - \Id)^{-1} \log(\Sigma)\} \mu + \exp[-2T] \tilde{\mu}_T \eqsp ,  }
  \end{align}
  with $\tilde{\mathrm{E}}_T, \tilde \mu_T$ bounded and not dependent on $T$.
\end{theorem}

Before turning to the proof of \Cref{prop:approx_gaussian}, we state a few
consequences of this result.

\begin{corollary}
  \label{sec:conv-results-discr-1}
  Let $\{\twx_{t_k}\}_{k=0}^N$ the sequence defined by
  \eqref{eq:backward_disc}. We have that
  $\twx_{t_N} \sim \mathrm{N}(\mu_N, \Sigma_N)$ with
  \begin{align}
    &\Sigma_N = \Sigma + \delta \Sigma_\delta + \exp[-4T] \Sigma_T + \Sigma_{\delta, T} \eqsp , \\
    &\mu_N = \mu + \delta \mu_\delta + \exp[-2T] \mu_T + \mu_{\delta, T} \eqsp ,
  \end{align}
  with
  \begin{align}
    & \Sigma_T = - (\Sigma - \Id)\Sigma^2 \eqsp , \\
    & \Sigma_\delta = \Id - (1/2)\Sigma^2 (\Sigma - \Id)^{-1} \log(\Sigma) \eqsp , \\    
    & \mu_T = \Sigma \mu  , \\
    & \mu_\delta = \{-2 \Sigma^{-1}-  (1/4) \Sigma (\Sigma - \Id)^{-1} \log(\Sigma)\} \mu  \eqsp . 
  \end{align}
  In addition, we have
  $\lim_{\delta \to 0, T \to +\infty} \normLigne{\Sigma_{\delta, T}}/(\delta +
  \exp[-4T]) = 0$ and
  $\lim_{\delta \to 0, T \to +\infty} \normLigne{\mu_{\delta, T}}/(\delta +
  \exp[-2T]) = 0$
\end{corollary}

At first sight, it might appear surprising that $\Sigma^{-1}$ does not appear in
${\Sigma}_T$ and ${\mu}_T$. Note that in the extreme case where
$\Sigma =0$ and $\delta \to0$, i.e. we only consider the error associated with
the fact that $T \neq +\infty$, then we have no error. This is because in this
case the associated continuous-time process is an Ornstein-Uhlenbeck bridge which
has distribution $\mathrm{N}(\mu, 0)$ at time $T$.

We will use the following result.

\begin{lemma}
\label{lemma:kl_gaussian}
  Let $\pi_i = \mathrm{N}(\mu_i, \Sigma_i)$ for $i \in \{0,1\}$, with
  $\mu_0, \mu_1 \in \rset^d$ and $\Sigma_0, \Sigma_1 \in
  \mathbb{S}_d(\rset)_+$. Then, we have that 
  \begin{equation}
    \KLLigne{\pi_0}{\pi_1} = (1/2) \{ \log(\det(\Sigma_1)/\det(\Sigma_0)) - d + \mathrm{Tr}(\Sigma_1^{-1}\Sigma_0) + (\mu_1 - \mu_0)^\top \Sigma_1^{-1}(\mu_1 - \mu_0)\} .
  \end{equation}
\end{lemma}

In particular, applying \Cref{lemma:kl_gaussian} we have that for any
$\Sigma \in \mathbb{S}_d(\rset)_+$
\begin{equation}
\label{eq:kl_computation}
  \KLLigne{\mathrm{N}(0, \Sigma)}{\mathrm{N}(0, \Id)} = (1/2)\{-\log(\det(\Sigma)) + \mathrm{Tr}(\Sigma) -d \}.
\end{equation}

\begin{proposition}
    \label{sec:conv-results-discr-2}
    Let $\{\twx_{t_k}\}_{k=0}^N$ the sequence defined by
    \eqref{eq:backward_disc}. We have that
    $\twx_{t_N} \sim \mathrm{N}(\mu_N, \Sigma_N)$, with $\mu_N, \Sigma_N$ given
    by \Cref{sec:conv-results-discr-1}.
    We have that
    \begin{align}
      \KLLigne{\mathrm{N}(\mu, \Sigma)}{\mathrm{N}(\mu_N, \Sigma_N)} \leq \delta \absLigne{\trace(\Sigma^{-1}\Sigma_\delta)} + \exp[-4T] \absLigne{\trace(\Sigma^{-1}\Sigma_T)} + \exp[-4T]\mu^\top\Sigma \mu + E_{T, \delta} \eqsp ,
    \end{align}
    with $E_{T, \delta}$ a higher order term such that $\lim_{T \to +\infty, \delta \to 0} E_{T, \delta} /(\delta + \exp[-4T]) = 0$.
\end{proposition}

We now prove \Cref{prop:approx_gaussian}.

\begin{proof}
  For any $k$, denote $Y_k = \twx_{t_{N-k}}$.
  First, we recall that for any $k \in \{0, \dots, N-1\}$ and $x \in \rset^d$,
  $\nabla \log p_{T-k\gamma}(x) = -\Sigma_{T-k\gamma}^{-1} x$ where for any $t \in \ccint{0,T}$
    \begin{equation}
      \Sigma_t = (1 - \exp[-2t]) \Id + \exp[-2t] \Sigma \eqsp . 
    \end{equation}
Hence, we get that for any $k \in \{0, \dots, N-1\}$
\begin{equation}
  \label{eq:recur_recu}
  Y_{k+1} = ((1 + \gamma) \Id - 2\gamma \Sigma_{T - k \gamma}^{-1}) Y_k  + 2 \gamma \Sigma_{T - k \gamma}^{-1} M_{T-k \gamma}+ \sqrt{2 \gamma} Z_{k+1} \eqsp ,
\end{equation}
where for any $t \in \ccint{0,T}$, $M_{t} = \exp[-t] \mu$. Therefore, we get
that for any $k \in \{0, \dots, N\}$, $Y_k$ is a Gaussian random variable.  
Using \eqref{eq:recur_recu}, we have that for any
$k \in \{0, \dots, N-1\}$
\begin{equation}
  \label{eq:recursion_cov}
  \expeLigne{\hat{Y}_{k+1} \hat{Y}_{k+1}^\top} = ((1 + \gamma) \Id - 2\gamma \Sigma_{T - k \gamma}^{-1}) \expeLigne{\hat{Y}_{k} \hat{Y}_{k}^\top} ((1 + \gamma) \Id - 2\gamma \Sigma_{T - k \gamma}^{-1}) + 2 \gamma \Id \eqsp ,
\end{equation}
where for any $k \in \{0, \dots, N\}$, $\hat{Y}_k = Y_k - \expeLigne{Y_k}$.
There exists $\rmP \in \rset^{d \times d}$ orthogonal such that
$\rmD = \rmP \Sigma \rmP^\top$ is diagonal. Note that for any
$k \in \{0, \dots, N-1\}$, we have that
$\Lambda_k = \rmP ((1+ \gamma) \Id - 2 \gamma \Sigma_{T-k\gamma}^{-1})
\rmP^\top$ is diagonal. For any $k \in \{0, \dots, N\}$, define
$\rmH_k = \rmP \expeLigne{\hat{Y}_{k} \hat{Y}_{k}^\top} \rmP^\top$. Note that
$\rmH_0 = \Id$.  Using \eqref{eq:recursion_cov}, we have that for any
$k \in \{0, \dots, N-1\}$
\begin{equation}
  \label{eq:recursion_cov_cov}
  \rmH_{k+1} = \Lambda_k^2 \rmH_k + 2 \gamma \Id \eqsp . 
\end{equation}
Hence, for any $k \in \{0, \dots, N\}$, $\rmH_k$ is diagonal. For any diagonal
matrix $C \in \rset^{d \times d}$ denote $\{c^1, \dots, c^d\}$ its diagonal
elements. Let $i \in \{1, \dots, d\}$. Using \eqref{eq:recursion_cov_cov}, we have
that for any $k \in \{0, \dots, N-1\}$
\begin{equation}
  h_{k+1}^i = (\lambda_k^i)^2 h_k^i + 2 \gamma \eqsp . 
\end{equation}
Using this result we have that for any $k \in \{0, \dots, N\}$
\begin{equation}
  \label{eq:unroll_recursion}
  \textstyle{
    h_k^i = (\prod_{\ell=0}^{k-1} \lambda_\ell^i)^2 + 2\gamma \sum_{\ell=0}^{k-1} (\prod_{j=0}^{\ell-1} \lambda_{k-1-j}^i)^2 = (\prod_{\ell=0}^{k-1} \lambda_\ell^i)^2 + 2\gamma \sum_{\ell=0}^{k-1} (\prod_{j=k-\ell}^{k-1} \lambda_{j}^i)^2 \eqsp . 
    }
\end{equation}
Let $k_1, k_2 \in \{0, \dots, N\}$ with $k_1 < k_2$. In what follows, we derive
an expansion of $I_{k_1, k_2} = \prod_{k=k_1}^{k_2} \lambda_k^i$ w.r.t.
$\gamma >0$. We have that
\begin{align}
  \label{eq:Ik1k2}
I_{k_1, k_2} &\textstyle{= \prod_{k=k_1}^{k_2} \lambda_k^i = \exp[\sum_{k=k_1}^{k_2} \log(\lambda_k^i)] = \exp[\sum_{k=k_1}^{k_2} \log(1 + \gamma a_k^i)]} \eqsp , 
\end{align}
where for any $k \in \{0, \dots, N\}$, $a_k^i = 1 - 2/d_{(N-k)\gamma}^i$, with
$d_{(N-k)\gamma}^i = 1 + \exp[-2(N-k)\gamma](d^i - 1)$.  Hence, there exist
$(b_{k, \gamma}^i)_{k \in \{0, \dots, N\}}$ bounded  such that for any $k \in \{0, \dots, N\}$ we have
\begin{equation}
  \log(1 + \gamma a_k^i) = \gamma a_k^i - (\gamma^2/2) (a_k^i)^2 + \gamma^3 b_{k, \gamma}^i \eqsp . 
\end{equation}
In addition, using \Cref{prop:maclaurin}, there exists
$C_{k_1, k_2}^\gamma \geq 0$ such that $ \gamma C_{k_1, k_2}^\gamma \leq C$ with
$C \geq 0$ not dependent on $k_2, k_2 \in \{0, \dots, N\}$, $\gamma >0$ and
\begin{equation}
  \textstyle{
    \sum_{k=k_1}^{k_2} \log(1 + \gamma a_k^i) = \int_{t_1}^{t_2^+} a^i(t) \rmd t - (\gamma/2) [\int_{t_1}^{t_2^+} a^i(t)^2 \rmd t + a^i(t_2^+) - a^i(t_1)] + C_{k_1, k_2}^\gamma \gamma^3  \eqsp ,
    }
  \end{equation}
  with $t_1 = k_1 \gamma$, $t_2^+ = (k_2 + 1) \gamma$ and for any
  $t \in \ccint{0,T}$, $a_t = 1 - 2/d_{T-t}^i$ with
  $d_{T-t}^i = 1 + \exp[-2(T-t)](d^i - 1)$.  Hence, using this result and
  \eqref{eq:Ik1k2}, we get that there exists $D_{k_1, k_2}^\gamma \geq 0$ such
  that $ \gamma D_{k_1, k_2}^\gamma \leq D$ with $D \geq 0$ not dependent on
  $k_2, k_2 \in \{0, \dots, N\}$, $\gamma >0$ and
\begin{equation}
  \textstyle{I_{k_1, k_2} = \exp[\int_{t_1}^{t_2^+} a^i(t) \rmd t ] - \exp[\int_{t_1}^{t_2^+} a^i(t)](\gamma/2)[\int_{t_1}^{t_2^+} a^i(t)^2 \rmd t + a^i(t_2^+) - a^i(t_1)] + \gamma^3 D_{k_1, k_2}^\gamma  \eqsp . } \label{eq:val_Ik1k2}
\end{equation}
Using this result, we get that there exists $E^\gamma_1 \geq 0$ such that
$\gamma E^\gamma_1 \leq E$ with $E \geq 0$ not dependent on $\gamma$ such that
\begin{equation}
  \textstyle{ (\prod_{\ell=0}^{N-1} \lambda_\ell^i)^2 =  \exp[2\int_0^T a^i(t) \rmd t] - \gamma \exp[2\int_0^T a^i(t) \rmd t][\int_{0}^{T} a^i(t)^2 \rmd t + a^i(T) - a^i(0)] + \gamma^3 E_1^\gamma \eqsp . } \label{eq:term_uno}
\end{equation}
Similarly, using \eqref{eq:val_Ik1k2}, there exist $E \geq 0$ and
$(E^\gamma_{2, \ell})_{\ell \in \{0, \dots, N\}}$ such that for any
$\ell \in \{0, \dots, N \}$, $E_{2, \ell}^\gamma \geq 0$ and
$\gamma E^\gamma_{2, \ell} \leq E$ with $E \geq 0$ not dependent on $\gamma$ and
$\ell$ such that
\begin{align}
  \textstyle{2 \gamma \sum_{\ell=0}^{N-1} (\prod_{j=N-\ell}^{N-1} \lambda_j^i)^2} &= \textstyle{ (2\gamma) \sum_{\ell=0}^{N-1}   \exp[2\int_{T-\ell \gamma}^T a^i(t) \rmd t]\}} \\
                                                                                  & \textstyle{ -2\gamma^2 \sum_{\ell=0}^{N-1} \{ \exp[2\int_{T- \ell \gamma}^T a^i(t) \rmd t][\int_{T- \ell \gamma}^{T} a^i(t)^2 \rmd t + a^i(T) - a^i(T-\ell \gamma)] \} } \\
  & + \textstyle{ \gamma^4 \sum_{\ell=0}^{N-1} E_{2, \ell}^\gamma \eqsp . }
\end{align}
Therefore, using \Cref{prop:maclaurin}, there exists $E_3^\gamma$ such that
$\gamma E_3^\gamma \leq E$ with $E \geq 0$ not dependent on $\gamma$ and
\begin{align}
  \textstyle{2 \gamma \sum_{\ell=0}^{N-1} (\prod_{j=N-\ell}^{N-1} \lambda_j^i)^2} &= \textstyle{2 \int_0^T \exp[2 \int_{T-t}^T a^i(s) \rmd s] \rmd t + \gamma (1 - \exp[2 \int_0^Ta^i(t) \rmd t]) } \\
                                                                                  &\textstyle{-2\gamma \int_0^T \{ \exp[2\int_{T-t}^T a^i(s) \rmd s][\int_{T- t}^{T} a^i(s)^2 \rmd t + a^i(T) - a^i(T-t)] \} \rmd t} \\
  &+ \gamma^3 E_3^\gamma \eqsp . \label{eq:term_duo}
\end{align}
Hence, combining \eqref{eq:term_uno} and \eqref{eq:term_duo} we get that 
\begin{equation}
  h_N^i = c_T^i - \gamma e_T^i + \gamma^3 E^\gamma \eqsp ,
\end{equation}
with
\begin{equation}
  \label{eq:ct_res}
  \textstyle{
    c_T^i  = \exp[2\int_0^T a^i(t) \rmd t] + 2\int_0^T \exp[2\int_{T-t}^T a^i(s) \rmd s] \rmd t \eqsp . 
    }
\end{equation}
and 
\begin{align}
  \label{eq:et_res}
  e_T^i  &= \textstyle{-\exp[2\int_0^T a^i(t) \rmd t][\int_{0}^{T} a^i(t)^2 \rmd t + a^i(T) - a^i(0)]} \textstyle{+ 1 - \exp[2 \int_0^T a^i(t) \rmd t ]} \\
    & \qquad \textstyle{ - 2 \int_0^T \exp[2\int_{T-t}^T a^i(s) \rmd s][\int_{T-t}^{T} a^i(s)^2 \rmd s + a^i(T) - a^i(T-t)] \rmd t \eqsp . 
    }
\end{align}
In what follows, we compute $c_T^i$ and $e_T^i$.
\begin{enumerate}[wide, labelwidth=!, labelindent=0pt, label=(\roman*)]
\item Using \Cref{lemma:integral_uno}  we have
  \begin{equation}
    \textstyle{ \exp[2\int_0^T a^i(t) \rmd t] = d^2 \exp[-2T]/(1 + \exp[-2T](d-1))^2 \eqsp . }
  \end{equation}
  In addition, using \Cref{lemma:integral_exp_uno} we have
  \begin{equation}
    \textstyle{ \int_0^T \exp[2\int_{T-t}^T a^i(s) \rmd s] = (d/2) (1 - \exp[-2T])/(1 + \exp[-2T](d-1)) \eqsp . }
  \end{equation}
  Combining these results and \eqref{eq:ct_res}, we get that
  \begin{align}
    c_T^i &= d + d^2 \exp[-2T](1 - (1 + \exp[-2T](d-1))^{-1})/(1 + \exp[-2T](d-1)) \\
    &= d + d^2 (d-1) \exp[-4T]/(1 + \exp[-2T](d-1))^2 \eqsp . 
  \end{align}
\item We conclude for $e_T^i$ using \Cref{prop:explicit_lambda} with
  $\lambda = d^i - 1$.
\end{enumerate}
This concludes the proof of \eqref{eq:gamma_approx}. Next, we compute the
evolution of the mean. Using \eqref{eq:recur_recu}, we have 
\begin{equation}
  \label{eq:recur_recu}
  \expeLigne{Y_{k+1}} = ((1 + \gamma) \Id - 2\gamma \Sigma_{T - k \gamma}^{-1}) \expeLigne{Y_k}  + 2 \gamma \expeLigne{\Sigma_{T - k \gamma}^{-1} M_{T-k \gamma}} \eqsp ,
\end{equation}
Note that for any $k \in \{0, \dots, N-1\}$, we have that
$\Lambda_k = \rmP ((1+ \gamma) \Id - 2 \gamma \Sigma_{T-k\gamma}^{-1})
\rmP^\top$ is diagonal. For any $k \in \{0, \dots, N\}$, define
$\rmH_k = \rmP \expeLigne{Y_{k}} \rmP^\top$. Note that
$\rmH_0 = 0$. For any $k \in \{0, \dots, N-1\}$ we have that
\begin{equation}
  \label{eq:recu_expectation}
  \rmH_{k+1} = \Lambda_k \rmH_k + 2 \gamma \rmD_{T- k\gamma}^{-1} \rmV_{T- k \gamma} \eqsp ,
\end{equation}
where for any $t \in \ccint{0,T}$, $\rmD_t = \rmP \Sigma_t \rmP^\top$ and
$\rmV_t = \rmP M_{t}$. Let $i \in \{1, \dots, d\}$. Using
\eqref{eq:recu_expectation}, we have for any $k \in \{0, \dots, N-1\}$
\begin{equation}
  \label{eq:recu_one_dim_exp}
  h_{k+1}^i = \lambda_k^i h_k^i + 2 \gamma v_{T-k \gamma}^i / d_{T-k\gamma}^i \eqsp . 
\end{equation}
In what follows, we define for any $t \in \ccint{0,T}$, $r(t)^i = v_{T-t}^i / d_{T-t}^i$
and note that for any $t \in \ccint{0,T}$
\begin{equation}
  \label{eq:r_express}
  r(t)^i = \exp[-(T-t)] / (1 + \exp[-2(T-t)](d^i - 1)) (\rmP \mu)^i \eqsp . 
\end{equation}
Using \eqref{eq:recu_one_dim_exp} and that $h_0^i = 0$, we have that for any $k \in \{0, \dots, N\}$
\begin{equation}
  \label{eq:unroll_recursion_exp}
  \textstyle{
    h_k^i =  2\gamma \sum_{\ell=0}^{k-1} r((k - \ell - 1)\gamma) \prod_{j=0}^{\ell-1} \lambda_{k-1-j}^i  =  2\gamma \sum_{\ell=0}^{k-1} r((k - \ell - 1)\gamma) \prod_{j=k-\ell}^{k-1} \lambda_{j}^i \eqsp . 
  }
\end{equation}
Using \eqref{eq:val_Ik1k2}, we get that there exists $D^\gamma \geq 0$ such that
$\gamma D^\gamma \leq D$ not dependent on $\gamma$ and
\begin{align}
  &h_N^i = \textstyle{2 \gamma \sum_{k=0}^{N-1}  r(T - (k + 1) \gamma) \exp[\int_{T - k \gamma}^{T} a^i(t) \rmd t ]} \\
        & \quad  \textstyle{- \gamma^2 \sum_{k=0}^{N-1} r(T - (k + 1) \gamma) \exp[\int_{T - k \gamma}^{T} a^i(t) \rmd t ][\int_{T - k \gamma}^{T} a^i(t)^2 \rmd t + a^i(T) - a^i(T - k \gamma)] } \\
        & \quad \textstyle{  + \sum_{k=0}^{N-1}  \gamma^4 D_{k, N}^\gamma} \\
&= \textstyle{2 \gamma \sum_{k=0}^{N-1}  r(T - (k + 1) \gamma) \exp[\int_{T - k \gamma}^{T} a^i(t) \rmd t ]} \\
        & \quad  \textstyle{- \gamma^2 \sum_{k=0}^{N-1} r(T - (k + 1) \gamma) \exp[\int_{T - k \gamma}^{T} a^i(t) \rmd t ][\int_{T - k \gamma}^{T} a^i(t)^2 \rmd t + a^i(T) - a^i(T - k \gamma)] } \\
  & \quad \textstyle{  + \gamma^3  D^\gamma \eqsp . }   
\end{align}
Using \Cref{prop:maclaurin}, we get that there exists $E^\gamma \geq 0$ such that
$\gamma E^\gamma \leq E$ not dependent on $\gamma$ and
\begin{align}
  &h_N^i = \textstyle{2 \gamma \sum_{k=0}^{N-1}  r(T - (k + 1) \gamma) \exp[\int_{T - k \gamma}^{T} a^i(t) \rmd t ]} \\
  & \quad  \textstyle{- \gamma \int_0^T r(T - t) \exp[\int_{T - t}^{T} a^i(s) \rmd s ][\int_{T - t}^{T} a^i(t)^2 \rmd t + a^i(T) - a^i(T - t)] \rmd t } \\
  & \quad \textstyle{  + \gamma^3  E^\gamma \eqsp . }     
\end{align}
In addition, for any $k \in \{0, \dots, N\}$, there exists $u_k \geq 0$ with
$u_k \leq u$ and $u \geq 0$ not dependent on $k$ and
\begin{equation}
  r(T-(k+1)\gamma) = r(T-k\gamma) -r'(T-k \gamma)\gamma + u_k \gamma^2 \eqsp . 
\end{equation}
Using this result, we get that exists $F^\gamma \geq 0$ such that
$\gamma F^\gamma \leq F$ not dependent on $\gamma$ and
\begin{align}
  &h_N^i = \textstyle{2 \gamma \sum_{k=0}^{N-1}  r(T - k \gamma) \exp[\int_{T - k \gamma}^{T} a^i(t) \rmd t ]} \\
  & \quad \textstyle{-2\gamma^2 \sum_{k=0}^{N-1} r'(T-k \gamma) \exp[\int_{T - k \gamma}^{T} a^i(t) \rmd t ] } \\
  & \quad  \textstyle{- \gamma \int_0^T r(T - t) \exp[\int_{T - t}^{T} a^i(s) \rmd s ][\int_{T - t}^{T} a^i(t)^2 \rmd t + a^i(T) - a^i(T - t)] \rmd t } \\
  & \quad \textstyle{  + \gamma^3  F^\gamma \eqsp . }     
\end{align}
Using \Cref{prop:maclaurin}, we get that there exists $G^\gamma \geq 0$ such that
$\gamma G^\gamma \leq G$ not dependent on $\gamma$ and
\begin{align}
  &h_N^i = \textstyle{2 \gamma \sum_{k=0}^{N-1}  r(T - k \gamma) \exp[\int_{T - k \gamma}^{T} a^i(t) \rmd t ]} \\
  & \quad \textstyle{-2\gamma \int_0^T r'(T-t) \exp[\int_{T - t}^{T} a^i(t) \rmd s ] \rmd t } \\
  & \quad  \textstyle{- \gamma \int_0^T r(T - t) \exp[\int_{T - t}^{T} a^i(s) \rmd s ][\int_{T - t}^{T} a^i(t)^2 \rmd t + a^i(T) - a^i(T - t)] \rmd t } \\
  & \quad \textstyle{  + \gamma^3  G^\gamma \eqsp . }     
\end{align}
In addition, using \Cref{prop:maclaurin}, we get that there exists
$H^\gamma \geq 0$ such that $\gamma H^\gamma \leq H$ not dependent on $\gamma$
and
\begin{align}
  &h_N^i = \textstyle{2 \int_0^T  r(T - t) \exp[\int_{T - t}^{T} a^i(s) \rmd s ] \rmd t} \\
  &\quad \textstyle{- \gamma \{r(0) \exp[\int_{0}^{T} a^i(t) \rmd t ]- r(T)\}}  \textstyle{-2\gamma \int_0^T r'(T-t) \exp[\int_{T - t}^{T} a^i(s) \rmd s ] \rmd t } \\
  & \quad  \textstyle{- \gamma \int_0^T r(T - t) \exp[\int_{T - t}^{T} a^i(s) \rmd s ][\int_{T - t}^{T} a^i(t)^2 \rmd t + a^i(T) - a^i(T - t)] \rmd t } \\
  & \quad \textstyle{  + \gamma^3  H^\gamma \eqsp . } \label{eq:inter_integ}    
\end{align}
In addition, we have by integration by part
\begin{align}
  &\textstyle{\int_0^T r'(T-t) \exp[\int_{T - t}^{T} a^i(s) \rmd s ] \rmd t} \\
  & \qquad \qquad \textstyle{= - \{r(0) \exp[\int_0^T a^i(t) \rmd t] - r(T)\} - \int_0^T r(T-t) a^i(T-t) \exp[\int_{T - t}^{T} a^i(s) \rmd s ] \rmd t \eqsp .  }
\end{align}
Combining this result and \eqref{eq:inter_integ} we get that
\begin{align}
  &h_N^i = \textstyle{2 \int_0^T  r(T - t) \exp[\int_{T - t}^{T} a^i(s) \rmd s ] \rmd t} \\
  &\quad \textstyle{+ \gamma \{r(0) \exp[\int_{0}^{T} a^i(t) \rmd t ]- r(T)\}}  \\
  & \quad  \textstyle{- \gamma \int_0^T r(T - t) \exp[\int_{T - t}^{T} a^i(s) \rmd s ][\int_{T - t}^{T} a^i(t)^2 \rmd t + a^i(T) - 3 a^i(T - t)] \rmd t } \\
  & \quad \textstyle{  + \gamma^3  H^\gamma \eqsp . }  \label{eq:almost_dooone}
\end{align}
In what follows, we assume that $d^i \neq 0$. The case where $d^i = 0$ is left
to the reader.  Finally using \eqref{eq:r_express} and \Cref{lemma:integral_uno}
we have that for any $t \in \ccint{0,T}$
\begin{align}
  \textstyle{
  \exp[\int_{T - t}^{T} a^i(s) \rmd s ] r(T-t)^i} &= \exp[-2t] / (1 + \exp[-2t](d^i - 1))^2 (\rmP \mu)^i d^i \\
  &= \textstyle{ \exp[2 \int_{T - t}^{T} a^i(s) \rmd s ] (\rmP \mu)^i / d^i \eqsp .
    }
\end{align}
Therefore, combining this result and \eqref{eq:almost_dooone}, we get that 
\begin{align}
  &h_N^i   = (\rmP \mu)^i/d^i [\textstyle{2 \int_0^T \exp[2 \int_{T-t}^{T} a^i(s) \rmd s] \rmd t} \\
  &\quad \textstyle{+ \gamma \{\exp[2 \int_{0}^{T} a^i(t) \rmd t] - 1\}}  \\
  & \quad  \textstyle{- \gamma \int_0^T \exp[2\int_{T-t}^T a^i(s) \rmd s] [\int_{T - t}^{T} a^i(t)^2 \rmd t + a^i(T) - 3 a^i(T - t)] \rmd t }] \\
  & \quad \textstyle{  + \gamma^3  H^\gamma \eqsp , }  \label{eq:almost_doooone}
\end{align}
which concludes the proof upon using \Cref{lemma:integral_exp_uno} and \Cref{prop:mean_mean}.
\end{proof}

\subsection{Technical lemmas}
\label{sec:technical-lemmas}

We are going to make use of the following lemma which is a direct consequence of
the Euler-MacLaurin formula.

\begin{proposition}
  \label{prop:maclaurin}
  Let $f \in \rmc^\infty(\ccint{0,T})$, and
  $(u_k^\gamma)_{k \in \{0, \dots, N-1\}}$ with $N \in \nset$ and
  $\gamma = T/N > 0$ such that for any $k \in \{0, \dots, N-1\}$,
  $u_k^\gamma = f(k \gamma)$. Then, there exists $C \geq 0$ such that
  \begin{equation}
    \textstyle{
      \int_0^T f(t) \rmd t - \gamma \sum_{k=0}^{N-1} u_k^\gamma -  (\gamma/2) \{f(T) - f(0)\} =   C \gamma^2 \eqsp .
      }
  \end{equation}
\end{proposition}

\begin{proof}
  Apply the classical Euler-MacLaurin formula to $t \mapsto f(t \gamma)$.
\end{proof}


We will also use the following lemmas.

\begin{lemma}
  \label{lemma:integral_uno}
  Let $\lambda \in \ooint{-1, +\infty}$ and $a: \ \ccint{0,T} \to \rset$ such
  that for any $t \in \ccint{0,T}$,
  \begin{equation}
    a(t) = 1 - 2/(1 + \exp[-2(T-t)]\lambda) \eqsp . 
  \end{equation}
  Then, we have that for any $t \in \ccint{0,T}$,
  \begin{equation}
    \textstyle{
      \int_{T-t}^T a(s) \rmd s = t + \log((1+\lambda)/(\exp[2t]+ \lambda)) \eqsp . 
      }
    \end{equation}
    In particular, we have that for any $t \in \ccint{0,T}$
    \begin{equation}
          \textstyle{
      \exp[2\int_{T-t}^T a(s) \rmd s] = \exp[-2t] (1+\lambda)^2 /(1+ \lambda\exp[-2t])^2 \eqsp . 
      }
    \end{equation}

  \end{lemma}

  \begin{proof}
    Let $t \in \ccint{0,T}$. We have that
    $\int_{T-t}^T a(s) \rmd s = \int_0^{t} a(T-s) \rmd s$.  Define $b$ such that
    for any $t \in \ccint{0,T}$, $b(t) = a (T-t)$. In particular, we have that
    for any $t \in \ccint{0,T}$
    \begin{equation}
      b(t) = 1 - 2 / (1 +\lambda \exp[-2t]) \eqsp . 
    \end{equation}
    Hence, we have
    \begin{align}
      \textstyle{\int_0^t b(s) \rmd s} &= \textstyle{t - 2 \int_0^t (1 + \lambda \exp[-2s])^{-1} \rmd s} \\
                                       &= \textstyle{t - \int_0^t 2 \exp[2s] / (\exp[2s] + \lambda ) \rmd s} \\
                                       &= \textstyle{t + \log((1 + \lambda)/(\exp[2t] + \lambda)) } \eqsp ,
    \end{align}
    which concludes the proof.
  \end{proof}

  \begin{lemma}
  \label{lemma:integral_exp_uno}
  Let $\lambda \in \ooint{-1, +\infty}$ and $a: \ \ccint{0,T} \to \rset$ such
  that for any $t \in \ccint{0,T}$,
  \begin{equation}
    a(t) = 1 - 2/(1 + \exp[-2(T-t)]\lambda) \eqsp . 
  \end{equation}
  Then,  we have that for any $t \in \ccint{0,T}$,
  \begin{align}
      \textstyle{\int_0^t \exp[ 2\int_{T-s}^T a(u) \rmd u] \rmd s} &=     \textstyle{(1/2)(1+\lambda)^2 [(1 + \lambda \exp[-2t])^{-1} - 1/(1+\lambda)]/\lambda} \\
      &= \textstyle{(1/2)(1+\lambda)  (1 -  \exp[-2t])/(1 + \lambda \exp[-2t])  } \eqsp . 
    \end{align}
  \end{lemma}

  \begin{proof}
    Let $t \in \ccint{0,T}$. Using \Cref{lemma:integral_uno} we have that for any $s \in \ccint{0,T}$
    \begin{equation}
      \textstyle{ \exp[2\int_{T-s}^T a(u) \rmd u] = (1+\lambda)^2 \exp[2s] / (\lambda + \exp[2s])^2 = (1+\lambda)^2 \exp[-2s]/(1 + \lambda \exp[-2s])^2 \eqsp .} 
    \end{equation}
    Assume that $\lambda \neq 0$. Then, we have that
    \begin{align}
      \textstyle{\int_0^t \exp[2\int_{T-s}^T a(u) \rmd u] \rmd s } &= \textstyle{(1/2)(1+\lambda)^2/\lambda\int_0^t 2 \lambda \exp[-2t]/(1 + \lambda \exp[-2t])^2 \rmd s} \\
                                                                  &= \textstyle{(1/2)(1+\lambda)^2 [(1 + \lambda \exp[-2t])^{-1} - 1/(1+\lambda)]/\lambda  } \\
      &= \textstyle{(1/2)(1+\lambda)  (1 -  \exp[-2t])/(1 + \lambda \exp[-2t])  } \eqsp .
    \end{align}
    We conclude the proof upon remarking that his result still holds in the case where $\lambda =0$.
  \end{proof}

  \begin{lemma}
  \label{lemma:integral_exp_duo}
  Let $\lambda \in \ooint{-1, +\infty}$ and $a: \ \ccint{0,T} \to \rset$ such
  that for any $t \in \ccint{0,T}$,
  \begin{equation}
    a(t) = 1 - 2/(1 + \exp[-2(T-t)]\lambda) \eqsp . 
  \end{equation}
  Then, if $\lambda \neq 0$, we have that for any $t \in \ccint{0,T}$,
  \begin{align}
    &\textstyle{
      \int_0^t \exp[ 2\int_{T-s}^T a(u) \rmd u]/(1 + \lambda \exp[-2s]) \rmd s} \\
    & \qquad \qquad \qquad \textstyle{= (1/4)(1+\lambda)^2 [(1 + \lambda \exp[-2t])^{-2} - 1/(1+\lambda)^2]/\lambda 
      } \\
    & \qquad \qquad \qquad \textstyle{=   (1/4)(1 - \exp[-2t])(2 + \lambda (1 + \exp[-2t]))/(1 + \lambda \exp[-2t])^2 } \eqsp . 
    \end{align}
    If $\lambda =0$ we have
    \begin{equation}
      \label{eq:2}
      \textstyle{
        \int_0^t \exp[ 2\int_{T-s}^T a(u) \rmd u]/(1 + \lambda \exp[-2s]) \rmd s = (1/2)(1 - \exp[-2t]) \eqsp . 
        }
    \end{equation}

  \end{lemma}

  \begin{proof}
    Let $t \in \ccint{0,T}$. Using \Cref{lemma:integral_uno} we have that for any $s \in \ccint{0,T}$
    \begin{equation}
      \textstyle{ \exp[\int_{T-s}^T a(u) \rmd u]/(1 + \lambda \exp[-2s]) = (1+\lambda)^2 \exp[-2s] / (1 + \lambda \exp[-2s])^3 \eqsp .} 
    \end{equation}
    Assume that $\lambda \neq 0$. Then, we have that
    \begin{align}
      &\textstyle{\int_0^t \exp[\int_{T-s}^T a(u) \rmd u]/(1 + \lambda \exp[-2s]) \rmd s } = \textstyle{(1/2)(1+\lambda)^2/\lambda\int_0^t 2 \lambda \exp[-2t]/(1 + \lambda \exp[-2t])^3 \rmd s} \\
                                                                  & \qquad \qquad = \textstyle{(1/4)(1+\lambda)^2 [(1 + \lambda \exp[-2t])^{-2} - 1/(1+\lambda)^2]/\lambda  } \\
      & \qquad \qquad = \textstyle{  (1/4)(1 - \exp[-2t])(2 + \lambda (1 + \exp[-2t]))/(1 + \lambda \exp[-2t])^2  } \eqsp .
    \end{align}
    We conclude the proof upon remarking that his result still holds in the case where $\lambda =0$.
  \end{proof}

  \begin{lemma}
    \label{lemma:hola_pas_carre}
  Let $\lambda \in \ooint{-1, +\infty}$ and $a: \ \ccint{0,T} \to \rset$ such
  that for any $t \in \ccint{0,T}$,
  \begin{equation}
    a(t) = 1 - 2/(1 + \exp[-2(T-t)]\lambda) \eqsp . 
  \end{equation}
  Then, we have that for any $t \in \ccint{0,T}$
  \begin{align}
    &\textstyle{
      \int_0^t \exp[ 2\int_{T-s}^T a(u) \rmd u] a(T-s) \rmd s} \\
    & \qquad \qquad \textstyle{= -(1/2)(1- \exp[-2t])(1 - \lambda^2 \exp[-2t])/(1 + \lambda \exp[-2t])^2 \eqsp . 
      }
  \end{align}
\end{lemma}

\begin{proof}
  Let $t \in \ccint{0,T}$. We have that
  \begin{align}
    &\textstyle{\int_0^t \exp[ 2\int_{T-s}^T a(u) \rmd u] a(T-s) \rmd s} \\
    & \qquad \textstyle{= \int_0^t \exp[ 2\int_{T-s}^T a(u) \rmd u] \rmd s - 2 \int_0^t  \exp[ 2\int_{T-s}^T a(u) \rmd u] / (1 + \lambda \exp[-2s]) \rmd s } \eqsp . \label{eq:combi_i}
  \end{align}
  Using \Cref{lemma:integral_exp_uno}, we have that
  \begin{equation}
    \textstyle{
      \int_0^t \exp[ 2\int_{T-s}^T a(u) \rmd u] \rmd t = (1/2)(1+\lambda)(1 - \exp[-2t])/(1 + \lambda \exp[-2t])\eqsp . 
    }
    \label{eq:uno_unO}
  \end{equation}
  In addition, using \Cref{lemma:integral_exp_duo}, we have 
  \begin{align}
    &\textstyle{
      \int_0^t  \exp[ 2\int_{T-s}^T a(u) \rmd u] / (1 + \lambda \exp[-2s]) \rmd s} \\
    & \qquad \qquad \textstyle{= (1/4)(1 - \exp[-2t])(2 + \lambda(1 + \exp[-2t]))/(1 + \lambda \exp[-2t])^2 \eqsp . 
      } \label{eq:uno_duO}
  \end{align}
  Combining \eqref{eq:uno_unO} and \eqref{eq:uno_duO} in \eqref{eq:combi_i} we have that
  \begin{align}
    &\textstyle{\int_0^t \exp[ 2\int_{T-s}^T a(u) \rmd u] a(T-s) \rmd s} \\
    & \qquad = (1/2)(1 - \exp[-2t])[(1+\lambda)(1 + \lambda \exp[-2t])]/(1 + \lambda \exp[-2t])^2 \\
    & \qquad \qquad - (1/2)(1-\exp[-2t]) (2 + \lambda(1 + \exp[-2t]))/(1 + \lambda \exp[-2t])^2 \\
    & \qquad = -(1/2)(1- \exp[-2t])(1 - \lambda^2 \exp[-2t])/(1 + \lambda \exp[-2t])^2 \eqsp ,
  \end{align}
  which concludes the proof.
\end{proof}

  \begin{lemma}
    \label{lemma:integral_changement}
    Let $\lambda \in \ooint{-1, +\infty}$ we have that for any $t \in \ccint{0,T}$
    \begin{equation}
      \textstyle{\int_0^t (1 + \lambda \exp[-2s])^{-1}} \rmd s = (1/2) \log((\lambda + \exp[2t])/(\lambda +1)) \eqsp . 
    \end{equation}
    In addition, we have for any $t \in \ccint{0,T}$
        \begin{equation}
      \textstyle{\int_0^t (1 + \lambda \exp[-2s])^{-2}} \rmd s = (1/2) \log((\lambda + \exp[2t])/(\lambda +1)) + (\lambda/2)[(\exp[2t] + \lambda)^{-1} - (\lambda + 1)^{-1}] \eqsp .
    \end{equation}
    Finally, we have that for any $t \in \ccint{0,T}$
        \begin{align}
      \textstyle{\int_0^t (1 + \lambda \exp[-2s])^{-3}} \rmd s &=  \textstyle{(1/2) \log((\lambda + \exp[2t])/(\lambda +1)) + \lambda [(\exp[2t] + \lambda)^{-1} - (\lambda + 1)^{-1}] } \\
      & \qquad \qquad - (\lambda^2/4)[(\exp[2t] + \lambda)^{-2} - (\lambda + 1)^{-2}] \eqsp . 
    \end{align}    
  \end{lemma}

  \begin{proof}
    Let $k \in \{1, 2, 3\}$. Using the change of variable $u \mapsto \exp[2u]$
    we have that
    \begin{equation}
      \textstyle{\int_0^{t}(1 + \lambda \exp[-2s])^{-k} \rmd s = (1/2) \int_1^{\exp[2t]} u^{k-1} / (u + \lambda) \rmd u \eqsp . }
    \end{equation}
    Therefore, we have that
    \begin{equation}
      \textstyle{\int_0^{t}(1 + \lambda \exp[-2s])^{-1} \rmd s = (1/2) \int_1^{\exp[2t]}  (u + \lambda)^{-1} \rmd u = (1/2)\log((\lambda + \exp[2t])/(\lambda +1)) \eqsp . }
    \end{equation}
    In addition, using that for any $u \in \ccint{0,T}$, $u = (u + \lambda) - \lambda$ we have that
    \begin{align}
      &\textstyle{\int_0^{t}(1 + \lambda \exp[-2s])^{-2} \rmd s} = \textstyle{(1/2) \int_1^{\exp[2t]}  u(u + \lambda)^{-2} \rmd u } \\
      & \qquad = \textstyle{(1/2) \int_1^{\exp[2t]}  (u + \lambda)^{-1} \rmd u - (\lambda/2) \int_1^{\exp[2t]} (u + \lambda)^{-2} \rmd u } \\
      & \qquad = \textstyle{(1/2)\log((\lambda + \exp[2t])/(\lambda +1)) + (\lambda/2)[(\exp[2t] + \lambda)^{-1} - (\lambda + 1)^{-1}] } \eqsp . 
    \end{align}
    Finally, using that for any $u \in \ccint{0,T}$, $u \in \ccint{0,T}$, $u^2 = (u + \lambda)^2 - 2\lambda(u+\lambda) + \lambda^2$ we have that
    \begin{align}
      &\textstyle{\int_0^{t}(1 + \lambda \exp[-2s])^{-3} \rmd s} = \textstyle{(1/2) \int_1^{\exp[2t]}  u^2(u + \lambda)^{-2} \rmd u } \\
      & \qquad = \textstyle{(1/2) \int_1^{\exp[2t]}  (u + \lambda)^{-1} \rmd u - \lambda \int_1^{\exp[2t]} (u + \lambda)^{-2} \rmd u + (\lambda^2/2) \int_1^{\exp[2t]} (u + \lambda)^{-3} \rmd u} \\
      & \qquad = \textstyle{(1/2) \log((\lambda + \exp[2t])/(\lambda +1)) + \lambda [(\exp[2t] + \lambda)^{-1} - (\lambda + 1)^{-1}] } \\
      & \qquad \qquad - (\lambda^2/4)[(\exp[2t] + \lambda)^{-2} - (\lambda + 1)^{-2}] \eqsp ,
    \end{align}
    which concludes the proof.
  \end{proof}
  
\begin{lemma}
  \label{lemma:integral_duo}
  Let $\lambda \in \ooint{-1, +\infty}$ and $a: \ \ccint{0,T} \to \rset$ such
  that for any $t \in \ccint{0,T}$,
  \begin{equation}
    a(t) = 1 - 2/(1 + \exp[-2(T-t)]\lambda) \eqsp . 
  \end{equation}
  Then, we have that for any $t \in \ccint{0,T}$,  
  \begin{equation}
    \textstyle{\int_{T-t}^T a(s)^2 \rmd s} = t - 2\lambda(1 - \exp[-2t]) / [(1 + \lambda)(1 + \lambda \exp[-2t])] \eqsp . 
  \end{equation}
  \end{lemma}
  
  \begin{proof}
    Let $t \in \ccint{0,T}$. Similarly to the proof of
    \Cref{lemma:integral_uno}, we have that
    $\int_{T-t}^T a(s) \rmd s = \int_0^{t} a(T-s) \rmd s$.  Define $b$ such that
    for any $t \in \ccint{0,T}$, $b(t) = a (T-t)$. In particular, we have that
    for any $t \in \ccint{0,T}$
    \begin{equation}
      b(t) = 1 - 2 / (1 +\lambda \exp[-2t]) \eqsp . 
    \end{equation}
    We have that
    \begin{equation}
      \textstyle{\int_{T-t}^Ta(s)^2 \rmd s = \int_0^t b(s)^2 \rmd s  = \int_0^t (1 - 4/(1 + \lambda \exp[-2s]) + 4/(1 + \lambda \exp[-2s])^2) \rmd s  \eqsp . }
    \end{equation}
    Combining this result and \Cref{lemma:integral_changement}, we have
  \begin{align}
    \textstyle{\int_{T-t}^T a(s)^2 \rmd s} &= t + 2\lambda [(\lambda + \exp[2t])^{-1} - (\lambda + 1)^{-1} ] \\
    &= t - 2\lambda(1 - \exp[-2t]) / [(1 + \lambda)(1 + \lambda \exp[-2t])] \eqsp . 
  \end{align}    
    \end{proof}

    \begin{lemma}
      \label{lemma:hola_carre}
  Let $\lambda \in \ooint{-1, +\infty}$ and $a: \ \ccint{0,T} \to \rset$ such
  that for any $t \in \ccint{0,T}$,
  \begin{equation}
    a(t) = 1 - 2/(1 + \exp[-2(T-t)]\lambda) \eqsp . 
  \end{equation}
  Then, if $\lambda \neq 0$, we have that 
  \begin{align}
    &\textstyle{\int_0^T \exp[2\int_{T-t}^T a(s) \rmd s] (\int_{T-t}^T a(s)^2 \rmd s) \rmd t} = - (T/2)(1+\lambda)^2\exp[-2T]/(1+\lambda\exp[-2T]) \\
                                                                                             & \qquad \qquad + (1+\lambda)^2/(4\lambda)\log((1+\lambda)/(1 + \lambda \exp[-2T])) \\
    & \qquad \qquad -(\lambda/2)(1-\exp[-2T])^2/(1 + \lambda \exp[-2T])^2 \eqsp . \label{eq:big_eq_err}
  \end{align}
  If $\lambda = 0$, we have that
  \begin{equation}
    \textstyle{\int_0^T \exp[2\int_{T-t}^T a(s) \rmd s] (\int_{T-t}^T a(s)^2 \rmd s) \rmd t} = -(T/2) \exp[-2T] + (1/4)(1 - \exp[-2T]) \eqsp . \label{eq:small_eq}
  \end{equation}
\end{lemma}
Note that taking $\lambda \to 0$ in \eqref{eq:big_eq_err} we recover
\eqref{eq:small_eq}, using that for any $u > 0$,
$\lim_{\lambda \to 0} \log(1 + \lambda u)/\lambda = u$.

\begin{proof}
  We first start with the case $\lambda \neq 0$.
  Similarly to the proof of \Cref{lemma:integral_uno}, we have that
  $\int_{T-t}^T a(s) \rmd s = \int_0^{t} a(T-s) \rmd s$.  Define $b$ such that
  for any $t \in \ccint{0,T}$, $b(t) = a (T-t)$. We have that
  \begin{equation}
    \textstyle{
      \int_0^T \exp[2\int_{T-t}^T a(s) \rmd s] (\int_{T-t}^T a(s)^2 \rmd s) \rmd t = \int_0^T \exp[2\int_{T-t}^T a(s) \rmd s] (\int_{0}^t b(s)^2 \rmd s) \rmd t \eqsp .
      }
    \end{equation}
    Let $A: \ \ccint{0,T} \to \rset$ such that for any $t \in \ccint{0,T}$,
    \begin{equation}
      \textstyle{
        A(t) = \int_0^t \exp[2\int_{T-s}^T a(u) \rmd u] \rmd s \eqsp .
        }
    \end{equation}
    Note that $A(0) = 0$. Hence, by integration by parts, we have
    \begin{equation}
      \textstyle{ \int_0^T \exp[2\int_{T-t}^T a(s) \rmd s] (\int_{0}^t b(s)^2 \rmd s) \rmd t = A(T)\int_0^T b(t)^2 \rmd t - \int_0^T A(t) b(t)^2 \rmd t \eqsp . }
    \end{equation}
    In what follows, we compute $\int_0^T A(t) b(t)^2 \rmd t$. First, we recall that for any $t \in \ccint{0,T}$
    \begin{equation}\label{eq:b_square}
      b(t)^2 = (1 - 2/(1 + \lambda \exp[-2t]))^2 = 1 - 4/(1 + \lambda \exp[-2t]) + 4/(1 + \lambda \exp[-2t])^2 \eqsp . 
    \end{equation}
    In addition, using \Cref{lemma:integral_exp_uno}, we have that for any $t \in \ccint{0,T}$
    \begin{equation}
      \label{eq:A_exp}
      A(t) = (1/2)\{ (1+\lambda)^2/(\lambda(1 + \lambda \exp[-2t])) - (1+\lambda)/\lambda\} \eqsp . 
    \end{equation}
    Using \eqref{eq:b_square} and \eqref{eq:A_exp} we have that for any $t \in \ccint{0,T}$
    \begin{align}
      2A(t)b(t)^2 &= -(1+\lambda)/\lambda + [4(1+\lambda)/\lambda +(1+\lambda)^2/\lambda] u_1(t) \\
                 & \qquad - [4(1+\lambda)^2/\lambda + 4(1+\lambda)/\lambda] u_2(t) + [4(1+\lambda)^2/\lambda] u_3(t) \\
      &= -(1+\lambda)/\lambda + [(1+\lambda)(5 + \lambda)/\lambda] u_1(t) \\
      & \qquad - [4(1+\lambda)(2+\lambda)/\lambda] u_2(t) + [4(1+\lambda)^2/\lambda] u_3(t)\eqsp , \label{eq:dev_long}
    \end{align}
    where for any $k \in \{1,2, 3\}$ and $t \in \ccint{0,T}$ we have
    \begin{equation}
      u_k(t) = (1 + \lambda \exp[-2t])^{-k} \eqsp . 
    \end{equation}
    For any $k \in \{0, 1, 2\}$ denote $v_k : \ \ccint{0,T} \to \rset$ such that
    for any $t \in \ccint{0,T}$ and $k \in \{1, 2\}$
    \begin{equation}
      v_0(t) = \log((\lambda + \exp[2t])/(\lambda+1)) \eqsp , \qquad v_k(t) = (\exp[2t] + \lambda)^{-k} - (1 + \lambda)^{-k}\eqsp .
    \end{equation}
    Combining \eqref{eq:dev_long} and \Cref{lemma:integral_changement}, we get
    that for any $t \in \ccint{0,T}$
    \begin{align}
      &\textstyle{2 \int_0^t A(s) b(s)^2 \rmd s} = -[(1+\lambda)/\lambda]t \\
                                              & \qquad + (1/2) \{ [(1+\lambda)(5 + \lambda)/\lambda] - [4(1+\lambda)(2+\lambda)/\lambda] + [4(1+\lambda)^2/\lambda] \} v_0(t) \\
                                              & \qquad + \{ - (\lambda/2) [4(1+\lambda)(2+\lambda)/\lambda] + \lambda [4(1+\lambda)^2/\lambda]\} v_1(t) \\
                                              & \qquad  - (\lambda^2/4)  [4(1+\lambda)^2/\lambda] v_2(t) \\
      &= -[(1+\lambda)/\lambda]t + (1+\lambda)^2/(2\lambda) v_0(t) + 2(1+\lambda)\lambda v_1(t) - \lambda(1+\lambda)^2 v_2(t) \eqsp . 
    \end{align}
    In addition, we have that
    \begin{align}
      &-[(1+\lambda)/\lambda]t + (1+\lambda)^2/(2\lambda) v_0(t) = -[(1+\lambda)/\lambda]t + [(1+\lambda)^2/\lambda]t \\ 
                                                                & \qquad \qquad \qquad + (1+\lambda)^2/(2\lambda)\log((1 + \lambda \exp[-2t])/(1 + \lambda)) \\
      &\qquad \qquad =(1+\lambda)t +(1+\lambda)^2/(2\lambda)\log((1 + \lambda \exp[-2t])/(1 + \lambda)) \eqsp . 
    \end{align}
    Therefore, we get that
    \begin{align}
      \textstyle{2 \int_0^t A(s) b(s)^2 \rmd s} & = (1+\lambda)t +(1+\lambda)^2/(2\lambda)\log((1 + \lambda \exp[-2t])/(1 + \lambda))  \\
      & \qquad + 2(1+\lambda)\lambda v_1(t) + \lambda(1+\lambda)^2 v_2(t) \eqsp . \label{eq:et_un}
    \end{align}
    In addition, we have that
    \begin{equation}
      (1+\lambda)\lambda v_1(t) = -\lambda(1 - \exp[-2T])/(1 + \lambda \exp[-2T]) \eqsp . \label{eq:et_deux}
    \end{equation}
    We also have that
    \begin{align}
      \lambda(1+\lambda)^2 v_2(t) &= \lambda (2 \lambda + 1 - 2 \lambda \exp[2T] - \exp[4T])/(\exp[2T] + \lambda)^2 \\
                                  &= \lambda (1 - \exp[2T])(1 + 2\lambda  + \exp[2T]) /(\exp[2T] + \lambda)^2 \\
                                    &= -\lambda (1 - \exp[-2T])(1 + (1 + 2\lambda) \exp[-2T]) /(1 + \lambda \exp[-2T])^2 \eqsp . \label{eq:et_trois}
    \end{align}
    Finally, using \Cref{lemma:integral_exp_uno} and \Cref{lemma:integral_duo} we have
    \begin{align}
      \textstyle{A(T) \int_0^T b(t)^2 \rmd t} &= \textstyle{(1/2)(1+\lambda)  (1 -  \exp[-2T])/(1 + \lambda \exp[-2T])}\\
                                              & \qquad \times (T - 2 \lambda(1 - \exp[-2T]) / [(1 + \lambda)(1 + \lambda \exp[-2T])])   \\
                                              & = (T/2)(1+\lambda)  (1 -  \exp[-2T])/(1 + \lambda \exp[-2T]) \\
      & \qquad - \lambda (1 -  \exp[-2T])^2/(1 + \lambda \exp[-2T])^2 \eqsp .  \label{eq:et_quatre}
    \end{align}
    Combining \eqref{eq:et_un}, \eqref{eq:et_deux},
    \eqref{eq:et_trois} and \eqref{eq:et_quatre} we get
  \begin{align}
    &\textstyle{\int_0^T \exp[2\int_{T-t}^T a(s) \rmd s] (\int_{T-t}^T a(s)^2 \rmd s) \rmd t} = (T/2)(1+\lambda)(1-\exp[-2T])/(1+\lambda \exp[-2T]) \\
                                                                                             & \qquad \qquad - \lambda(1- \exp[-2T])^2/(1 + \lambda \exp[-2T])^2 \\
                                                                                             & \qquad \qquad - (1 + \lambda)(T/2) + (1+\lambda)^2/(4\lambda)\log((1+\lambda)/(1 + \lambda \exp[-2T])) \\
    & \qquad \qquad +\lambda(1 - \exp[-2T])/(1 + \lambda \exp[-2T]) \\
    & \qquad \qquad -(\lambda/2)(1-\exp[-2T])((1+2\lambda)\exp[-2T] +1)/(1 + \lambda \exp[-2T])^2 \eqsp . \label{eq:big_eq_err_inter}
  \end{align}
  In addition, we have that
  \begin{align}
    &-(\lambda/2)(1-\exp[-2T])^2/(1 + \lambda \exp[-2T])^2 & \\
    &\qquad \qquad \qquad = - \lambda(1- \exp[-2T])^2/(1 + \lambda \exp[-2T])^2 \\
                                                                            & \qquad \qquad \qquad \qquad +\lambda(1 - \exp[-2T])/(1 + \lambda \exp[-2T]) \\
    & \qquad \qquad \qquad \qquad -(\lambda/2)(1-\exp[-2T])((1+2\lambda)\exp[-2T] +1)/(1 + \lambda \exp[-2T])^2 \eqsp .
  \end{align}
  Combining this result and \eqref{eq:big_eq_err_inter}, we get
  \begin{align}
    &\textstyle{\int_0^T \exp[2\int_{T-t}^T a(s) \rmd s] (\int_{T-t}^T a(s)^2 \rmd s) \rmd t} = (T/2)(1+\lambda)(1-\exp[-2T])/(1+\lambda \exp[-2T]) \\
                                                                                             & \qquad \qquad - (1 + \lambda)(T/2) + (1+\lambda)^2/(4\lambda)\log((1+\lambda)/(1 + \lambda \exp[-2T])) \\
    & \qquad \qquad -(1/2)(1-\exp[-2T])^2/(1 + \lambda \exp[-2T])^2 \eqsp . \label{eq:big_eq_err_inter2}
  \end{align}
  Finally, we have
  \begin{align}
    &(T/2)(1+\lambda)(1-\exp[-2T])/(1+\lambda \exp[-2T]) - (T/2)(1+\lambda) \\
    &\qquad = - (T/2)(1+\lambda)^2\exp[-2T]/(1+\lambda\exp[-2T]) \eqsp ,
  \end{align}
  which concludes the proof in the case $\lambda \neq 0$ upon combining this
  result and \eqref{eq:big_eq_err_inter2}. In the case $\lambda = 0$, we have that
  for any $t \in \ccint{0,T}$, $a(t) = -1$ and therefore by integration by part
  we have
    \begin{equation}
      \textstyle{\int_0^T \exp[2\int_{T-t}^T a(s) \rmd s] (\int_{T-t}^T a(s)^2 \rmd s) \rmd t} = -(T/2) \exp[-2T] + (1/4)(1 - \exp[-2T]) \eqsp ,
    \end{equation}
    which concludes the proof.
  \end{proof}

    We are now ready to prove the following results.

    \begin{proposition}
      \label{prop:expression_fat_1}
      Let $\lambda \in \ooint{-1, +\infty}$ and $a: \ \ccint{0,T} \to \rset$ such
  that for any $t \in \ccint{0,T}$,
  \begin{equation}
    a(t) = 1 - 2/(1 + \exp[-2(T-t)]\lambda) \eqsp . 
  \end{equation}
  Then,  we have that for any $t \in \ccint{0,T}$,
  \begin{align}
    &\textstyle{\exp[2\int_0^T a(t) \rmd t]\{ \int_0^Ta(t)^2 \rmd t + a(T) - a(0) + 1\}} \\
    & \qquad \qquad \qquad = (T+1)\exp[-2T](\lambda+1)^2/(1 +\lambda \exp[-2T])^2 \eqsp . 
  \end{align}
  \end{proposition}

  \begin{proof}
    First, we have that
    \begin{align}
      a(T) - a(0) &= 1 - 2 /(1+\lambda) - 1 + 2/(1+\lambda \exp[-2T]) \\
                  &= 2\lambda(1 - \exp[-2T])/[(1+\lambda)(1+\lambda \exp[-2T])] \eqsp .
                    \label{eq:al_uno}
    \end{align}
    In addition, using \Cref{lemma:integral_duo} we have
  \begin{equation}
    \textstyle{\int_{0}^T a(s)^2 \rmd s} = T - 2 \lambda(1 - \exp[-2T]) / [(1 + \lambda)(1 + \lambda \exp[-2T])] \eqsp .
    \label{eq:al_duo}
  \end{equation}
  Finally, using \Cref{lemma:integral_uno} we have that
  \begin{equation}
          \textstyle{
      \exp[2\int_{0}^T a(s) \rmd s] = \exp[-2T] (\lambda+1)^2 /(1+ \lambda\exp[-2T])^2 \eqsp . 
      }\label{eq:al_tertio}
    \end{equation}
    We conclude the proof upon combining \eqref{eq:al_uno}, \eqref{eq:al_duo}
    and \eqref{eq:al_tertio}.
  \end{proof}

  Finally, we have the following proposition.

  \begin{proposition}
    \label{prop:expression_fat_2}
      Let $\lambda \in \ooint{-1, +\infty}$ and $a: \ \ccint{0,T} \to \rset$ such
  that for any $t \in \ccint{0,T}$,
  \begin{equation}
    a(t) = 1 - 2/(1 + \exp[-2(T-t)]\lambda) \eqsp . 
  \end{equation}
  Then, if $\lambda \neq 0$, we have that for any $t \in \ccint{0,T}$,
  \begin{align}
  &\textstyle{\int_0^T \exp[2\int_{T-t}^T a(s) \rmd s] [\int_{T-t}^T a(s)^2 \rmd s + a(T) - a(T-t)] \rmd t } \\
  & \qquad = - (T/2)(1+\lambda)^2\exp[-2T]/(1+\lambda\exp[-2T]) \\
   & \qquad \qquad + (1+\lambda)^2/(4\lambda)\log((1+\lambda)/(1 + \lambda \exp[-2T])) \\
    & \qquad \qquad + (\lambda/2) \exp[-2T]/(1 + \lambda \exp[-2T])^2 \eqsp .
  \end{align}
  If $\lambda = 0$, we have that
  \begin{align}
    &\textstyle{\int_0^T \exp[2\int_{T-t}^T a(s) \rmd s] [\int_{T-t}^T a(s)^2 \rmd s + a(T) - a(T-t)] \rmd t } \\
    & \qquad \qquad  \qquad \qquad = -(T/2)\exp[-2T] + (1/4)(1 - \exp[-2T]) \eqsp . 
  \end{align}

  \end{proposition}

  \begin{proof}
    We assume that $\lambda \neq 0$. The case where $\lambda = 0$ is left to the
    reader.  First, using \Cref{lemma:hola_carre}, we have that
    \begin{align}
          &\textstyle{\int_0^T \exp[2\int_{T-t}^T a(s) \rmd s] (\int_{T-t}^T a(s)^2 \rmd s) \rmd t} = - (T/2)(1+\lambda)^2\exp[-2T]/(1+\lambda\exp[-2T]) \\
                                                                                             & \qquad \qquad + (1+\lambda)^2/(4\lambda)\log((1+\lambda)/(1 + \lambda \exp[-2T])) \\
          & \qquad \qquad +(3\lambda/2)\exp[-2T](1-\exp[-2T])(1+\lambda)/(1 + \lambda \exp[-2T])^2 \eqsp .
              \label{eq:almost_uno}
    \end{align}
  Second, using \Cref{lemma:hola_pas_carre}, we have that
  \begin{align}
        &\textstyle{
      \int_0^T \exp[ 2\int_{T-t}^T a(u) \rmd u] a(T-t) \rmd t} \\
    & \qquad \qquad \textstyle{= -(1/2)(1- \exp[-2T])(1 - \lambda^2 \exp[-2T])/(1 + \lambda \exp[-2T])^2 \eqsp . \label{eq:almost_duo}
      }
  \end{align}
  Third, using \Cref{lemma:integral_exp_uno} and that $a(T) = 1 -2/(1+\lambda)$, we have that
  \begin{align}
    \textstyle{
    a(T) \int_0^T \exp[2\int_{T-t}^T a(s) \rmd s] \rmd t} &= a(T) \textstyle{(1/2)(1+\lambda)  (1 -  \exp[-2T])/(1 + \lambda \exp[-2T])  }\\
                                                          &=(1/2)(1+\lambda)  (1 -  \exp[-2T])/(1 + \lambda \exp[-2T]) \\
    & \qquad  -  (1 -  \exp[-2T])/(1 + \lambda \exp[-2T]) \eqsp .\label{eq:almost_tertio}
  \end{align}
  Combining \eqref{eq:almost_uno}, \eqref{eq:almost_duo} and \eqref{eq:almost_tertio} we get
\begin{align}
  &\textstyle{\int_0^T \exp[2\int_{T-t}^T a(s) \rmd s] [\int_{T-t}^T a(s)^2 \rmd s + a(T) - a(T-t)] \rmd t } \\
  & \qquad = - (T/2)(1+\lambda)^2\exp[-2T]/(1+\lambda\exp[-2T]) \\
   & \qquad \qquad + (1+\lambda)^2/(4\lambda)\log((1+\lambda)/(1 + \lambda \exp[-2T])) \\
   & \qquad \qquad +(3\lambda/2)\exp[-2T](1-\exp[-2T])(1+\lambda)/(1 + \lambda \exp[-2T])^2 \\
   & \qquad \qquad  +(1/2)(1- \exp[-2T])(1 - \lambda^2 \exp[-2T])/(1 + \lambda \exp[-2T])^2 \\ 
      & \qquad \qquad + (1/2)(1+\lambda)  (1 -  \exp[-2T])/(1 + \lambda \exp[-2T]) \\
     & \qquad \qquad  -  (1 -  \exp[-2T])/(1 + \lambda \exp[-2T])
\end{align}
In addition, we have
\begin{align}
  & (\lambda/2) (1 - \exp[-2T])/(1 + \lambda \exp[-2T])^2 \\
  &\qquad = (1/2)(1- \exp[-2T])(1 - \lambda^2 \exp[-2T])/(1 + \lambda \exp[-2T])^2 \\ 
      & \qquad \qquad + (1/2)(1+\lambda)  (1 -  \exp[-2T])/(1 + \lambda \exp[-2T]) \\
     & \qquad \qquad  -  (1 -  \exp[-2T])/(1 + \lambda \exp[-2T]) \eqsp ,
\end{align}
Finally, we have
\begin{align}
  &(\lambda/2)\exp[-2T]/(1 + \lambda \exp[-2T])^2 \\
    & \qquad \qquad = -(\lambda/2)(1-\exp[-2T])^2/(1 + \lambda \exp[-2T])^2 \\
    & \qquad \qquad \qquad + (\lambda/2) (1 - \exp[-2T])/(1 + \lambda \exp[-2T])^2 \eqsp .
\end{align}
which concludes the proof.
  \end{proof}

  Finally, we have the following result.

  \begin{proposition}
    \label{prop:explicit_lambda}
        Let $\lambda \in \ooint{-1, +\infty}$ and $a: \ \ccint{0,T} \to \rset$ such
  that for any $t \in \ccint{0,T}$,
  \begin{equation}
    a(t) = 1 - 2/(1 + \exp[-2(T-t)]\lambda) \eqsp . 
  \end{equation}
  Then, if $\lambda \neq 0$, we have that for any $t \in \ccint{0,T}$,
  \begin{align}
    &\textstyle{-\exp[2\int_0^T a(t) \rmd t]\{ \int_0^Ta(t)^2 \rmd t + a(T) - a(0)\} + 1 - \exp[2\int_0^T a(t) \rmd t]} \\
    & \qquad \qquad - 2 \textstyle{\int_0^T \exp[2\int_{T-t}^T a(s) \rmd s] [\int_{T-t}^T a(s)^2 \rmd s + a(T) - a(T-t)] \rmd t } \\
    & \qquad \qquad \qquad = 1 - \exp[-2T] (1 - \lambda T \exp[-2T]) (\lambda+1)^2/(1 +\lambda \exp[-2T])^2 \\
   & \qquad \qquad \qquad \qquad - (1+\lambda)^2/(2\lambda)\log((1+\lambda)/(1 + \lambda \exp[-2T])) \\
  & \qquad \qquad \qquad \qquad -\lambda\exp[-2T]/(1 + \lambda \exp[-2T])^2 \eqsp . 
      \end{align}
      In particular, we have that 
  \begin{align}
        & \textstyle{-\exp[2\int_0^T a(t) \rmd t]\{ \int_0^Ta(t)^2 \rmd t + a(T) - a(0)\} + 1 - \exp[2\int_0^T a(t) \rmd t]} \\
        & \qquad \qquad - 2 \textstyle{\int_0^T \exp[2\int_{T-t}^T a(s) \rmd s] [\int_{T-t}^T a(s)^2 \rmd s + a(T) - a(T-t)] \rmd t } \\
    & \qquad \qquad \qquad = 1 - (1/2) (1+\lambda)^2\log(1+\lambda)/\lambda + O(\exp[-2T]) \eqsp . 
  \end{align}  
  If $\lambda = 0$, we have that for any $t \in \ccint{0,T}$
  \begin{align}
        &\textstyle{-\exp[2\int_0^T a(t) \rmd t]\{ \int_0^Ta(t)^2 \rmd t + a(T) - a(0)\} + 1 - \exp[2\int_0^T a(t) \rmd t]} \\
        & \qquad \qquad - 2 \textstyle{\int_0^T \exp[2\int_{T-t}^T a(s) \rmd s] [\int_{T-t}^T a(s)^2 \rmd s + a(T) - a(T-t)] \rmd t } \\
    & \qquad \qquad \qquad = (1/2)(1 - \exp[-2T]) \eqsp . 
  \end{align}

  \end{proposition}

  \begin{proof}
    The proof is a direct consequence of \Cref{prop:expression_fat_1}, \Cref{prop:expression_fat_2} and the fact that
    \begin{align}
      &- \exp[-2T] (1 - \lambda T \exp[-2T]) (\lambda+1)^2/(1 +\lambda \exp[-2T])^2 \\      
      &\qquad =-(T+1)\exp[-2T] (1+\lambda^2)/(1+\lambda\exp[-2T])^2 \\
      &\qquad \qquad + T(1+\lambda)^2\exp[-2T]/(1+\lambda\exp[-2T]) \eqsp . 
    \end{align}
  \end{proof}

  \begin{proposition}
    \label{prop:mean_mean}
        Let $\lambda \in \ooint{-1, +\infty}$ and $a: \ \ccint{0,T} \to \rset$ such
  that for any $t \in \ccint{0,T}$,
  \begin{equation}
    a(t) = 1 - 2/(1 + \exp[-2(T-t)]\lambda) \eqsp . 
  \end{equation}
  Then, we have that
  \begin{align}
    &\textstyle{ \exp[2\int_0^T a(t) \rmd t] - 1 - \int_0^T\exp[2 \int_{T-t}^T a(s) \rmd s] \{ \int_{T-t}^T a(s)^2 \rmd s + a(T) - 3 a(T-t) \} \rmd t } \\
    & \qquad = \exp[-2T] (1+\lambda)^2 /(1+ \lambda\exp[-2T])^2 - 1 \\
    & \qquad \qquad -(1/2)(1+\lambda)  (1 -  \exp[-2T])/(1 + \lambda \exp[-2T])(1 -2/(1+\lambda)) \\
    & \qquad \qquad -(3/2)(1- \exp[-2T])(1 - \lambda^2 \exp[-2T])/(1 + \lambda \exp[-2T])^2 \\
& \qquad \qquad + (T/2)(1+\lambda)^2\exp[-2T]/(1+\lambda\exp[-2T]) \\
    & \qquad \qquad - (1+\lambda)^2/(4\lambda)\log((1+\lambda)/(1 + \lambda \exp[-2T])) \\
    & \qquad \qquad +(\lambda/2)(1-\exp[-2T])^2/(1 + \lambda \exp[-2T])^2 \eqsp .    
  \end{align}
  In particular, we have
  \begin{align}
    &\textstyle{ \exp[2\int_0^T a(t) \rmd t] - 1 - \int_0^T\exp[2 \int_{T-t}^T a(s) \rmd s] \{ \int_{T-t}^T a(s)^2 \rmd s + a(T) - 3 a(T-t) \} \rmd t } \\
    & \qquad \qquad =  -2 - (1+\lambda)^2/(4\lambda)\log(1+\lambda) + O(\exp[-2T])  \eqsp . 
  \end{align}
\end{proposition}

\begin{proof}
  Using \Cref{lemma:integral_uno}, we have that
  \begin{equation}
    \textstyle{\exp[2\int_0^T a(t) \rmd t] = \exp[-2T] (1+\lambda)^2 /(1+ \lambda\exp[-2T])^2 \eqsp . } \label{eq:uno_exp}
  \end{equation}
  Using \Cref{lemma:integral_exp_uno}, we have
  \begin{equation}
    \textstyle{ \int_0^T \exp[2  \int_{T-t}^T a(s) \rmd s] \rmd t =  (1/2)(1+\lambda)  (1 -  \exp[-2t])/(1 + \lambda \exp[-2t]) }\eqsp .\label{eq:duo_exp}
  \end{equation}
  Using \Cref{lemma:hola_pas_carre}, we have 
  \begin{align}
    &\textstyle{
      \int_0^T \exp[ 2\int_{T-t}^T a(s) \rmd s] a(T-t) \rmd t} \\
    & \qquad \qquad \textstyle{= -(1/2)(1- \exp[-2T])(1 - \lambda^2 \exp[-2T])/(1 + \lambda \exp[-2T])^2 \eqsp . 
      } \label{eq:tertio_exp}
  \end{align}
  Finally, using \Cref{lemma:hola_carre}, we have
  \begin{align}
    &\textstyle{\int_0^T \exp[2\int_{T-t}^T a(s) \rmd s] (\int_{T-t}^T a(s)^2 \rmd s) \rmd t} = - (T/2)(1+\lambda)^2\exp[-2T]/(1+\lambda\exp[-2T]) \\
    & \qquad \qquad + (1+\lambda)^2/(4\lambda)\log((1+\lambda)/(1 + \lambda \exp[-2T])) \\
    & \qquad \qquad -(\lambda/2)(1-\exp[-2T])^2/(1 + \lambda \exp[-2T])^2 \eqsp . \label{eq:quatro_exp}
  \end{align}
  We conclude upon combining \eqref{eq:uno_exp}, \eqref{eq:duo_exp},
  \eqref{eq:tertio_exp}, \eqref{eq:quatro_exp} and that
  $a(T) = 1- 2/(1+\lambda)$.
\end{proof}

\subsection{General setting}
\label{sec:main-content}

In this section, we prove \Cref{thm:control_diffusion}. In order to compare our
results with \cite[Theorem 1]{debortoli2021diffusion}, we redefine a few
processes.  Let $p \in \Pens(\rset^d)$ be the target distribution. Consider the
Ornstein-Ulhenbeck forward dynamics $(x_t)_{t \in \ccint{0,T}}$ such that
$\rmd x_t = -x_t \rmd t + \sqrt{2} \rmd w_t$ and $x_0$ has distribution
$p_0$. We consider the backward chain $(X_k)_{k \in \{0, \dots, N\}}$ such that
for any $k \in \{0, \dots, N-1\}$,
\begin{equation}
  \label{eq:reverse_discrete}
  X_k = X_{k+1} + \gamma_{k+1} \{ X_{k+1} + 2 \nabla \log p_{t_{k+1}}(X_{k+1}) \} + \sqrt{2 \gamma_{k+1}} Z_{k+1}  ,
\end{equation}
with $\{Z_k\}_{k \in \nset}$ a family of i.i.d. Gaussian random variables with
zero mean and identity covariance matrix, $t_{k} = \sum_{\ell=1}^k \gamma_{\ell}$,
$\sum_{\ell=1}^{N} \gamma_{\ell} = T$ and $X_N$ has distribution
$p_0 = \mathrm{N}(0, \Id)$ (independent from $\{Z_k\}_{k \in \nset}$). Notice
that here we do not consider a score approximation in the recursion in order to
clarify our approximation results. We recall the following result from
\cite[Theorem 1]{debortoli2021diffusion}.

\begin{theorem}
  Assume that $p_0$ admits a bounded density (w.r.t. the Lebesgue measure)
  $p_0 \in \rmc^3(\rset^d, \ooint{0, +\infty})$  and that there
  exist $d_1, A_1, A_2, A_3 \geq 0$, $\beta_1, \beta_2, \beta_3 \in \nset$ and
  $\mtt_1 > 0$ such that for any $x \in \rset^d$ and $i \in \{1, 2, 3\}$
            \begin{equation}
              \textstyle{
              \normLigne{\nabla^i \log p_0(x)} \leq A_i(1 + \normLigne{x}^{\beta_i}) , \quad \langle \nabla \log p_0(x), x \rangle \leq -\mtt_1 \norm{x}^2 + d_1 \norm{x} ,}
          \end{equation}
          with $\beta_1 = 1$.  Then  there exist
          $B, C, D \geq0$ such that for any $N \in \nset$
          and $\{\gamma_k\}_{k=1}^N$ with $\gamma_k > 0$ for any
          $k \in \{1, \dots, N\}$ we have 
          \begin{equation}
            \label{eq:ineq_thm_1}
 \tvnormLigne{\mathcal{L}(X_0)-p_0} \leq   C \exp[D T] \sqrt{\gamma^\star} + B \exp[-T]  .          
          \end{equation}
          where $\gamma^\star = \sup_{k \in \{1, \dots, N\}} \gamma_k$ and
          $\mathcal{L}(X_0)$ is the distribution of $X_0$ given in \eqref{eq:reverse_discrete}.
        \end{theorem}

In the rest of this note we improve the theorem in the following way:
\begin{enumerate}[label=(\alph*)]
\item We remove the exponential dependency w.r.t. the time in the first term of the RHS of \eqref{eq:ineq_thm_1}.
\item We provide explicit bounds $B, C, D \geq 0$ depending on the parameters of $p_0$.
\end{enumerate}

\begin{lemma}
  \label{lemma:control_growth}
  Assume
  \begin{equation}\label{params:thm2_app}
    \textstyle{
  \sup_{x,t} \normLigne{\nabla^2 \log p_t(x)} \leq
  \Ktt}~~\mbox{and}~~
  \textstyle{\normLigne{\partial_t \nabla \log p_t(x)} \leq \Mtt\, \rme^{-\alpha t}\,
  \normLigne{x}}.
  \end{equation}
  Then there exists $D \geq 0$ such that for any
  $x \in \rset^d$ and $t \in \ccint{0,T}$,
  $\normLigne{\nabla \log p_t(x)} \leq D(1 + \normLigne{x})$ with
  $D = \normLigne{\nabla \log p_0(0)} + \Ktt + C T$.
\end{lemma}

\begin{proof}
  Let $x \in \rset^d$ and $t \in \ccint{0,T}$. Since
  $(t,x) \mapsto \log p_t(x) \in \rmc^{2}(\ccint{0,T}\times \rset^d,
  \ooint{0,+\infty})$, we have that
  \begin{align}
    \textstyle{
      \nabla \log p_t(x)} &= \textstyle{\nabla \log p_0(x) + \int_0^t \partial_s \nabla \log p_s(x) \rmd s }\\
    &=  \textstyle{\nabla \log p_0(0) + \int_0^1 \nabla^2 \log p_{0}(ux)(x) \rmd u +  \int_0^t \partial_s \nabla \log p_s(x) \rmd s .
      }
  \end{align}
  Therefore,  we have that
  \begin{align}
    \normLigne{\nabla \log p_t(x)} &\leq \textstyle{\normLigne{\nabla \log p_0(0)} + \Ktt \normLigne{x} +  \int_0^t \normLigne{\partial_s \nabla \log p_s(x)} \rmd s } \\
                                   &\leq \textstyle{\normLigne{\nabla \log p_0(0)} + \Ktt \normLigne{x} +  \Mtt \sum_{k=0}^{N-1}(t_k - t_{k-1}) \exp[-\alpha t_k] \normLigne{x} } \\
                                   &\leq \textstyle{\normLigne{\nabla \log p_0(0)} + \Ktt \normLigne{x} +  \Mtt T \normLigne{x} } ,
  \end{align}
which concludes the proof.
\end{proof}

Note that in the previous proposition we can derive a tighter bound for $D$
which does not depend on the limiting time $T > 0$. However, we do not use the
bound $D>0$ in our quantitative result and therefore our simple bound suffices.

We also have the following useful lemma.

\begin{lemma}
  \label{item:bound_p_infty}
  Let $T \geq \log(2 \expeLigne{\normLigne{X_0}^2}) + \log(2)/2$ and assume that there exists $\eta > 0$ such that $aaa$ Then, we have
  \begin{equation}    
    \textstyle{
      \int_{\rset^d} p_\infty(x_T)^2 / p_T(x_T) \rmd x_T \leq \exp[4] + E_T \eqsp ,
      }
    \end{equation}
    with $E_T  \sim C\exp[-T]$ when $T \to +\infty$ and $C \geq 0$.
\end{lemma}

If $p_\infty$ satisfies the following $\Phi$-entropy inequality for any
$f: \ \rset_d \to \ooint{0,\infty}$ measurable
\begin{equation}
  \label{eq:functional_inequality}
  \textstyle{
    \int_{\rset^d} \normLigne{\nabla f(x)}^2/f(x)^3 p_\infty(x) \rmd x \leq C [\int_{\rset^d} (1/f(x)) p_\infty(x) \rmd x - 1/(\int_{\rset^d} f(x) p_\infty(x) \rmd x)] \eqsp ,
    }
  \end{equation}
  with $C \geq 0$. Then, we have as in  \cite[Proposition 7.6.1]{bakry:gentil:ledoux:2014}
  \begin{equation}
    \textstyle{
      \chi^2(p_\infty||p_t) = \int_{\rset^d} p_\infty^2(x)/p_T(x) \rmd x - 1 \leq \rme^{-Ct} \eqsp ,
      }
  \end{equation}
  which immediately concludes the proof of \Cref{item:bound_p_infty}. However,
  to the best of our knowledge, establishing \eqref{eq:functional_inequality}
  remains an open problem. Note that controlling $\chi^2(p_t||p_\infty)$ is much
  easier as the exponential decay of this divergence is linked with the
  Poincar\'e inequality which is satisfied in our Gaussian setting. In what
  follows, we consider another approach which relies on the structure of the
  Ornstein-Ulhenbeck transition kernel and provide non-tight upper bounds.
  
\begin{proof}
  Let $T \geq 0$, $\vareps >0$ and $x_T \in \rset^d$
  \begin{equation}
    \normLigne{x_T - \rme^{-T} x_0}^2 \leq (1 + \vareps) \normLigne{x_T}^2 + (1+1/\vareps) \rme^{-2T} \normLigne{x_0}^2 \eqsp . 
  \end{equation}
  Let $\vareps >0$ and $x_T \in \rset^d$, we have
  \begin{align}
    &\textstyle{
      p_T(x_T)^{-2} \leq \exp[(1+\vareps)/\sigma_T^2\normLigne{x_T}^2 ] }\\
    & \qquad \qquad \qquad \qquad \times \textstyle{(\int_{\rset^d} p(x_0) \exp[-\rme^{-2T}(1+1/\vareps)/(2\sigma_T^2)\normLigne{x_0}^2]\rmd x_0)^{-2} (2\uppi\sigma_T^2)^d  \eqsp .
      }
  \end{align}
  For any $x_T \in \rset^d$, we have
  \begin{align}
    p_\infty(x_T)^2/p_T(x_T) &\leq \exp[\{-1+(1+\vareps)/(2\sigma_T^2)\}\normLigne{x_T}^2 ](2\uppi/\sigma_T^2)^{-d/2} \\ & \qquad \qquad \textstyle{ (\int_{\rset^d} p(x_0) \exp[-\rme^{-2T}(1+1/\vareps)/(2\sigma_T^2)\normLigne{x_0}^2] \rmd x_0)^{-1}  \eqsp . }
  \end{align}
  In what follows, we set $\vareps = \rme^{-T}$. We have that
  \begin{equation}
    -1 + (1+\vareps)/(2\sigma_T^2) = (2\sigma_T^2)^{-1}(-2\sigma_T^2 + 1 + \vareps) = -(1 - 2\rme^{-T} + \vareps)/(2\sigma_T^2) = -(1 - \rme^{-T})/(2\sigma_T^2) \eqsp . 
  \end{equation}
  Therefore, we get that
  \begin{equation}
    \label{eq:UNO_de}
    \textstyle{
      \int_{\rset^d} \exp[\{-1+(1+\vareps)/(2\sigma_T^2)\}\normLigne{x_T}^2 ](2\uppi/\sigma_T^2)^{-d/2} \rmd x_T = (1 - \rme^{-T})^{-d/2} \eqsp .
      }
    \end{equation}
    In addition, we have that for any $R \geq 0$ using that $\sigma_T^2 \geq 1/2$ since $T \geq \log(2)/2$
    \begin{align}
      &\textstyle{
        \int_{\rset^d} p(x_0) \exp[-\rme^{-2T}(1+1/\vareps)/\sigma_T^2\normLigne{x_0}^2] \rmd x_0 } \\
      & \qquad \qquad  \geq \textstyle{\probaLigne{X_0 \in \cball{0}{R}} \exp[-\rme^{-2T}(1+1/\vareps)/\sigma_T^2R^2]} \\ 
        & \qquad \qquad  \geq \textstyle{\probaLigne{X_0 \in \cball{0}{R}} \exp[-4\rme^{-T}R^2]
        } \label{eq:inter_DUO}
    \end{align}
    Now let $R^2 = \rme^T$. We obtain
    \begin{equation}
      \textstyle{
        \int_{\rset^d} p(x_0) \exp[-\rme^{-2T}(1+1/\vareps)/\sigma_T^2\normLigne{x_0}^2] \rmd x_0 } \geq \probaLigne{X_0 \in \cball{0}{\rme^{T/2}}} \exp[-4] \eqsp . 
    \end{equation}
    In addition, using Markov inequality, we have
    \begin{equation}
      \probaLigne{X_0 \in \cball{0}{\rme^{T/2}}} = 1 - \probaLigne{\normLigne{X_0}^2 \geq \rme^T} \geq 1 - \expeLigne{\normLigne{X_0}^2} \rme^{-T} \geq 0 \eqsp . 
    \end{equation}
    Therefore, combining this result and \eqref{eq:inter_DUO}, we have
    \begin{equation}
      \label{eq:DUO_de}
      \textstyle{
        \int_{\rset^d} p(x_0) \exp[-\rme^{-2T}(1+1/\vareps)/\sigma_T^2\normLigne{x_0}^2] \rmd x_0 } \geq \exp[-4] (1 - \expeLigne{\normLigne{X_0}^2} \rme^{-T}) > 0 \eqsp . 
    \end{equation}
    We conclude upon combining \eqref{eq:UNO_de} and \eqref{eq:DUO_de}.
\end{proof}

We are now ready to state the following lemma.

\begin{lemma}
  \label{lemma:strong_solution_backward}
  There exists a unique strong solution to the SDE
  $\rmd y_t = \{ y_t + 2 \nabla \log p_{T-t}(y_t)\} \rmd t + \sqrt{2}
  \rmd w_t$ with initial condition $\mathcal{L}(y_0) = p_\infty$. In
  addition, we have that
  $\expeLigne{\sup_{t \in \ccint{0,T}}\normLigne{y_t}^\alpha}<+\infty$ for any 
  $\alpha >0$.
\end{lemma}

\begin{proof}
  Let $b: \ \ccint{0,T} \times \rset^d$ given for any $t \in \ccint{0,T}$ and
  $x \in \rset^d$ by $b(t,x) = x + 2 \nabla \log p_t(x)$. We have that
  $b \in \rmc^1(\ccint{0,T} \times \rset^d, \rset^d)$ and in particular is
  locally Lipschitz. In addition, using \Cref{lemma:control_growth} we have that
  for any $t \in \ccint{0,T}$ and $x \in \rset^d$,
  $\normLigne{b(t,x)} \leq (1 + D)\normLigne{x}$. Hence using \cite[Theorem 2.3,
  Theorem 3.1]{ikeda1989sto} and \cite[Theorem 2.1]{meyn1993criteria_iii} (with
  $V(x) = (1/2)\normLigne{x}^2$) there exists a unique strong solution to the
  SDE
  $\rmd y_t = \{ y_t + 2 \nabla \log p_{T-t}(y_t)\} \rmd t + \sqrt{2}
  \rmd w_t$ with initial condition $\mathcal{L}(y_0) = p_\infty$. Let
  $\alpha > 1$, then we have for any $t \in \ccint{0,T}$
  \begin{align}
    \textstyle{\sup_{s \in \ccint{0,t}} \normLigne{y_t}^\alpha} &\leq \textstyle{3^{\alpha-1} [\normLigne{y_0}^\alpha + t^{\alpha-1}(1+D)^\alpha\int_0^t \sup_{u \in \ccint{0,s}} \normLigne{y_u}^\alpha \rmd u + 2^{\alpha/2} \sup_{s \in \ccint{0,t}} \normLigne{w_u}^\alpha]} . 
  \end{align}
  Using that $\expeLigne{\sup_{s \in \ccint{0,T}} \normLigne{w_u}^\alpha}$
  and Gr\"onwall's lemma, we get that
  $\expeLigne{\textstyle{\sup_{t \in \ccint{0,T}}
      \normLigne{y_t}^\alpha}}<+\infty$ for any $\alpha >1$. The result is
  extended to any $\alpha >0$ since for any $\alpha \in \ocint{0,1}$ we have
  that
  \begin{equation}
    \textstyle{\expeLigne{\textstyle{\sup_{t \in \ccint{0,T}}
      \normLigne{y_t}^\alpha}} \leq \expeLigne{\textstyle{\sup_{t \in \ccint{0,T}}
      \normLigne{y_t}}}^\alpha <+\infty} .
  \end{equation}

\end{proof}

We are now ready to prove \Cref{thm:control_diffusion}.

\begin{proof}
  The beginning of the proof is similar to the one of \cite[Theorem
  1]{debortoli2021diffusion}.  For any $k \in \{1, \dots, N\}$, denote $\Rker_k$
  the Markov kernel such that for any $x \in \rset^d$, $\msa \in \mcb{\rset^d}$
  and $k \in \{0, \dots, N-1\}$ we have
    \begin{equation}
      \textstyle{\Rker_{k+1}(x, \msa) = (4 \uppi \gamma_{k+1})^{-d/2} \int_{\msa} \exp[-\norm{\tilde{x} - \Tnplusun(x)}^2/(4 \gamma_{k+1})] \rmd \tilde{x}  , }
    \end{equation}
    where for any $x \in \rset^d$,
    $\Tnplusun(x) = x + \gamma_{k+1} \defEnsLigne{x + 2 \nabla \log
      p_{t_{k+1}}(x)}$. Define for any $k_0, k_1 \in \{1, \dots, N\}$ with
    $k_1 \geq k_0$ $\Qker_{k_0, k_1} = \prod_{\ell = k_0}^{k_1}
    \Rker_{k_1 + k_0 - \ell}$. Finally, for ease of notation, we also define for any
    $k \in \{1, \dots, N\}$, $\Qker_k = \Qker_{k+1,N}$. Note that for any
    $k \in \{1, \dots, N\}$, $X_k$ has distribution $p_{\infty} \Qker_k$,
    where $p_{\infty}\in \Pens(\rset^d)$ with density w.r.t. the Lebesgue
    measure $\pref$. Let $\Pbb \in \Pens(\contspace)$ be the probability measure
    associated with the diffusion
    \begin{equation}
      \rmd x_t = - x_t \rmd t + \sqrt{2} \rmd w_t  , \quad x_0 \sim p_0  , 
    \end{equation}
    First, we have for any $\msa \in \mcb{\rset^d}$
    \begin{equation}
      p_0 \Pbb_{T|0} (\Pbb^R)_{T|0} (\msa) = \Pbb_T (\Pbb^R)_{T|0} (\msa) = (\Pbb^R)_0 (\Pbb^R)_{T|0} (\msa)  = (\Pbb^R)_T(\msa) = p_0(\msa)  . 
    \end{equation}
    Hence $p_0 = p_0 \Pbb_{T|0} (\Pbb^R)_{T|0}$.  Using this result 
    we have
    \begin{align}
      \tvnormLigne{p_0 - p_{\infty} \Qker_0} &= \tvnormLigne{p_0 \Pbb_{T|0} (\Pbb^R)_{T|0} - p_{\infty} \Qker_0} \\
                                         &\leq \tvnormLigne{p_0 \Pbb_{T|0} (\Pbb^R)_{T|0} - p_{\infty} (\Pbb^R)_{T|0}} +\tvnormLigne{p_{\infty} (\Pbb^R)_{T|0} - p_{\infty} \Qker_0} \\
          &\leq \tvnormLigne{p_0 \Pbb_{T|0}- p_{\infty}} +\tvnormLigne{p_{\infty} (\Pbb^R)_{T|0} - p_{\infty} \Qker_0}  .                                  
    \end{align}
    Note that $\mathcal{L}(X_0)= p_{\infty} \Qker_0$ and therefore
    \begin{equation}
      \tvnormLigne{\mathcal{L}(X_0) - p_0} \leq \tvnormLigne{p_0 \Pbb_{T|0}- p_{\infty}} +\tvnormLigne{p_{\infty} (\Pbb^R)_{T|0} - p_{\infty} \Qker_0}  .
    \end{equation}
    We now bound each one of these terms.
    \begin{enumerate}[wide, labelwidth=!, labelindent=0pt, label=(\alph*)]
    \item First, we bound $\tvnormLigne{p_0 \Pbb_{T|0}- p_{\infty}}$. Using the 
      Pinsker inequality \cite[Equation 5.2.2]{bakry:gentil:ledoux:2014} we have
      that
      \begin{equation}
        \label{eq:pinsker}
        \tvnormLigne{p_0 \Pbb_{T|0}- p_{\infty}} \leq \sqrt{2} \KLLignesqrt{p_0 \Pbb_{T|0}}{p_{\infty}}  . 
      \end{equation}
      In addition, $p_\infty$ satisfies the log-Sobolev inequality with
      constant $C=1$, \cite{gross1975logarithmic}. Namely, for any
      $f \in \rmc^1(\rset^d, \ooint{0,+\infty})$ such that
      $f \in \rmL^1(p_\infty)$ and
      $\int_{\rset^d} \normLigne{\nabla \log f(x)}^2 f(x) \rmd p_\infty(x) <
      +\infty$ we have
      \begin{align}
        &\textstyle{
          \int_{\rset^d} f(x) \log f(x) \rmd p_{\infty}(x) - (\int_{\rset^d} f(x) \rmd p_{\infty}(x)) (\log \int_{\rset^d} f(x) \rmd p_{\infty}(x)) } \\
        & \qquad \qquad \qquad \qquad  \textstyle{\leq (C/2) \int_{\rset^d}  \normLigne{\nabla \log f(x)}^2 f(x) \rmd p_\infty(x)  ,          
          }
      \end{align}
      with $C=1$. Therefore, using \cite[Theorem
      5.2.1]{bakry:gentil:ledoux:2014} we have that for any
      $f \in \rmL^1(p_\infty)$ with
      $\int_{\rset^d} \abs{f(x)} \abs{\log f(x)} \rmd p_\infty(x) < +\infty$      
      \begin{equation}
        \label{eq:entropy_decay}
        \Ent{p_\infty}{\Pbb_{T|0}[f]} \leq \exp[-2T]  \Ent{p_\infty}{f}  ,
      \end{equation}
      where for any $g \in \rmL^1(p_\infty)$ with 
      $\int_{\rset^d} \abs{g(x)} \abs{\log g(x)} \rmd p_\infty(x) < +\infty$ we define
      \begin{equation}
        \textstyle{
          \Ent{p_\infty}{g} = \int_{\rset^d} g(x) \log g(x) \rmd p_{\infty}(x) - (\int_{\rset^d} g(x) \rmd p_{\infty}(x)) (\log \int_{\rset^d} g(x) \rmd p_{\infty}(x)) .
          }
        \end{equation}
        Note that
        $(\rmd p_T / \rmd p_\infty) = \Pbb_{T|0} [\rmd p_0 / \rmd
        p_\infty]$ and that for any $\mu \in \Pens(\rset^d)$ with
        $\KL{\mu}{p_\infty} < +\infty$ we have
        $\Ent{p_\infty}{\rmd \mu / \rmd p_\infty} =
        \KLLigne{\mu}{p_\infty}$. Using these results, \eqref{eq:pinsker} and 
        \eqref{eq:entropy_decay} we get that
        \begin{equation}
          \label{eq:first_bound_kl}
        \tvnormLigne{p_0 \Pbb_{T|0}- p_{\infty}} \leq \sqrt{2} \exp[-T] \KLLignesqrt{p_0}{p_\infty}  . 
      \end{equation}
      In addition, we have that
      \begin{equation}
        \textstyle{
          \KLLigne{p_0}{p_\infty} = (d/2) \log(2 \uppi) + \int_{\rset^d} \normLigne{x}^2 \rmd p_0(x) - \mathrm{H}(p_0)  ,
          }
        \end{equation}
        where $\mathrm{H}(p_0) = - \int_{\rset^d} \log(p_0(x)) p_0(x) \rmd x$.
        Combining this result and \eqref{eq:first_bound_kl} we get that
        \begin{equation}
          \label{eq:first_bound_kl_clop}
          \textstyle{
            \tvnormLigne{p_0 \Pbb_{T|0}- p_{\infty}} \leq \sqrt{2} \exp[-T] ((d/2) \log(2 \uppi) + \int_{\rset^d} \normLigne{x}^2 \rmd p_0(x) - \mathrm{H}(p_0))^{1/2}   ,
            }
          \end{equation}
          which concludes the first part of the proof.
        \item First, let $\Qbb \in \Pens(\mathcal{C})$ such that
          $\Qbb = p_\infty \Pbb^R_{|0}$, where $\Pbb^R_{|0}$ is the
          disintegration of $\Pbb^R$ w.r.t. $\phi: \ \contspace \to \rset^d$
          given for any $\omega \in \contspace$ by $\phi(\omega) = \omega_T$,
          see \cite{leonard2014some} for instance. Note that for any
          $f \in \rmc(\contspace)$ with $f$ bounded we have
      \begin{align}
        \textstyle{
        \Qbb[f] = \int_{\rset^d} \int_{\contspace} f(\omega) \Pbb^{R}_{|0}(\omega_0, \rmd \omega) \rmd p_\infty(\omega_0)} &= \textstyle{\int_{\rset^d} \int_{\contspace} f(\omega) \Pbb^{R}_{|0}(\omega_0, \rmd \omega) (\rmd p_\infty / \rmd p_T)(\omega_0)\rmd p_T(\omega_0) } \\
&= \textstyle{\int_\contspace f(\omega) (\rmd p_\infty / \rmd p_T)(\omega_0) \rmd \Pbb^R(\omega)}  . 
      \end{align}
      Therefore, we get that for any $\omega \in \contspace$,
      $(\rmd \Qbb / \rmd \Pbb^R)(\omega) = (\rmd p_\infty / \rmd
      p_T)(\omega_0)$.  Let
      $\Rbb = p_\infty \Pbb_{|0}$. Note that for any $t \in \ccint{0,T}$,
      $\Rbb_t = p_\infty$ and that $\Rbb$ is associated with the process
      $\rmd x_t = -x_t \rmd t + \sqrt{2} \rmd w_t$ with
      $\mathcal{L}(x_0) = p_\infty$. In particular, $\Rbb$ satisfies
      \cite[Hypothesis 1.8]{cattiaux2021time}. Using \cite[Theorem
      2.4]{leonard2014some} we have that
      \begin{equation}
        \textstyle{\KLLigne{\Pbb}{\Rbb} = \KLLigne{p_0}{p_\infty} + \int_{\rset^d} \KLLigne{\Pbb_{|0}(x_0)}{\Pbb_{|0}(x_0)} \rmd p_0(x_0) = \KLLigne{p_0}{p_\infty} < +\infty.}
      \end{equation}
      Therefore, we can apply \cite[Theorem 4.9]{cattiaux2021time}. Let
      $u \in \rmc^\infty_c(\rset^d, \rset)$, we have that 
      $(\bfM_t^u(y))_{t \in \ccint{0,T}}$ is a local martingale, where we
      have for any $t \in \ccint{0,T}$
      \begin{equation}
        \textstyle{\bfM_t^u(y) = u(y_t) - u(y_0) - \int_0^t \{ \langle \nabla u(y_s), y_s + 2\nabla \log p_{T-s}(y_s) \rangle + \Delta u(y_s) \}\rmd s ,}
      \end{equation}
      where $\mathcal{L}(y) = \Pbb^R$.  Since $u$ is compactly supported we
      have that
      $\sup_{\omega \in \contspace} \sup_{t \in \ccint{0,T}}
      \absLigne{\bfM_t^u(\omega)} < +\infty$ and therefore
      $(\bfM_t^u(y))_{t \in \ccint{0,T}}$ is a martingale. We now show that
      $(\bfM_t^u(y))_{t \in \ccint{0,T}}$ is a martingale, with
      $\mathcal{L}(y) = \Qbb$. Since
      $\sup_{\omega \in \contspace} \sup_{t \in \ccint{0,T}}
      \absLigne{\bfM_t^u(\omega)} < +\infty$, we have that for any
      $t \in \ \ccint{0,T}$, $\expeLigne{\absLigne{\bfM_t^u}}<+\infty$. Let
      $t,s \in \ccint{0,T}$ with $t > s$ and $g : \contspace \to \rset^d$
      bounded.  We have that
      $\expeLigne{\absLigne{g(\{x_{T-s}\}_{s \in \ccint{0,t}})}^2(\rmd p_\infty / \rmd
        p_T)(x_T)^2} < +\infty$. Hence, we have that
      \begin{equation}
        \expeLigne{(\bfM_t^u(x_{T-\cdot}) - \bfM_s^u(x_{T-\cdot})) g(\{x_{T-s}\}_{s \in \ccint{0,t}})(\rmd p_\infty / \rmd
        p_T)(x_T)} = 0 .
      \end{equation}
      Using this result and that for any $\omega \in \contspace$,
      $(\rmd \Qbb / \rmd \Pbb^R)(\omega) = (\rmd p_\infty / \rmd
      p_T)(\omega_0)$ we get
      \begin{equation}
        \expeLigne{(\bfM_t^u(y) - \bfM_s^u(y)) g(\{y_s\}_{s \in \ccint{0,t}})} = 0 .
      \end{equation}
      Hence, for any $u \in \rmc^2_c(\rset^d, \rset)$,
      $(\bfM_t^u(y))_{t \in \ccint{0,T}}$ is a martingale. In addition,
      $(\bfM_t^u(\bfZ))_{t \in \ccint{0,T}}$ is a martingale using
      \Cref{lemma:strong_solution_backward} and It\^o's lemma, where $\bfZ$ is
      the solution to the SDE in \Cref{lemma:strong_solution_backward}. In
      addition, we have that
      $\mathcal{L}(\bfZ_0) = \mathcal{L}(y_0) = p_\infty$. Using
      \Cref{lemma:control_growth} and the remark following \cite[Hypothesis
      1.8]{cattiaux2021time}, we get that
      $\mathcal{L}(\bfZ) = \mathcal{L}(y) = \Qbb$. We have just shown that the
      time-reversed process with initialisation $p_\infty$ can be obtained as a
      strong solution of an SDE.  Using \Cref{lemma:control_growth} and
      \Cref{lemma:strong_solution_backward}, we have that for any
      $t \in \ccint{0,T}$
      \begin{equation}
        \textstyle{\expeLigne{\int_0^t \normLigne{x_s + 2\nabla \log p_s(x_s)}^2 \rmd s + \int_0^t \normLigne{w_s + 2\nabla \log p_s(w_s)}^2 \rmd s}<+\infty.}
      \end{equation}
      Combining this result and \cite[Lemma S13]{debortoli2021diffusion} we have
      that
      \begin{equation}
        \label{eq:girsanov_uno}
    \textstyle{
    \tvnormLigne{p_\infty \Pbb^R_{T|0} - p_\infty \Qker_0}^2  \leq  (1/2) \int_0^T \expeLigne{\normLigne{b_1(t, (y_s)_{s \in \ccint{0,T}}) - b_2(t, (y_s)_{s \in \ccint{0,T}})}^2} \rmd t  ,}
  \end{equation}
      where for any $t \in \ccint{0,T}$ and $\omega \in \contspace$ we have that
      \begin{equation}
        b_1(t, \omega) = \omega_t + 2\nabla \log p_{T-t}(\omega_t)  , \qquad b_2(t, \omega) = \omega_{t_\gamma} + 2\nabla \log p_{T- t_\gamma}(\omega_{t_\gamma})  ,
      \end{equation}
      where
      $t_\gamma = \sum_{k=0}^{N-1} \mathbb{1}_{\coint{T - t_{k+1}, T -
          t_{k}}}(t) (T - t_{k+1})$. Noting that $(y_t)_{t \in \ccint{0,T}}$
      is distributed according to $\Qbb$ and using that
      $(\rmd \Qbb / \rmd \Pbb^R)(\omega) = (\rmd p_\infty / \rmd
      p_T)(\omega_0)$, \eqref{eq:girsanov_uno} and the Cauchy-Schwarz
      inequality we have
      \begin{align}
        \label{eq:pinsker_girsanov}
        &\textstyle{
          \tvnormLigne{p_\infty \Pbb^R_{T|0} - p_\infty \Qker_0}^2}  \\
        &\textstyle{\leq (1/2) \expeLigne{(\rmd p_\infty / \rmd
          p_T)(x_T)^2}^{1/2} \int_0^T \expeLignesqrt{\normLigne{b_1(t, (x_{T-s})_{s \in \ccint{0,T}}) - b_2(t, (x_{T-s})_{s \in \ccint{0,T}})}^4} \rmd t }  \\
        & \textstyle{\leq (1/2) \expeLigne{(\rmd p_\infty / \rmd
          p_T)(x_T)^2}^{1/2}  }\\
        & \qquad \textstyle{\times \int_0^T \expeLignesqrt{\normLigne{b_1(T-t, (x_{T-s})_{s \in \ccint{0,T}}) - b_2(T-t, (x_{T-s})_{s \in \ccint{0,T}})}^4} \rmd t }  . 
      \end{align}
In addition, we have that for any $t \in \ccint{0,T}$ and $\omega \in \contspace$ we have
\begin{align}
  &\normLigne{b_1(t, \omega)  - b_2(t, \omega)} \\
  & \leq \normLigne{\omega_t - \omega_{t_\gamma}} + 2 \normLigne{\nabla \log p_{T-t}(\omega_t) - \nabla \log p_{T-t_\gamma}(\omega_t)} \\
  & \qquad \qquad + 2 \normLigne{\nabla \log p_{T-t_\gamma}(\omega_t) - \nabla \log p_{T-t_\gamma}(\omega_{t_\gamma})} \\
  &  \textstyle{\leq (1 + 2\sup_{s \in \ccint{0,T}} \sup_{x \in \rset^d}\normLigne{\nabla^2 \log p_s(x)})} \normLigne{\omega_t - \omega_{t_\gamma}} \\
  & \qquad \qquad \textstyle{+ 2 \sup_{s \in \ccint{T-t, T- t_{\gamma}}} \normLigne{\partial_t \nabla \log p_t(\omega_t)}(t - t_\gamma)} \\
&  \textstyle{\leq (1 + 2\Ktt) \normLigne{\omega_t - \omega_{t_\gamma}} + 2 \sup_{s \in \ccint{T-t, T- t_{\gamma}}} \normLigne{\partial_s \nabla \log p_s(\omega_t)} (t - t_\gamma)  }  . 
\end{align}
Note that
\begin{equation}
  \textstyle{
  T - (T - t)_\gamma = T - \sum_{k=0}^{N-1} \mathbb{1}_{\coint{T - t_{k+1}, T -
          t_{k}}}(T - t) (T - t_{k+1}) = \sum_{k=0}^{N-1} \mathbb{1}_{\ocint{t_k,t_{k+1}}}(t) t_{k+1}  . }
\end{equation}
For any $t \in \ccint{0,T}$, denote
$t^\gamma = T - (T - t)_\gamma = \sum_{k=0}^{N-1}
\mathbb{1}_{\ocint{t_k,t_{k+1}}}(t) t_{k+1}$. Therefore, we get that for any
$t \in \ocint{t_k,t_{k+1}}$
\begin{align}
  &\normLigne{b_1(T-t, \omega)  - b_2(T-t, \omega)} \\
  & \qquad \leq \textstyle{(1 + 2 \Ktt) \normLigne{\omega_{T-t} - \omega_{(T-t)_\gamma}} + 2 \sup_{s \in \ccint{t, t^\gamma}} \normLigne{\partial_s \nabla \log p_s(\omega_{T-t})} (t^\gamma - t)} \\
                                                   & \qquad \leq \textstyle{(1 + 2 \Ktt) \normLigne{\omega_{T-t} - \omega_{(T-t)_\gamma}} + 2 \sup_{s \in \ccint{t_k, t_{k+1}}} \normLigne{\partial_s \nabla \log p_s(\omega_{T-t})} \gamma_{k+1}} \\
  & \qquad \leq \textstyle{(1 + 2 \Ktt) \normLigne{\omega_{T-t} - \omega_{(T-t)_\gamma}} + 2 \mathrm{S}_{t_{k}}(\omega_{T-t}) \gamma_{k+1}}  . 
\end{align}
Combining this result and that for any $a,b \geq 0$, $(a+b)^4 \leq 8 a^4 + 8b^4$ we get that for any $t \in \ocint{t_k, t_{k+1}}$
\begin{align}
    \label{eq:almost_there}
    &\expeLigne{\normLigne{b_1(T-t, (x_{T-s})_{s \in \ccint{0,T}})  - b_2(T-t, (x_{T-s})_{s \in \ccint{0,T}})}^4} \\
    & \qquad \qquad \leq \textstyle{8(1 + 2 \Ktt)^4 \expeLigne{\normLigne{x_{t} - x_{t_k}}^4} + 16 \expeLigne{\mathrm{S}_{t_{k}}(x_t)^4} \gamma_{k+1}^4}.
\end{align}
In addition, we have that for any $t \in \ccint{0,T}$, $x_t = \exp[-t]x_0 + w_{(1-\exp[-2t])^{1/2}}$.
Hence, for any $s,t \in \ccint{0,T}$ with $t > s$ we have
\begin{equation}
  \normLigne{x_t - x_s} \leq \exp[-s](\exp[t- s] - 1) \normLigne{x_0} + \normLigne{w_{(1-\exp[-2t])} - w_{(1-\exp[-2s])}}  . 
\end{equation}
Therefore, we have that for any $s,t \in \ccint{0,T}$ with $t > s$
\begin{align}
  \expeLigne{\normLigne{x_t - x_s}^4} &\leq 8\exp[-4s](1 - \exp[-t+s])^4 \expeLigne{\normLigne{x_0}^4} + 8 \expeLigne{\normLigne{w_{(1-\exp[-2t])} - w_{(1-\exp[-2s])}}^4} \\
                                            &\leq 8\exp[-4s](1 - \exp[-t+s])^4 \expeLigne{\normLigne{x_0}^4} + 24 (\exp[-t] - \exp[-s])^2 \\
                                            &\leq 8\exp[-4s](1 - \exp[-t+s])^4 \expeLigne{\normLigne{x_0}^4} + 24 \exp[-2s](1 - \exp[-t+s])^2 \\
  &\leq 8\expeLigne{\normLigne{x_0}^4}\exp[-4s](t-s)^4  + 24 \exp[-2s](t-s)^2  . \label{eq:inter_ou}
\end{align}
In addition, using that that for any $k \in \{0, \dots, N-1\}$ and
$x \in \rset^d$, $\mathtt{S}_{t_k}(x) \leq \Mtt \exp[-\alpha t_k] \normLigne{x}$ we get
that
\begin{equation}
  \expeLigne{\mathtt{S}_{t_k}(x_t)^4} \leq 24 \Mtt^4 \exp[-4\alpha t_k] \{1 + \expeLigne{\normLigne{x_0}^4}\}  . 
\end{equation}

Combining this result, \eqref{eq:almost_there} and \eqref{eq:inter_ou} we get that for any $t \in \ocint{t_k, t_{k+1}}$
\begin{align}
  &\expeLigne{\normLigne{b_1(T-t, (x_{T-s})_{s \in \ccint{0,T}})  - b_2(T-t, (x_{T-s})_{s \in \ccint{0,T}})}^4} \\
  & \qquad \qquad \qquad \leq 64(1 + 2 \Ktt)^4\expeLigne{\normLigne{x_0}^4}\exp[-4t_k]\gamma_{k+1}^4  \\
  &\qquad \qquad \qquad \qquad   + 192 (1 + 2 \Ktt)^4 \exp[-2t_k]\gamma_{k+1}^2 + 384 \Mtt^4 \exp[-4\alpha t_k] \{1 + \expeLigne{\normLigne{x_0}^4}\} \gamma_{k+1}^4  . 
\end{align}
Using this result and that for any $a,b \geq 0$,
$(a+b)^{1/2} \leq a^{1/2} + b^{1/2}$, we have for any $t \in \ocint{t_k,t_{k+1}}$
\begin{align}
  &\expeLignesqrt{\normLigne{b_1(T-t, (x_{T-s})_{s \in \ccint{0,T}})  - b_2(T-t, (x_{T-s})_{s \in \ccint{0,T}})}^4} \\
  & \qquad \qquad \qquad \leq 8(1 + 2 \Ktt)^2\expeLignesqrt{\normLigne{x_0}^4}\exp[-2t_k]\gamma_{k+1}^2  \\
  &\qquad \qquad \qquad  \qquad + 14 (1 + 2 \Ktt)^2 \exp[-t_k]\gamma_{k+1} + 20 \Mtt^2 \exp[-2\alpha t_k] \{1 + \expeLignesqrt{\normLigne{x_0}^4}\} \gamma_{k+1}^2  . \label{eq:almost_almost_there}
\end{align}
We have that for any $\beta >0$,
\begin{equation}
  \textstyle{
 \sum_{k=0}^{N-1} \exp[-\beta t_k] \leq \sum_{k \in \nset} \exp[-\beta
 \gamma_\star k] \leq (1 - \exp[-\beta \gamma_\star])^{-1} \leq 1 + \beta / \gamma_\star  .
 }
\end{equation}
Then using this result, \eqref{eq:almost_almost_there} and \eqref{eq:pinsker_girsanov} we get that
\begin{align}
        &\tvnormLigne{p_\infty \Pbb^R_{T|0} - p_\infty \Qker_0}^2  \textstyle{\leq  \expeLigne{(\rmd p_\infty / \rmd
                                                                        p_T)(x_T)^2}^{1/2}} [4(1 + 2 \Ktt)^2\expeLignesqrt{\normLigne{x_0}^4}(1 +/(2\gamma_\star))(\gamma^\star)^3  \\
  &\qquad \qquad \qquad + 7 (1 + 2 \Ktt)^2 (1 +1/\gamma_\star)(\gamma^\star)^2 + 10 \Mtt^2 \{1 + \expeLignesqrt{\normLigne{x_0}^4}\} (1 +1/(2 \alpha\gamma_\star))  (\gamma^\star)^3]  . 
\end{align}
Therefore, we get that
\begin{align}
        &\tvnormLigne{p_\infty \Pbb^R_{T|0} - p_\infty \Qker_0}  \textstyle{\leq  \expeLigne{(\rmd p_\infty / \rmd
                                                                        p_T)(x_T)^2}^{1/4}} [2(1 + 2 \Ktt)\expeLignesqrtt{\normLigne{x_0}^4}(1 +/(2\gamma_\star)^{1/2})(\gamma^\star)^{3/2}  \\
        &\qquad \qquad \qquad \  + 3 (1 + 2 \Ktt) (1 +1/\gamma_\star^{1/2})\gamma^\star + 4 \Mtt \{1 + \expeLignesqrtt{\normLigne{x_0}^4}\} (1 +1/(2 \alpha\gamma_\star)^{1/2})  (\gamma^\star)^{3/2}] \\
  &\qquad \textstyle{\leq  \expeLigne{(\rmd p_\infty / \rmd
    p_T)(x_T)^2}^{1/4}} [6(1 + 2 \Ktt)(1+\expeLignesqrtt{\normLigne{x_0}^4}) \\
  & \qquad \qquad \qquad + 4 \Mtt \{1 + \expeLignesqrtt{\normLigne{x_0}^4}\} (1 +1/(2 \alpha)^{1/2}) ]((\gamma^\star)^{2}/\gamma_\star)^{1/2} \\
  &\qquad \textstyle{\leq  6 (1+\expeLignesqrtt{\normLigne{x_0}^4}) \expeLigne{(\rmd p_\infty / \rmd
                                                                        p_T)(x_T)^2}^{1/4}} [1 +  \Ktt + \Mtt (1 +1/(2 \alpha)^{1/2}) ]((\gamma^\star)^{2}/\gamma_\star)^{1/2}  ,
\end{align}
which concludes the proof upon using \Cref{item:bound_p_infty}.
    \end{enumerate}    
  \end{proof}

  We now check that the assumption of \Cref{thm:control_diffusion} are satisfied in a Gaussian setting.

\begin{proposition}
  \label{prop:gaussian_case}
  Assume that $p_0 = \mathrm{N}(0, \Sigma)$ and that
  $T \geq 1 +(1/2)[\log^+(\normLigne{\Sigma}) + \log(d+1)]$ then we have that
  for any $t \in \ccint{0,T}$ and $x \in \rset^d$
  \begin{equation}
    \normLigne{\nabla^2 \log p_t(x)} \leq \max(1, \normLigne{\Sigma^{-1}})  , \qquad \normLigne{\partial_t \nabla \log p_t(x)} \leq 2 \exp[-2 t] \max(1, \normLigne{\Sigma^{-1}})^2 \normLigne{\Sigma - \Id}\normLigne{x}  . 
  \end{equation}
  In addition, we have that $\int_{\rset^d} p_\infty(x)^2/p_T(x) \rmd x \leq \sqrt{2}$.
\end{proposition}

\begin{proof}
  Recall that for any $t \in \ccint{0,T}$,
  $x_t = \exp[-t] x_0 + w_{1- \exp[-2t]}$. Therefore, we have that for
  any $t \in \ccint{0,T}$, $p_t = \mathrm{N}(0, \Sigma_t)$ with
  $\Sigma_t = \exp[-2t] \Sigma + (1- \exp[-2t]) \Id$. Hence, we get that for any
  $t \in \ccint{0,T}$ and $x \in \rset^d$,
  $\nabla^2 \log p_t(x) = (\exp[-2t] \Sigma + (1 - \exp[-2t] \Id)^{-1}$. Using
  this result, we have that for any $t \in \ccint{0,T}$ and $x \in \rset^d$,
  $\normLigne{\nabla^2 \log p_t(x)} \leq \max(1, \normLigne{\Sigma^{-1}})$.  Similarly,
  for any $t \in \ccint{0,T}$ and $x \in \rset^d$ we have
  \begin{equation}
    \partial_t \nabla \log p_t(x) = \partial_t \Sigma_t^{-1} x = - \Sigma_t^{-1} (\partial_t \Sigma_t) \Sigma_t^{-1} x  . 
  \end{equation}
  Hence, for any $t \in \ccint{0,T}$ and $x \in \rset^d$ we have
  $\normLigne{\partial_t \nabla \log p_t(x)} \leq 2 \exp[-2 t] \max(1,
  \normLigne{\Sigma^{-1}})^2 \normLigne{\Sigma - \Id}\normLigne{x}$. Finally, we have
  that for any $t \in \ccint{0,T}$ and $x \in \rset^d$
  \begin{equation}
   \langle x, [2 \Id - (\exp[-2t] \Sigma + (1 - \exp[-2t]) \Id)^{-1}]x \rangle \geq (2 - (\exp[-2t] \normLigne{\Sigma^{-1}}^{-1} + (1- \exp[-2t]))^{-1})\normLigne{x}^2  . 
 \end{equation}
 Let $\vareps \in \ocint{0,1/2}$.  For any $t \in \ccint{0,T}$, we have that
 $2 - (\exp[-2t] \normLigne{\Sigma^{-1}}^{-1} + (1- \exp[-2t]))^{-1} \geq 1 -
 \vareps$ if and only if
 $\exp[-2t](1 - \normLigne{\Sigma^{-1}}^{-1}) \leq 1 -
 (1+\vareps)^{-1}$. Using that
 $-\log(1 - (1+\vareps)^{-1}) = \log(1 + \vareps^{-1})$ we have that for any
 $t \geq (1/2) \log(1 + \vareps^{-1})$
 and $x \in \rset^d$
 \begin{equation}
   p_\infty(x)^2/p_t(x) \leq \exp[-\normLigne{x}^2/4](2\uppi)^{-d/2} \det(\Sigma_t)  . 
 \end{equation}
 Combining this result and the fact that
 $\int_{\rset^d} \exp[-\normLigne{x}^2/2(1-\vareps)] \rmd x= (2(1-\vareps)
 \uppi)^{d/2}$, we get that for any 
 $t \geq (1/2) \log(1 +
 \vareps^{-1})$ 
 \begin{equation}
   \textstyle{
     \int_{\rset^d} p_\infty(x)^2/p_t(x) \rmd x \leq \int_{\rset^d} \exp[-\normLigne{x}^2/2(1-\vareps)](2\uppi)^{-d/2} \det(\Sigma_t)^{1/2}  \rmd x \leq (1-\vareps)^{d/2} \det(\Sigma_t)^{1/2} .
     }
   \end{equation}
   Let $\vareps = 1/(2d) \leq 1/2$. Note that
   $T \geq (1/2) \{-\log(\absLigne{\normLigne{\Sigma^{-1}}^{-1}-1}) + \log(1 +
   2d)\}$. Hence, we have that
   \begin{equation}
      \textstyle{
        \int_{\rset^d} p_\infty(x)^2/p_T(x) \rmd x \leq \exp[-\log(1-1/(2d))(d/2)] \det(\Sigma_T)^{1/2} .
        }
   \end{equation}
Since for any $t \in \coint{0,1/2}$, $-\log(1-t)\leq \log(2)t$ we get that
\begin{equation}
  \label{eq:inequality_inter}
    \textstyle{
        \int_{\rset^d} p_\infty(x)^2/p_T(x) \rmd x \leq 2^{1/4} \det(\Sigma_T)^{1/2} .
        }  
\end{equation}
Finally, using that $\Sigma_T = \exp[-2T] \Sigma + (1 - \exp[-2T]) \Id$ we have that
\begin{equation}
  \det(\Sigma_T)^{1/2} \leq (\exp[-2T]\normLigne{\Sigma} + 1 - \exp[-2T])^{d/2} \leq (1 + \exp[-2T] \normLigne{\Sigma})^{d/2} .
\end{equation}
Hence, using that result and that for any $t \geq 0$, $\log(1+t) \leq t$ we have
\begin{equation}
  \det(\Sigma_T)^{1/2} \leq \exp[\exp[-2T]\normLigne{\Sigma}(d/2)] \eqsp . 
\end{equation}
Since,
$T \geq (1/2) \{ \log(\normLigne{\Sigma}) + \log(d) + \log(2) -
\log(\log(2^{1/4})) \}$, we get that $\det(\Sigma_T)^{1/2} \leq 2$, which
concludes the proof upon combining this result and \eqref{eq:inequality_inter}.
\end{proof}

Therefore, we get the following simplified result in the Gaussian setting.

\begin{corollary}
  Assume that $p = \mathrm{N}(0, \Sigma)$, with
  $\normLigne{\Sigma^{-1}} \geq 1$, $\gamma^\star = \gamma_\star= \gamma >0$ and
  $T \geq 1 + (1/2)[\log^+(\normLigne{\Sigma}) + \log(d +1)]$, then we have
  \begin{align}
    \tvnormLigne{\mathcal{L}(X_0) - p_0} &\leq \textstyle{\exp[-T/2] (\log^+(\normLigne{\Sigma^{-1}}) + \normLigne{\Sigma - \Id})^{1/2}} \\
    & \qquad \qquad  \textstyle{+ 48  (1+\normLigne{\Sigma}^{1/2}d^{1/2}) \normLigne{\Sigma^{-1}}^2 [1 +  \normLigne{\Sigma - \Id} ]\sqrt{\gamma}  .} 
  \end{align}  
\end{corollary}

\begin{proof}
  Using \eqref{eq:kl_computation} and \Cref{prop:gaussian_case} we have 
  \begin{align}
    &\tvnormLigne{\mathcal{L}(X_0) - p_0} \leq \textstyle{\exp[-T/2] (-\log(\det(\Sigma)) + \mathrm{Tr}(\Sigma) -d)^{1/2}} \\
    & \qquad \qquad \qquad \qquad \qquad   \textstyle{+ 12  (1+(\int_{\rset^d} \normLigne{x}^4 \rmd p_0(x))^{1/4})  [1 +  \Ktt + 2C ]\sqrt{(\gamma^\star)^{2}/\gamma_\star} } \\
    &\qquad \leq \textstyle{\exp[-T/2] (-\log(\det(\Sigma)) + \mathrm{Tr}(\Sigma) -d)^{1/2}} \\
    & \qquad \qquad  \textstyle{+ 12  (1+(\int_{\rset^d} \normLigne{x}^4 \rmd p_0(x))^{1/4})  [1 +  \normLigne{\Sigma^{-1}} + 2\normLigne{\Sigma^{-1}}^2 \normLigne{\Sigma - \Id} ]\sqrt{(\gamma^\star)^{2}/\gamma_\star}  } \\
    &\qquad \leq \textstyle{\exp[-T/2] (-\log(\det(\Sigma)) + \mathrm{Tr}(\Sigma) -d)^{1/2}} \\
    & \qquad \qquad  \textstyle{+ 12  (1+3^{1/4}\normLigne{\Sigma}^{1/2}d^{1/2})  [1 +  \normLigne{\Sigma^{-1}} + 2\normLigne{\Sigma^{-1}}^2 \normLigne{\Sigma - \Id} ]\sqrt{(\gamma^\star)^{2}/\gamma_\star}  } \\
    &\qquad \leq \textstyle{\exp[-T/2] (-\log(\det(\Sigma)) + \mathrm{Tr}(\Sigma) -d)^{1/2}} \\
    & \qquad \qquad  \textstyle{+ 48  (1+\normLigne{\Sigma}^{1/2}d^{1/2}) \normLigne{\Sigma^{-1}}^2 [1 +  \normLigne{\Sigma - \Id} ]\sqrt{(\gamma^\star)^{2}/\gamma_\star}  } \\
    &\qquad \leq \textstyle{\exp[-T/2] (\log^+(\normLigne{\Sigma^{-1}}) + \normLigne{\Sigma - \Id})^{1/2}} \\
    & \qquad \qquad  \textstyle{+ 48  (1+\normLigne{\Sigma}^{1/2}d^{1/2}) \normLigne{\Sigma^{-1}}^2 [1 +  \normLigne{\Sigma - \Id} ]\sqrt{(\gamma^\star)^{2}/\gamma_\star}  .}             
  \end{align}  
\end{proof}


\section{Proof of \Cref{lemma:bound_spectrum}}
\label{sec:proof33}

\begin{proof}
  For any $x$ and $j$, denote $\bar{p}_{j,0}(\cdot|x)$ the distribution of
  $\twx_{j,0}$ given $x_j = x$ and $p_{j,0}$ the distribution of
  $\tx_{j,0}$. For any $j$ we have
  \begin{equation}
    \KLLigne{p_j}{p_{j,0}} = \KLLigne{p_{j+1}}{p_{j+1,0}} + \expeLigne{\KLLigne{\bar{p}_{j+1}(\cdot|x_{j+1})}{\bar{p}_{j+1,0}(\cdot|x_{j+1}) }}  . 
  \end{equation}
  By recursion, we have that
  \begin{equation}
    \textstyle{
      \KLLigne{p}{p_{0}} = \KLLigne{p_{J}}{p_{j,0}} + \sum_{j=1}^{J} \expeLigne{\KLLigne{\bar{p}_{j}(\cdot|x_{j})}{\bar{p}_{j,0}(\cdot|x_{j}) }}  .
    }
    \end{equation}
    Combining \Cref{sec:conv-results-discr-2} and \Cref{sec:norm-covar},
    we get that
  \begin{equation}
    \textstyle{
      \KLLigne{p}{p_{0}} \leq (\delta + \exp[-4T]) (2^{-J}L)^n + \sum_{j=1}^J (\delta + \exp[-4T]) (2^{-j}L)^n(2^n-1) +E_{T, \delta}  .
    }
  \end{equation}
  Therefore, $\KLLigne{p}{p_{0}} \leq (\delta + \exp[-4T]) L^n + E_{T, \delta}$, which concludes the proof.
\end{proof}


\section{Experimental Details on Gaussian Experiments}\label{app:gaussian_exps}

We now give some details on the experiments in \Cref{sec:time-sampl-accur} (\Cref{fig:gaussian}). We use the exact formulas for the Stein score of $p_t$ in this case: if $x_0 \sim \mathcal{N}(M,\Sigma)$, then $x_t \sim N(M_t, \Sigma_t)$ with $M_t = e^{-t}M$ and
$$ \Sigma_t = \mathrm{e}^{-2t}\Sigma + (1 - \mathrm{e}^{-2t})\Id.$$
Under an ideal situation where there is no score error, the discretization of the (backward) generative process is given by equation \eqref{eq:recur_recu}: 
\begin{equation}
  x_{k+1} = ((1 + \delta) \Id - 2\delta \Sigma_{T - k \delta}^{-1}) x_k  + 2 \delta \Sigma_{T - k \delta}^{-1} M_{T-k \delta}+ \sqrt{2 \delta} z_{k+1} \eqsp ,
\end{equation}
where $\delta$ is the uniform step size and $z_k$ are iid white Gaussian random variables. For the SGM case, $M=0$. The starting step of this discretization is itself $x_0 \sim \mathcal{N}(0,\Id)$. From this formula, the covariance matrix $\hat{\Sigma}_k$ of $x_k$ satisfies the recursion \eqref{eq:recursion_cov}: 
\begin{equation}
  \hat{\Sigma}_{k+1} = ((1 + \delta) \Id - 2\delta \Sigma_{T - k \delta}^{-1}) \hat{\Sigma}_k ((1 + \delta) \Id - 2\delta \Sigma_{T - k \delta}^{-1}) + 2 \delta \Id \eqsp ,
\end{equation}
from which we can exactly compute $\hat{\Sigma}_k$ for very $k$, and especially for $k=N=T/\delta$, as a function of $\Sigma$, the final time $T$, and the step size $\delta$. In all our experiments, we choose stationary processes: their covariance $\Sigma$ is diagonal in a Fourier basis, with eigenvalues (\emph{power spectrum}) noted $\hat{P}_k$. All the $x_k$ remain stationary so $\hat{\Sigma}_k$ is still diagonal in a Fourier basis, with power spectrum noted $\hat{P}_k$. The error displayed in the left panel of  \Cref{fig:gaussian} is: $$\Vert \hat{P}_N - P\Vert = \max_\omega |\hat{P}_N(\omega) - P(\omega)|/\max_\omega |P(\omega)|,$$ normalized by the operator norm of $\Sigma$. 

The illustration in the middle panel of \Cref{fig:gaussian}, for WSGM, is done for simplicity only at one scale (ie, at $j=1$ in \Cref{alg:cascaded_wavelet}): instead of stacking the full cascade of conditional distributions for all $j=J, \dotsc, 1$, we use the true low-frequencies $\txwav_{j,0} = x_1$. Here, we use Daubechies-4 wavelets. We sample $\tox_{j,0}$ using the Euler-Maruyama recursion \eqref{eq:recur_recu}-\eqref{eq:recursion_cov} for the conditional distribution. We recall that in the Gaussian case, $\bar{x}_1$ and $x_1$ are jointly Gaussian. The conditional distribution of $\bar{x}_1$ given $x_1$ is known to be $\mathcal{N}(Ax_1, \Gamma)$, where:
\begin{align}
    &A = - \mathrm{Cov}(\bar{x}_1, x_1)\mathrm{Var}(x_1)^{-1}, && \Gamma = \mathrm{Var}(\bar{x}_1) - \mathrm{Cov}(\bar{x}_1, x_1)\mathrm{Var}(x_1)^{-1}\mathrm{Cov}(\bar{x}_1, x_1)^\top.
\end{align}
We solve the recursion \eqref{eq:recursion_cov} with a step size $\delta$ and $N = T/\delta$ steps; the sampled conditional wavelet coefficients $\tox_{j,0}$ have conditional distribution noted $\mathcal{N}(\hat{A}_N x, \hat{\Gamma}_N)$. The full covariance of $(\tx_{j,0}, \tox_{j,0})$, written in the basis given by the high/low frequencies, is now given by
$$ \hat{\Sigma}_N = \begin{bmatrix}\hat{\Gamma}_N & \mathrm{Cov}(x_1, \bar{x}_1)\hat{A}_N^\top \\ \hat{A}_N \mathrm{Cov}(x_1, \bar{x}_1)^\top & \mathrm{Cov}(x_1, x_1) \end{bmatrix}.$$
\Cref{fig:gaussian}, middle panel compares the eigenvalues (power spectrum) of these covariances, as a function of $\delta$, with the ones of $\Sigma$. 

The right panel of \Cref{fig:gaussian} gives the smallest $N$ needed to reach $\Vert \hat{P}_N - P \Vert = 0.1$ in both cases (SGM and WSGM), based on a power law extrapolation of the curves $N\mapsto \hat{P}_N$.

\section{Experimental Details on the $\varphi^4$ Model}
\label{sec:deta-varph-model}
In this section, we develop and precise the results in \Cref{sec:results_phi4}.

\subsection{The Critical $\varphi^4$ Process and its Stein Score Regularity }

The macroscopic energy of non-Gaussian distributions can be specified as shown in \eqref{eq:gibbs}, 
where $K$ is a coupling matrix and $V$ is non-quadratic potential. The $\varphi^4$-model over the $L\times L$ periodic grid is the special case defined by $C=\Delta$ (the two-dimensional discrete Laplacian) and $V$ is a quartic potential: 
\begin{equation}\label{eq:defin_energy}
  \textstyle{
    E(x) = \frac{\beta}{2}\sum_{|u-v|=1} (x(u) - x(v))^2 + \sum_u (x(u)^2 - 1)^2.
    }
\end{equation}
Here, $\beta$ is a parameter proportional to an inverse temperature.

In physics, the $\varphi^4$ model is a typical example of second-order phase transitions: the quadratic part reduces spatial fluctuations, and $V$ favors configurations whose entries remain close to $\pm 1$ (in physics, this is often called a \emph{double-well potential}). In the thermodynamic limit $L\to \infty$, both term compete according to the value of $\beta$. 
\begin{itemize}
    \item For $\beta \ll 1$, the quadratic term becomes negligible and the marginals of the field become independent; this is the disordered state.
    \item For $\beta \gg 1$, the quadratic term favors configuration which are spatially smooth and the potential term drives the values towards $\pm 1$, resulting in an ordered state, where all values of the field are simultaneously close to either $+1$ or to $-1$.
\end{itemize}

\begin{figure}
  \begin{center}
    \includegraphics[width=0.2\textwidth]{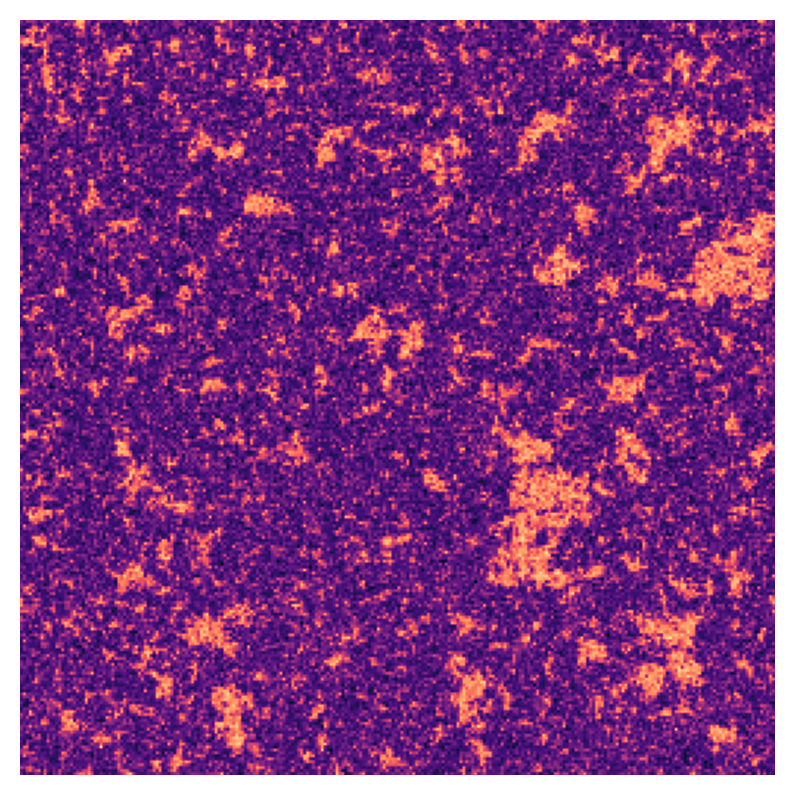}    \includegraphics[width=0.2\textwidth]{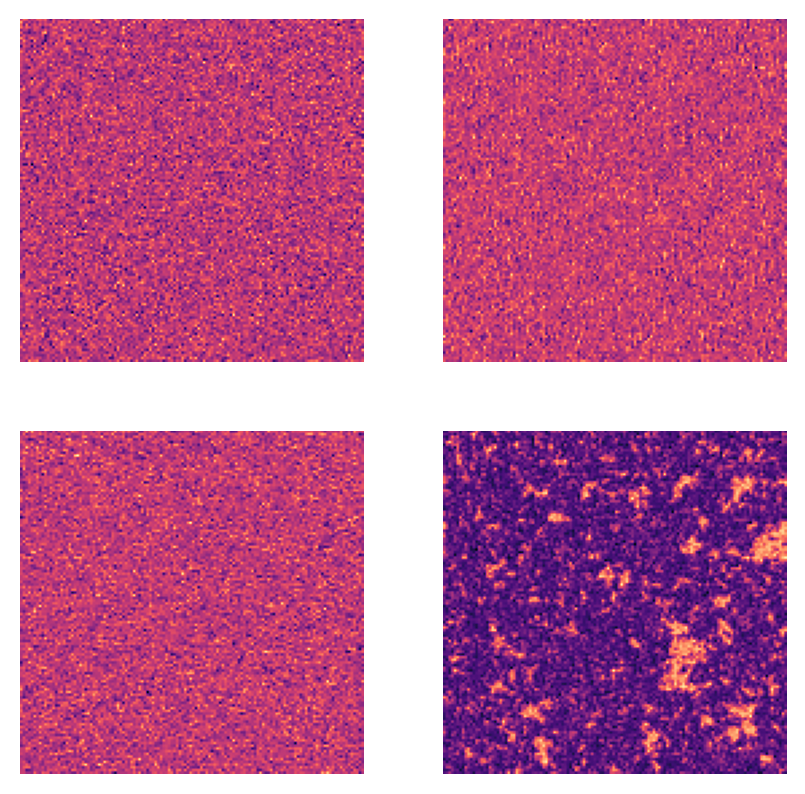}
  \end{center}
  \caption{Example of a realization of a $\varphi^4$ critical field ($L=256$) with its wavelet decomposition on the left (lower-frequencies are on bottom right panel). }
\end{figure}

 A phase transition occurs between these two regimes at a critical temperature $\beta_c \sim 0.68$ \cite{milchev1986finite,  kaupuvzs2016corrections}. At this point, the $\varphi^4$ field display very long-range correlations and an ill-conditioned Hessian $\nabla^2 \log p$. The sampling of $\varphi^4$ at this critical point becomes very difficult. This ``critical slowing down'' phenomenon is why, from a machine learning point of view, the critical $\varphi^4$ field is an excellent example of hard-to-learn and hard-to-sample distribution, still accessible for mathematical analysis.

Our wavelet diffusion considers the sampling of the conditional probability $p(\bar{x}_1|x_1)$ instead of $p(x_0)$, by inverting the noise diffusion projected on the wavelet coefficients. \Cref{thm:control_diffusion} indicates that the loss obtained with any SGM-type method depends on the regularity parameters of $\nabla \log p_t$ in \eqref{params:thm2}. 

Strictly speaking, to get a bound on $K$ we should control the norm of $\nabla^2 \log p_t$ over all $x$ and $t$. However, a look at the proof of the theorem indicates that this control does not have to be uniform in $x$; for instance, there is no need to control this Hessian in domains which have an extremely small probability under $p_t$. Moreover, since $p_t$ is a convolution between $p_0$ and a Gaussian, we expect that a control over $\nabla^2 \log p_0(x)$ will actually be sufficient to control $\nabla^2\log p_t(x)$ for all $t>0$; these facts are non-rigorous for the moment. The distribution of some spectral statistics of $\nabla^2 \log p_0$ over samples drawn from the $\varphi^4$-model are shown in  \Cref{fig:hessian_conditioning} (blue).

Considering conditional probabilities $\bar{p}$ instead of $p$ acts on the Hessian of the $\varphi^4$-energy as a projection over the wavelet field: in the general context of \eqref{eq:gibbs}, 

\begin{align}\label{eq:hessian_cond}
 -\nabla^2_x \log p(x_0) = K + \nabla^2 V(x_0) , \qquad -\nabla^2_{\bar{x}_1} \log p(\bar{x}_1|x_1) =  \gamma^2 \bar{G}(K  +  \nabla^2 V(x_0)) \bar{G}^\top .
\end{align} 
The proof is in \Cref{sec:deta-varph-model}. The  distribution of the conditioning number of $\nabla_x^2 \log p$ and $ \nabla^2_{\bar{x}}\log p$ over samples drawn from the $\varphi^4$ model is shown at \Cref{fig:hessian_conditioning}: the Hessian of the wavelet log-probability is orders-of-magnitude better conditioned than its single-scale counterpart, with a very concentrated distribution. The same phenomenon occurs at each scale $j$, and the same is true for $\lambda_{\min},\lambda_{\max}$. It turns out that considering wavelet coefficient not only concentrates these eigenvalues, but also drives $\lambda_{\min}$ away from 0. In the context of multiscale Gaussian processes, \Cref{sec:norm-covar} gives a rigorous proof of this phenomenon. In the general case, $\nabla^2 \log p_t$ is not reduced to the inverse of a covariance matrix, but we expect the same phenomenon to be true.

\begin{figure}
    \centering
    \includegraphics[width=0.8\textwidth]{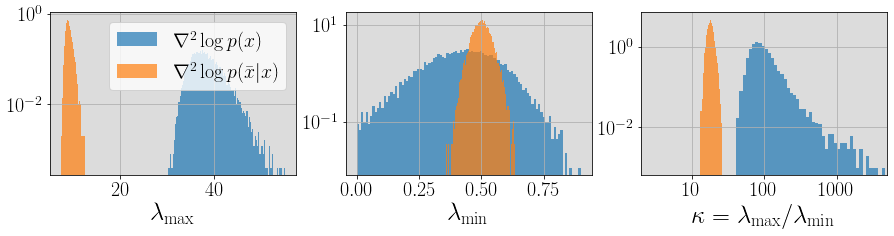}
    \caption{Histograms of $10^5$ realizations of $\lambda_{\min}, \lambda_{\max}$ and $\kappa = \lambda_{\max}/\lambda_{\min}$ of the  Hessian matrices in \eqref{eq:hessian_cond} for the critical $\varphi^4$ model in dimension $L=32$. The mean values of $\kappa$ are respectively $\mu=18.32$ and $\bar{\mu} = 210.53$; standard deviations are $\sigma = 1.78$ and $\bar{\sigma} = 9451.37$.  }
    \label{fig:hessian_conditioning}
\end{figure}

\subsection{Score Models and Details on our Numerical Experiments of $\varphi^4$}\label{app:score_phi4}

In this section, we give some details on our numerical experiments from \Cref{sec:results_phi4}. 

\subsubsection*{Training Data and Wavelets}

We used samples from the $\varphi^4$ model generated using a classical MCMC algorithm --- the sampling script will be publicly available in our repository. 

The wavelet decompositions of our fields were performed using Python's \texttt{pywavelets} package and \texttt{Pytorch Wavelets} package. For synthetic experiments, we used the Daubechies wavelets with $p=4$ vanishing moments (see \cite[Section 7.2.3]{mallat1999wavelet}).

\subsubsection*{Score Model}

At the first scale $j=0$, the distribution of the $\varphi^4$ models falls into the general form given in \eqref{eq:gibbs}, and it is assumed that at each scale $j$, the distribution of the field at scale $j$ still assumes this shape --- with modified constants and coupling parameters. The score model we use at each scale is given by:
\begin{equation}\label{eq:phi4score}
  \textstyle{
    {\bs}_{K,\theta}(x) = \frac{1}{2}x^\top K x + \sum_{u}(\theta_1v_1(x(u)) + \dotsb + \theta_m v_m(x(u))),
    }
\end{equation}
where the parameters are $K,\theta_1, \dotsc, \theta_m$ and $v_i$ are a family of smooth functions. One can also represent this score as $s_{K,\theta} = K \cdot xx^\top + \theta^\top U(x)$ where $U_i(x) = \sum_u v_i(x(u))$.

\subsubsection*{Learning}

We trained our various algorithms using SGM or WSGM up to a time $T=5$ with $n_{\mathrm{train}} = 2000$ steps of forward diffusion. At each step $t$, the parameters were learned by minimizing the score loss:
$$ \ell(K,\theta) = \mathbb{E}[|\nabla s_{K,\theta}(x_t)|^2 + 2 \Delta_x s_{K,\theta}(x_t)]$$
using the Adam optimiser with learning rate $\texttt{lr} = 0.01$ and default parameters $\alpha, \beta$. At the start of the diffusion ($t=0$) we use 10000 steps of gradient descent. For $t>1$, we use only 100 steps of gradient descent, but initialized at $(K_{t-1}, \theta_{t-1})$.

\subsubsection*{Sampling}

For the sampling, we used uniform steps of discretization. 

For the error metric, we first measure the $L^2$-norm between the power spectra $P, \hat{P}$ of the true $\varphi^4$ samples, and our syntheses; more precisely, we set: 
$$ D_1 = \Vert P - \hat{P}\Vert^2.$$
This error on second-order statistics is perfectly suitable for Gaussian processes, but must be refined for non-Gaussian processes. We also consider the total variation distance between the histograms of the marginal distributions (in the case of two-dimensions, pixel-wise histograms). We note this error $D_2$; our final error measure is $D_1 + D_2$. This is the error used in Figure \Cref{fig:hessian_conditioning}. 

\subsection{Proofs of \eqref{eq:hessian_cond}}

In the sequel, $\nabla f$ is the gradient of a function $f$, and $\nabla^2$ is the Hessian matrix of $f$. The \emph{Laplacian} of $f$ is the trace of the Hessian. 

\begin{lemma}
Let $U:\mathbb{R}^n \to \mathbb{R}$ be smooth and $M$ be a $n\times m$ matrix. We set $F(x) = U(Mx)$ where $x\in \mathbb{R}^m$. Then, $\nabla^2 F(x) = M^\top \nabla^2 U(x) M$. 
\end{lemma}
\begin{proof}
  Let $U:\mathbb{R}^n \to \mathbb{R}$ be smooth and $M$ be a $n\times m$ matrix. Then, $$\partial_k F(x) = \sum_{i=1}^n M_{i,k} (\partial_i U)(Mx).$$ Similarly,  
  \begin{equation}\label{p:partialij}\partial_{k,\ell}F(x) = \sum_{i=1}^n \sum_{j=1}^n M_{i,k}M_{j,\ell}\partial_{j}(\partial_iU)(Mx).\end{equation}
  This is equal to $(M^\top \nabla^2U M)_{k,\ell}$. 
\end{proof}

\begin{lemma}Under the setting of the preceding lemma, if $U(x) = \sum_{i=1}^n f(x_i)$, then (1) $\nabla^2 U(x) = \mathrm{diag}(u''(x_1), \dotsc, u''(x_n))$ and (2) the Laplacian of $F(x) = U(Mx)$ is given by 
$$ \Delta F(x) = \sum_{i=1}^n (M^\top M)_{i,i}u''(x_i). $$
\end{lemma}

\begin{proof}
  The proof of (1) comes from the fact that $\partial_i U(x) = u'(x_i)$, hence $\partial_j \partial_i U(x) = u''(x_i)$ if $i=j$, zero otherwise. The proof of (2)  consists in summing the $k=\ell$ terms in \eqref{p:partialij} and using (1). 
\end{proof}

For simplicity, let us note $p(x) = \rme^{-H(x)}/Z$ where $Z_0$ is a normalization constant and $H(x) = x^\top K x /2 + V(x)$. Then, 
\begin{align}\nabla_x p(x) = -\nabla_x H(x), && \nabla^2 _x p(x)  = - \nabla^2_x H(x),\end{align}and the formula in the left of \eqref{eq:hessian_cond} comes from the fact that the Hessian of $x^\top K x$ is $2K$. 

For the second term, let us first recall that if $\bar{x}_1$ and $x_1$ are the wavelet coefficients and low-frequencies of $x$, they are linked by \eqref{eq:reconstruction}. Consequently, the joint density of $(\bar{x}_1, x_1)$ is: 
$$ q(\bar{x}_1, x_1) = \rme^{-H(\gamma G^\top x_1 + \gamma \bar{G}^\top \bar{x}_1)} / Z_1$$
where $Z_1$ is another normalization constant. The conditional distribution of $\bar{x}_1$ given $x_1$ is: 
$$ q(\bar{x}_1 | x_1) = \frac{q(\bar{x}_1, x_1)}{Z_1(x_1)}$$
where $Z_1(x) = \int q(\bar{x}_1, u)\rmd u$. Consequently, 
\begin{align} \nabla_{\bar{x}_1} \log q(\bar{x}_1|x_1) &= \nabla_{\bar{x}_1}(-H(\gamma G^\top x_1 + \gamma \bar{G}^\top \bar{x}_1) - \log Z_1) - \nabla_{\bar{x}_1} Z_1(x_1)\\
&= - \nabla_{\bar{x}_1}H(\gamma G^\top x_1 + \gamma \bar{G}^\top \bar{x}_1)
\end{align}
and additionally: 
$$ \nabla^2_{\bar{x}_1}q(\bar{x}_1 | x_1) = - \nabla^2_{\bar{x}_1}H(\gamma G^\top x_1 + \gamma \bar{G}^\top \bar{x}_1). $$
The RHS of \eqref{eq:hessian_cond} then follows from the lemmas in the preceding section.

\section{Experimental Details on CelebA-HQ}
\label{sec:exper-deta-addit}

\paragraph{Data} We use Haar wavelets. The $128 \times 128$ original images are thus successively brought to the $64 \times 64$ and $32 \times 32$ resolutions, separately for each color channel. Each of the $3$ channels of $x_j$ and $9$ channels of $\bar x_j$ are normalized to have zero mean and unit variance.

\paragraph{Architecture}  Following \cite{nichol2021improved}, both the conditional and unconditional scores are parametrized by a neural network with a UNet architecture. It has $3$ residual blocks at each scale, with a base number of channels of $C = 128$. The number of channels at the $k$-th scale is $a_k C$, where the multipliers $(a_k)_k$ depend on the resolution of the generated images. They are $(1,2,2,4,4)$ for models at the $128 \times 128$ resolution, $(2,2,4,4)$ for models at the $64 \times 64$ resolution, $(4,4)$ for the conditional model at the $32 \times 32$ resolution, and $(1,2,2,2)$ for the unconditional model at the $32 \times 32$ resolution. All models include multi-head attention layers at in blocks operating on images at resolutions $16 \times 16$ and $8 \times 8$. The conditioning on the low frequencies $x_j$ is done with a simple input concatenation along channels, while conditioning on time is done through affine rescalings with learned time embeddings at each GroupNorm layer \citep{nichol2021improved,saharia2021image}.

\paragraph{Training} The networks are trained with the (conditional) denoising score matching losses:
\begin{align}
    \ell(\theta_J) &= \mathbb E_{x_J, t, z}\left[ || s_{\theta_J}(t, e^{-t} x_{J} + \sqrt{1 - e^{-2t}} z) - z  ||^2 \right] \\
    \ell(\bar \theta_j) &= \mathbb E_{\bar x_j, x_j, t, z}\left[ || \bar s_{\bar \theta_j}(t, e^{-t} \bar x_{j} + \sqrt{1 - e^{-2t}} z \,|\, x_j) - z  ||^2 \right]
\end{align}
where $z \sim \mathcal N(0, \Id)$ and the time $t$ is distributed as $T u^2$
with $u \sim \mathcal U([0, 1])$. We fix the maximum time $T = 5$ for all
scales. Networks are trained for $5\times{10}^5$ gradient steps with a batch
size of $128$ at the $32 \times 32$ resolution and $64$ otherwise. We use the
Adam optimizer with a learning rate of ${10}^{-4}$ and no weight decay.

\paragraph{Sampling} For sampling, we use model parameters from an exponential moving average with a rate of $0.9999$. For each number of discretization steps $N$, we use the Euler-Maruyama discretization with a uniform step size $\delta_k = T/N$ starting from $T = 5$. This discretization scheme is used at all scales. For FID computations, we generate $30,000$ samples in each setting.


\end{document}